\pdfoutput=1
\documentclass[sn-mathphys,Numbered]{sn-jnl}% Math and Physical Sciences Reference Style
\usepackage{wrapfig}
\usepackage{microtype}
\usepackage{subfig}
\usepackage{graphicx}%
\usepackage{multirow}%
\usepackage{amsmath,amssymb,amsfonts}%
\usepackage{amsthm}%
\usepackage{mathrsfs}%
\usepackage[title]{appendix}%
\usepackage{xcolor}%
\usepackage{textcomp}%
\usepackage{manyfoot}%
\usepackage{booktabs}%
\usepackage{algorithm}%
\usepackage{algorithmicx}%
\usepackage{algpseudocode}%
\usepackage{listings}%
\usepackage{tabularx}
\usepackage{epstopdf}
\usepackage{minitoc}
\newcommand{\ra}[1]{\renewcommand{\arraystretch}{#1}}

\theoremstyle{thmstyleone}%
\newtheorem{theorem}{Theorem}[section]
\newtheorem{proposition}[theorem]{Proposition}
\newtheorem{lemma}[theorem]{Lemma}

\newtheorem{definition}[theorem]{Definition}

\newtheorem{exmp}{Example}[section]
\begin{document}
\doparttoc % Tell to minitoc to generate a toc for the parts
\faketableofcontents % Run a fake tableofcontents command for the partocs

% \part{} % Start the document part
% \parttoc % Insert the document TOC
\title[Article Title]{A New Formulation of Lipschitz Constrained With Functional Gradient  Learning for GANs}


\author[1]{\sur{Chang Wan}}\email{wanchang\_phd2020@zjnu.edu.cn}

\author[2]{\sur{Ke Fan}}\email{kfan21@m.fudan.edu.cn}
\author[2,3]{\sur{Xinwei Sun}}\email{sunxinwei@fudan.edu.cn}
\equalcont{These authors should be nominated as the corresponding author.}
\author[2,3]{\sur{Yanwei Fu}}\email{yanweifu@fudan.edu.cn}
\equalcont{These authors should be nominated as the corresponding author.}

\author[1]{\sur{Minglu Li}}\email{mlli@zjnu.edu.cn}
\author[1]{\sur{Yunliang Jiang}}\email{jyl2022@zjnu.cn}
\author*[1]{\sur{Zhonglong Zheng}}\email{zhonglong@zjnu.edu.cn}
\equalcont{These authors should be nominated as the corresponding author.}

\affil*[1]{\orgdiv{School of Computer Science and Technology}, \orgname{Zhejiang Normal University}, \orgaddress{\street{No. 688 Yingbin Avenue}, \city{Jinhua}, \postcode{321004}, \state{Zhejiang}, \country{China}}}

\affil[2]{\orgdiv{School of Data Science and MOE Frontiers Center for Brain Science}, \orgname{Fudan University}, \orgaddress{\street{No.220 Handan Road}, \city{Shanghai}, \postcode{200433}, \state{Shanghai}, \country{China}}}

\affil[3]{\orgdiv{Fudan ISTBI-JNU Algorithm Centre for Brain inspired Intelligence}, \orgname{Zhejiang Normal University}, \orgaddress{\street{No. 688 Yingbin Avenue}, \city{Jinhua}, \postcode{321004}, \state{Zhejiang}, \country{China}}}

\abstract{
This paper introduces a promising alternative method for training Generative Adversarial Networks (GANs) on large-scale datasets with clear theoretical guarantees. GANs are typically learned through a minimax game between a generator and a discriminator, which is known to be empirically unstable. Previous learning paradigms have encountered mode collapse issues without a theoretical solution.
To address these challenges, we propose a novel Lipschitz-constrained Functional Gradient GANs learning (Li-CFG) method to stabilize the training of GAN and provide a theoretical foundation for effectively increasing the diversity of synthetic samples by reducing the neighborhood size of the latent vector. Specifically, we demonstrate that the neighborhood size of the latent vector can be reduced by increasing the norm of the discriminator gradient, resulting in enhanced diversity of synthetic samples.
To efficiently {enlarge} the norm of the discriminator gradient, we introduce a novel $\boldsymbol\varepsilon$-centered gradient penalty that amplifies the norm of the discriminator gradient using the hyper-parameter $\boldsymbol\varepsilon$. In comparison to other constraints, our method {enlarging} the discriminator norm, thus obtaining the smallest neighborhood size of the latent vector. {Extensive experiments on benchmark datasets for image generation demonstrate the efficacy of the Li-CFG method and the $\boldsymbol\varepsilon$-centered gradient penalty.} The results showcase improved stability and increased diversity of synthetic samples.
}




\maketitle

\begin{wrapfigure}{r}{8cm}
\centering
\vspace{-0.25in}
\includegraphics[width=0.45\textwidth]{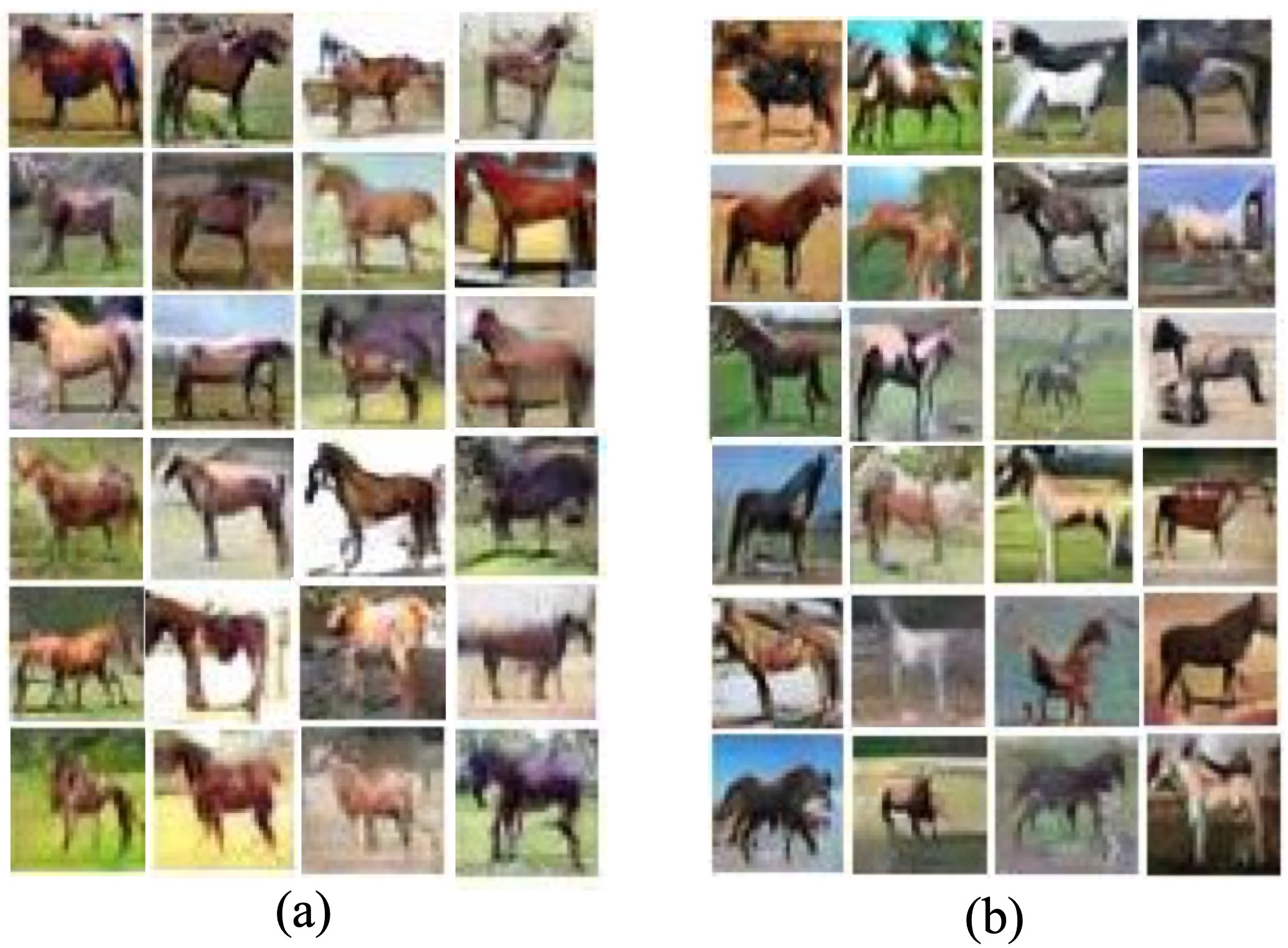}
\caption{ \small
Highlighting Diversity. 
We underscore the significance of diversity in image synthesis. The left column (a) and right column (b) display horse label images generated using the CFG and Li-CFG methods trained on the CIFAR-10 dataset, respectively.
\label{diversity}}
\vspace{-0.2in}
\end{wrapfigure}
\section{Introduction}\label{Introduction:main}
GANs are designed to sample a random variable $z$ from a known distribution $p_z$, approximating the underlying data distribution $p_*$. This learning process is modeled as a minimax game where the generator and discriminator are iteratively optimized, as introduced by~\citet{goodfellow2014generative}. The generator produces samples that mimic the true data distribution, while the discriminator distinguishes between the generated data and real data samples.
Despite numerous remarkable efforts that have been made~\cite{arjovsky2017wasserstein, gulrajani2017improved}, the GANs learning still suffers from training instability and mode collapse.  

{Recently, Composite Functional Gradient Learning (CFG), as proposed in~\citet{johnson2018composite}, has gained attention. CFG utilizes a strong discriminator and functional gradient learning for the generator, leading to convergent GAN learning theoretically and empirically.}
However, we have observed that there are still various hyper-parameters in CFG that may significantly impact the GAN learning process. Despite the advancements in stability, one still needs to carefully set these hyper-parameters to ensure a successful and well-trained GAN. Properly tuning these parameters remains an essential aspect of achieving optimal performance in GAN training. 
This issue hinders the widespread adoption of the CFG method for training GAN in real-world and large-scale datasets. {One of the most effective mechanisms for addressing this issue for GAN training is the Lipschitz constraint, with the $0$-centered gradient penalty introduced by ~\citet{mescheder2018training} and the $1$-centered gradient penalty introduced by~\citet{gulrajani2017improved} being among the most well-known.}
{Empirically, stable training of a GAN can result in more diverse synthesis results. However, the theoretical foundation for stable GAN training using the Lipschitz gradient penalty and generating diverse synthesis samples is still unclear. It is therefore important to develop a new theoretical framework that can account for the Lipschitz constraint of the discriminator and the diversity of synthesis samples.}

To address these challenges, we propose Lipschitz Constrained Functional Gradient GANs Learning (Li-CFG), an improved version of CFG. 
We present a comparative analysis of our Li-CFG and CFG methods in Fig.~\ref{motivate} and Example.~\ref{ex:example}. 
Additionally, we emphasize the importance of diversity in synthetic samples through Fig.~\ref{diversity}. Remarkably, as shown in Fig.~\ref{diversity}, the horse images generated by Li-CFG display a more diverse range of gestures, colors, and textures in comparison to those produced by the CFG method. When synthetic samples lack diversity, they fail to capture essential image characteristics. While recent generative models have made notable progress in enhancing diversity, there remains a gap in providing comprehensive theoretical explanations for certain aspects.
\begin{figure*}
\centering
\setlength{\abovecaptionskip}{-1.4cm}
\includegraphics[width=1\textwidth]{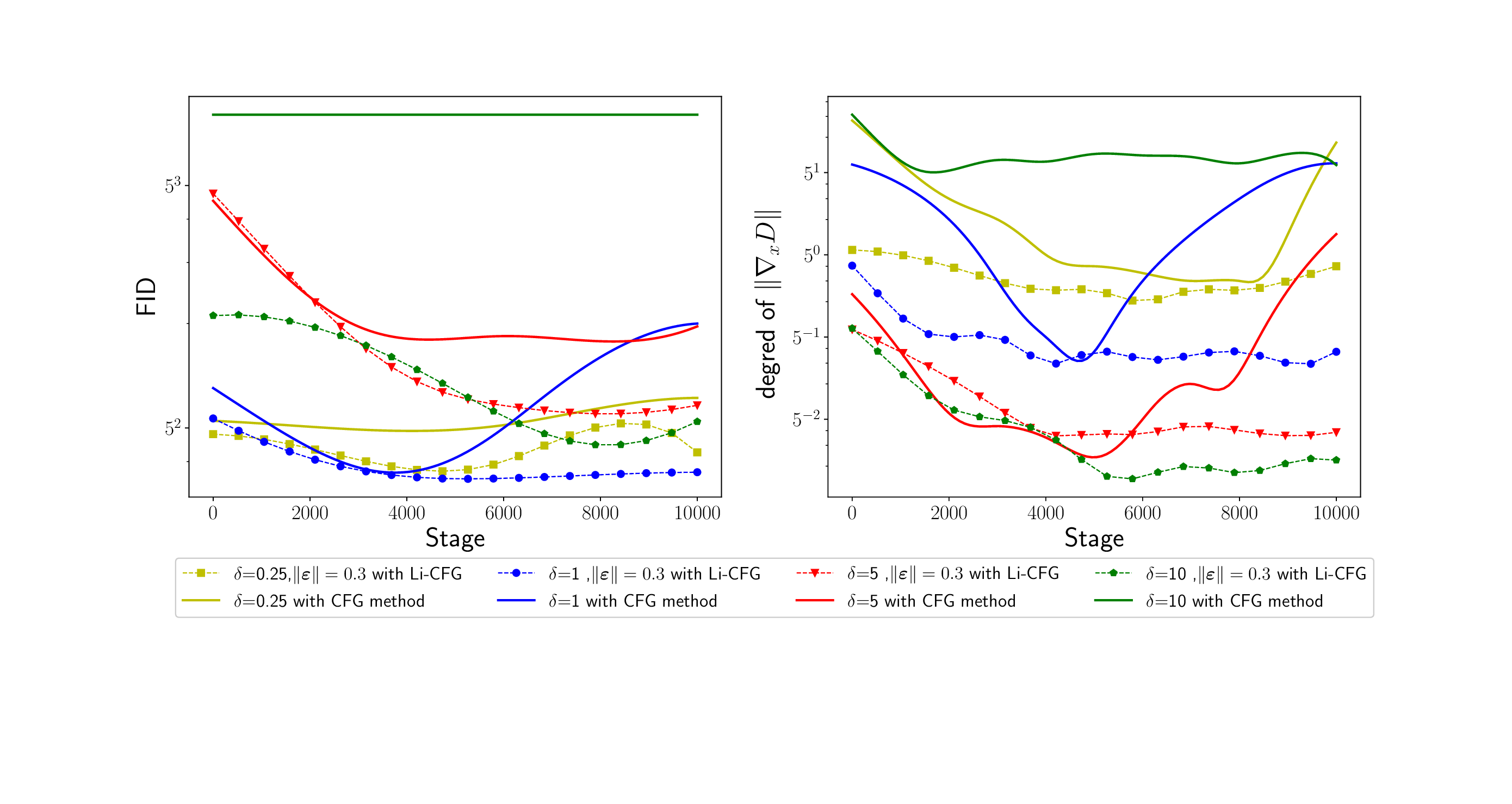}
\caption{ \small
CFG method V.S Li-CFG. The left and right figures show the results of the FID and the norm of the discriminator gradient $\|\nabla_{\boldsymbol x} D(\boldsymbol x)\|$ for the CFG method and Li-CFG with different values of the hyper-parameter $\delta(\boldsymbol x)$, respectively. 
{FID is a metric that measures the diversity between synthetic samples and real samples. A lower score is better. More information about FID is present in Section \ref{FID-explain}. The hyper-parameter $\delta(\boldsymbol x)$ is important as it controls the gradient magnitude for the CFG method, which is defined in Eq.~(\ref{eq:delta_x}).}
Solid and dashed lines of the same color in both figures indicate FID and $\|\nabla_{\boldsymbol x} D(\boldsymbol x)\|$ result of CFG method and Li-CFG with the same $\delta(\boldsymbol x)$.  \label{motivate}}
\vspace{-0.1in}
\end{figure*}

\begin{exmp}
  Fig.~\ref{motivate} indicates the idea that too large degree of $\|\nabla_{\boldsymbol x} D(\boldsymbol x)\|$ leads to an excessively small neighborhood of the latent vector which causes an untrained model, a too-small degree of $\|\nabla_{\boldsymbol x} D(\boldsymbol x)\|$
lead to an overly large neighborhood of the latent vectors which cause a worse diversity, i.e., the blue dashed lines with a reasonable degree of $\|\nabla_{\boldsymbol x} D(\boldsymbol x)\|$ lead to the best FID compared to the results of other colors, the green solid lines with a too large degree of $\|\nabla_{\boldsymbol x} D(\boldsymbol x)\|$ lead to an untrained result and the red dashed lines with a too small degree of $\|\nabla_{\boldsymbol x} D(\boldsymbol x)\|$ lead to a worse diversity.
Training results with different $\delta(\boldsymbol x)$ of CFG method and Li-CFG, illustrating that FID result of CFG with different $\delta(\boldsymbol x)$ varying dramatically due to the unstable change of $\|\nabla_{\boldsymbol x} D(\boldsymbol x)\|$. However, the FID result of our Li-CFG with different $\delta(\boldsymbol x)$ varying more stable due to the smooth
$\|\nabla_{\boldsymbol x} D(\boldsymbol x)\|$
thanks to the gradient penalty. By controlling the degree of $\|\nabla_{\boldsymbol x} D(\boldsymbol x)\|$ via changing the  $\boldsymbol\varepsilon$ value of our $\boldsymbol\varepsilon$-centered gradient penalty, we can adjust the neighborhood size of the latent vector and consequently influence the degree of diversity of synthetic samples.\label{ex:example}
\end{exmp}
\textbf{Our key insight is the introduction of Lipschitz continuity, a robust form of uniform continuity, into CFG. 
This provides a theoretical basis showing that synthetic sample diversity can be enhanced by reducing the latent vector neighborhood size through a discriminator constraint.}
{For simplicity and clarity, we will use latent N-size to denote the \textbf{n}eighborhood \textbf{s}ize of the \textbf{l}atent vector and constraint to refer to the discriminator constraint for the rest of the article.}
First, Li-CFG integrates the Lipschitz constraint into CFG to tackle instability in training under dynamic theory
Second, we establish a theoretical link between the discriminator gradient norm and the latent N-size.
By increasing the discriminator gradient norm, the latent N-size is reduced, thus enhancing the diversity of synthetic samples.
Lastly, to efficiently adjust the discriminator gradient norm, we introduce a novel Lipschitz constraint mechanism, the $\boldsymbol\varepsilon$-centered gradient penalty.
This mechanism enables fine-tuning of the latent vector neighborhood size by varying the hyper-parameter $\boldsymbol\varepsilon$. 
Through this approach, we aim to achieve more effective control over the discriminator gradient norm and further improve the diversity of generated samples.

% \uppercase\expandafter{\romannumeral3.8.}
We summarize the key contributions of this work as follows: 
(1) We introduce a novel Lipschitz constraint to CFG, the $\boldsymbol\varepsilon$-centered gradient penalty, addressing the hyperparameter sensitivity in regression-based GAN training. 
Our Li-CFG method enables stable GAN training, producing superior results compared to traditional CFG. 
(2) We present a new perspective on analyzing the relationship between the discriminator constraint and synthetic sample diversity.
To the best of our knowledge, we are the first to explore this relationship, demonstrating that our $\boldsymbol\varepsilon$-centered gradient penalty allows for effective control of diversity during training.
(3) Our empirical studies highlight the superiority of Li-CFG over CFG across a variety of datasets, including synthetic and real-world data such as MNIST, CIFAR10, LSUN, and ImageNet. 
In an ablation study, we show a trade-off between diversity and model trainability.
Additionally, we demonstrate the generalizability of our $\boldsymbol\varepsilon$-centered gradient penalty across multiple GAN models, achieving better results compared to existing gradient penalties.

\section{Preliminary}
{
To enhance readability, we provide definitions for frequently used mathematical symbols in our theory. For CFG and dynamic theory, the symbols will be explained in the respective sections.
}

\noindent \textbf{Notation.}
{
$G_\theta(z): \mathcal{Z} \rightarrow \mathcal{Y}$ represents a generator model that maps an element of the latent space $\mathcal{Z}$ to the image space $\mathcal{Y}$, where the parameter of generator is denoted by $\theta$.
Here $z$ is a sample drawn from a low dimensional latent distribution $p_z$.  
We will use $G_\theta$ and $G$ interchangeably to refer to the generator throughout this article. % The symbol $\mathcal{Z}$ and $\boldsymbol{z}$ both represent samples drawn from a low-dimensional latent distribution $p_z$.
}

{
$D_\psi: \mathcal{Y} \rightarrow \mathcal{C}$ represents a discriminator model that maps elements from the image space $\mathcal{Y}$ to the probability of belonging to real data distribution or the generated data distribution, where the parameter of generator is denoted by $\psi$.
$\mathcal{Y}$ is a sample drawn from real data distribution $p_*$. $\mathcal{C}$ consists of the scale value obtained from the outputs of the discriminator.
The symbol $D$ in the CFG method, along with the notation $D_{\psi}$ used in both the gradient penalty and our neighborhood theory, all represent the discriminator.}

{
We use $R$ to denote the gradient penalty. 
Furthermore, we use $R_{1}$ to denote the $1$-centered Gradient Penalty,  $R_{0}$ to denote the $0$-centered Gradient Penalty and $R_{\boldsymbol\varepsilon}$ to denote the $\boldsymbol\varepsilon$-centered Gradient Penalty. $r_{R}$ stands for the latent N-size of the corresponding gradient penalty.
The symbol $\hat{\epsilon}$ in the definition of our neighborhood method represents a small quantity in image space. It is distinct from the symbol $\boldsymbol\varepsilon$ used in our $\boldsymbol\varepsilon$-centered GP method.}

% \subsection{Overview}
\noindent \textbf{Composite Functional Gradient GANs.}
{We follow the definition in CFG~\citep{johnson2019framework}. The CFG employs discriminator works as a logistic regression to differentiate real samples from synthetic samples. Meanwhile, it employs the functional compositions gradient to learn the generator as the following form
\begin{gather}\small
   \begin{aligned}
  G_{m}( \boldsymbol{z})=G_{m-1}( \boldsymbol{z})+\eta_{m}g_m(G_{m-1}( \boldsymbol{z})),(m=1...,M) \label{eq:generator}
\end{aligned} 
\end{gather} 
to obtain $G(\boldsymbol{z})=G_{M}(\boldsymbol{z})$.
The $M$ represents the number of steps in the generator used to approximate the distribution of real data samples. Each $g_m$ is a function to be estimated from data.} {$g_m$ is a residual which gradually move the generated samples of $m-1$ step  $G_{m-1}( \boldsymbol{z})$ towards the $m$ step $G_{m}( \boldsymbol{z})$.
 $g_m$ guarantees that the distance between the latent distribution $p_z$ and the real data distribution $p_*$ will gradually decrease until it reaches zero, }
{and the $\eta_m$ is a small step size.}

{To simplify the analysis of this problem, we first transform the discrete M into the continuous M.
First, we transform the Eq.~(\ref{eq:generator}) into $
G_{m+\delta}(Z)=G_m(Z)+\delta g_m\left(G_m(Z)\right)$ by setting $\eta_m=\delta$, where $\delta$ is a small time step. 
By letting $\delta \rightarrow 0$, we have a generator that evolves continuously in time $m$ that satisfies an ordinary differential equation
$$
\frac{d\left(G_m(Z)\right)}{d m}=g_m\left(G_m(Z)\right).
$$}

{
The goal is to learn $g_m: \mathbb{R}^k \rightarrow \mathbb{R}^k$ from data so that the probability density $p_m$ of $G_m(\boldsymbol{z})$, which continuously evolves by Eq.~(\ref{eq:generator}), becomes close to the density $p_*$ of real data as $m$ continuously increases. 
To measure the ‘closeness’, we use
$L$ denotes a distance measure between two distributions:}

{
$$
L(p_m)=\int \ell\left(p_*(x), p_m(x)\right) d x ,
$$
where $\ell: R^2 \rightarrow R$ is a pre-defined function so that $L$ satisfies $L(p_m)=0$ if and only if $p_m=p_*$ and $L(p) \geq 0$ for any probability density function $p$.
}

{
From the above equation, we will derive the choice of $g_m(\cdot)$ that guarantees that transformation Eq.~(\ref{eq:generator}) can always reduce $L(\cdot)$.
Let $p_m$ be the probability density of random variable $G_t(\boldsymbol{z})$. Let $\ell_2^{\prime}\left(\rho_*, \rho\right)=$ $\partial \ell\left(\rho_*, \rho\right) / \partial \rho$. Then we have
$$
\frac{d L\left(p_m\right)}{d m}=\int p_m(x) \nabla_x \ell_2^{\prime}\left(p_*(x), p_m(x)\right) \cdot g_m(x) d x.
$$
}

{
{With this definition of $\frac{d L\left(p_m\right)}{d m}$, we aim to keep that $\frac{d L\left(p_m\right)}{d m}$ is negative, so that the distance $L$ decreases. To achieve this goal, we choose $g_m(x)$ to be:}
\begin{gather}
\begin{aligned}
  g_m(\boldsymbol{x})=-s_m(\boldsymbol{x}) \phi_0\left(\nabla_x \ell_2^{\prime}\left(p_*(\boldsymbol{x}), p_m(\boldsymbol{x})\right)\right),
\end{aligned}
\end{gather}
where $s_m(x)>0$ is an arbitrary scaling factor. $\phi_0(u)$ is a vector function such that $\phi(u)=u \cdot \phi_0(u) \geq 0$ and $\phi(u)=0$ if and only if $u=0$. Here are two examples: $(\phi_0(u)=u, \phi(u)=\|u\|_2^2)$ and $(\phi_0(u)=\operatorname{sign}(u), \phi(u)=\|u\|_1)$.
} 

{With this choice of $g_m(x)$, we obtain
\begin{gather}
\begin{aligned}\label{eq:theorem 2.1}
\frac{d L\left(p_m\right)}{d m}=-\int s_m(x) p_m(x) \phi\left(\nabla_x \ell_2^{\prime}\left(p_*(x), p_m(x)\right)\right) d x \leq 0,
\end{aligned}
\end{gather}
that is, the distance $L$ is guaranteed to decrease unless the equality holds. Moreover, this implies that we have $\lim _{m \rightarrow \infty} \int s_m(x) p_m(x) \phi\left(\nabla_x \ell_2^{\prime}\left(p_*(x), p_m(x)\right)\right) d x=0$. (Otherwise, $L\left(p_m\right)$ would keep going down and become negative as $m$ increases, but $L\left(p_m\right) \geq 0$ by definition.)
For simplicity, we omit the subscript $m$ in the following empirical settings.}

{
Let us consider a case where the distance measure $L(\cdot)$ is an $f$-divergence. With a convex function $f: \mathbb{R}^{+} \rightarrow \mathbb{R}$ such that $f(1)=0$ and that $f$ is strictly convex at $1, L\left(p_m\right)$ defined by
$$
L\left(p_m\right)=\int p_*(x) f\left(r_m(x)\right) d x \text { where } r_m(x)=\frac{p_m(x)}{p_*(x)}
$$
is called $f$-divergence. Here we focus on a special case where $f$ is twice differentiable and strongly convex so that the second order derivative of $f$, denoted here by $f^{\prime \prime}$, is always positive.
For instance, when we consider the KL divergence, $f$ can be represented by $-\ln x$, in which case $f^{\prime}=-1/x$ and $f^{\prime\prime}=1/x^2$. On the other hand, if we consider the reverse KL divergence, $f$ can be represented as $x\ln{x}$, in which case $f^{\prime}=\ln{x}+1$ and $f^{\prime\prime}=1/x$.}

{
{With this definition of function $f$ and $L$ as $f$-divergence setting, the value of $g_m(x)$ should to be}
\begin{gather}
\begin{aligned}
  g_m(\boldsymbol{x})&=-s_m(\boldsymbol{x}) \phi_0\left(\nabla_x \ell_2^{\prime}\left(p_*(\boldsymbol{x}), p_m(\boldsymbol{x})\right)\right) \\
  &=-s_m(\boldsymbol{x})\nabla_x \ell_2^{\prime}\left(p_*(\boldsymbol{x}), p_m(\boldsymbol{x})\right)\\
  &=-s_m(\boldsymbol{x})f^{\prime \prime}\left(r_m(\boldsymbol{x})\right) \nabla r_m(\boldsymbol{x})\\
  &\approx s_m(\boldsymbol{x})f^{\prime \prime}\left(\tilde{r}_m(\boldsymbol{x})\right) \tilde{r}_m(\boldsymbol{x})\nabla_x D(\boldsymbol{x}),\label{eq:empirical}
\end{aligned}
\end{gather}
where $s_m(x)>0$ is an arbitrary scaling factor, $\ell_2\left(\rho_*, \rho\right)=\rho_* f\left(\rho / \rho_*\right)$, $\nabla_x \ell_2^{\prime}\left(p_*(x), p_m(x)\right)=f^{\prime \prime}\left(r_m(x)\right) \nabla r_m(x)$,  $f^{\prime\prime}=1/x^2$ when $f$ is KL-divergence and $r_m(\boldsymbol{x})=\exp(-D(\boldsymbol{x}))\approx p_m(x)/p_*(x)=\tilde{r}_m(\boldsymbol{x})$ when $D(x)\approx \ln \frac{p_*(x)}{p_m(x)}$, which is the analytic solution of the CFG discriminator. 
} 

{
Empirically,
we define the function in Eq.~(\ref{eq:empirical}) as $g( \boldsymbol{x})=\delta(\boldsymbol{x}) \nabla_{\boldsymbol x}  D( \boldsymbol{x})$. The $\delta(\boldsymbol{x}) $ can be computed as
\begin{gather}
\begin{aligned}
\delta(\boldsymbol{x})=s_m( \boldsymbol{x})\tilde{r}_m( \boldsymbol{x})f''(\tilde{r}_m( \boldsymbol{x})), \label{eq:delta_x} 
\end{aligned}
\end{gather}
where we have an arbitrary scaling factor $s_m( \boldsymbol{x})$, a KL-divergence function $f=-\ln  \boldsymbol{x}$, $f^{\prime\prime}=1/x^2$ and $\tilde{r}_m( \boldsymbol{x})=\exp(-D( \boldsymbol{x}))$. Since $ \tilde{r}_m(\boldsymbol{x})f^{\prime\prime}(x)>0$, it can be absorbed into the $s_m(x)$. So the value of $\delta(\boldsymbol{x})$ is always greater than 0. For simplicity, we directly regard $\delta(\boldsymbol{x})$  as the scaling factor and set it to a fixed value as a hyper-parameter.  We also demonstrate the results of different $\delta(\boldsymbol x)$ values in the Fig.~\ref{motivate}.}

\noindent \textbf{Gradient Penalty for GANs.}
{Let us work with the most commonly used  WGAN-GP~\citep{gulrajani2017improved}. 
Its regularization term is commonly referred to as
\begin{gather}
R(\theta, \psi)=\frac{\gamma}{2} \mathrm{E}_{\hat{\boldsymbol x}}\left(\left\|\nabla_{\boldsymbol x} D_{\psi}(\hat{\boldsymbol x})\right\|-g_{0}\right)^{2},
\end{gather}
\noindent 
where $\hat{\boldsymbol x}$ is sampled uniformly on the line segment between two random points vector $\boldsymbol x_{1} \sim p_{\theta}\left(\boldsymbol x_{1}\right), \boldsymbol x_{2} \sim p_{\mathcal{D}}\left(\boldsymbol x_{2}\right)$. 
The notation of $R(\theta, \psi)$ presents the regularization term of the generator with weight $\theta$ and the discriminator with weight $\psi$.
The notion of $\gamma$ is a coefficient that controls the magnitude of regularization. {The value of $g_0$ is always empirically set to 1.} We call it $g_0$-centered GP or $1$-centered GP.}
{
The other solution is referred to as the $0$-centered GP method, proposed by~\citet{roth2017stabilizing} and ~\citet{mescheder2018training}. The formulation is 
\begin{gather}
\begin{aligned}
R(\theta, \psi) =\frac{\gamma}{2} \mathrm{E}_{(\hat{\boldsymbol x})}\left[\left\|\nabla_{\hat{\boldsymbol x}} D_{\psi}(\hat{\boldsymbol x})\right\|^{2}\right],
\end{aligned}
\end{gather}
where $\hat{\boldsymbol x}$ is sampled uniformly on the line segment between two random points vector $\boldsymbol x_{1} \sim p_{\theta}\left(\boldsymbol x_{1}\right), \boldsymbol x_{2} \sim p_{\mathcal{D}}\left(\boldsymbol x_{2}\right)$ like WGAN-GP.}

{
When integrating the gradient penalty into the GAN model, it is expressed as a regularization term following the loss function of the GAN. The formulation is represented as
\begin{gather}\label{general-d}
\mathop{\min_{G_{\theta}}}\mathop{\max_{D_{\psi}}} L_{GAN}(G_{\theta}, D_{\psi})=\mathbb{E}_{\boldsymbol z\sim  p_{z}}\left[f\left(D_{\psi}\left(G_{\theta}(\boldsymbol z)\right)\right)\right]+ 
\mathbb{E}_{\boldsymbol x\sim p_{*}}\left[f\left(-D_{\psi}(\boldsymbol x)\right)\right]+\lambda R(\theta, \psi),    
\end{gather}
where $\lambda$ controls the importance of the regularization, $R(\theta, \psi)$ denotes the gradient penalty term, $f$ represents a general function from \cite{gulrajani2017improved} with different forms in various GAN models. This specific type of loss function indicates that the gradient penalty can influence the variability of the discriminator's gradient, thereby impacting the generation of synthetic samples by the generator.
}

\noindent \textbf{Dynamic Theory for CFG.} 
\label{Dynamic for icfg}
We incorporate dynamic theory to gain a theoretical understanding of the equivalence between the CFG method and the common GAN theory.
However, the CFG method also faces challenges related to unstable training and the lack of local convergence near the Nash-equilibrium point.
A summary of the key outcomes is provided in Appendix~\ref{adtftcm}.

\section{Methodology}\label{gp}
\noindent \textbf{Overview.}
The main structure of this section is organized as follows:
In Section~\ref{relation between}, we 
present the definition of the latent N-size with the gradient penalty.
Guided by this, in Section~\ref{our theory}, we introduce our regularization termed as $\boldsymbol\varepsilon$-centered gradient penalty.
In Section~\ref{relation theory}, 
we present our main theorem to demonstrate the connection between the latent N-size and the various gradient penalties.
Proofs of these theorems are provided in Appendix~\ref{appendix proof}.

\subsection{Latent N-size with gradient penalty} 
\label{relation between}
In this section, our objective is to reveal the interrelation between the latent N-size and the gradient penalty. 
First, we will define the latent N-size, along with an intuitive explanation. 
Then, we will explain three basic definitions of the latent N-size by showing how it relates to the diversity of synthetic samples. 
Additionally, building on the previous step, we will discuss the expansion of the latent N-size by including the gradient penalty.

\noindent \textbf{Latent N-size.}
We present the definition of latent N-size, which forms the basis for the subsequent theory.
\begin{definition}[Latent Neighborhood Size]\label{definition-radius}
Let $\boldsymbol{z}_1$, $\boldsymbol{z}_2$ be two samples in the latent space. Suppose $\boldsymbol z_1$ is attracted to the mode $\mathcal{M}_{i}$ by $\hat{\epsilon}$, then there exists a neighborhood $\mathcal{N}_r\left(\boldsymbol{z}_1\right)$ of $\boldsymbol{z}_1$ such that $\boldsymbol{z}_2$ is distracted to $\mathcal{M}_{i}$ by $(\hat\epsilon/ 2-2\alpha)$, for all $\boldsymbol{z}_2 \in \mathcal{N}_r\left(\boldsymbol{z}_1\right)$. The size of $\mathcal{N}_r\left(\boldsymbol{z}_1\right)$ can be arbitrarily large but is bounded by an open ball of radius $r$. The $r$ is defined as
$$
  r=\hat\epsilon \cdot\left(2 \inf _{\boldsymbol{z}}\left\{\left(\frac{\left\|G_{\theta_t}\left(\boldsymbol z_1\right)-G_{\theta_t}(\boldsymbol z)\right\|}{\left\|\boldsymbol z_1-\boldsymbol z\right\|}+ \frac{\left\|G_{\theta_{t+1}}\left(\boldsymbol z_1\right)-G_{\theta_{t+1}}(\boldsymbol{z})\right\|}{\left\|\boldsymbol z_1-\boldsymbol{z}\right\|}\right)\right\}\right)^{-1},$$
\end{definition}
where the mode $\mathcal{M}$, $attracted$ and $distracted$ are defined in Definition.~\ref{definition1},~\ref{definition2},~\ref{definition3}, respectively.

According to this definition, the radius $r$ is inversely proportional to the discrepancy between the preceding and subsequent outputs of the generator, given a similar latent vector $\boldsymbol z$. 
A large value of $r$ results in a small difference between the previous and subsequent generator outputs, leading to mode collapse. 
Conversely, a small value of $r$ leads to a large difference, resulting in diverse synthesis.

\noindent \textbf{Latent N-size and the diversity.}
% We offer three fundamental definitions of the latent N-size and demonstrate its relationship to the diversity of synthetic samples. 
In this paragraph, we discuss the relationship between the latent N-size and
the gradient penalty.
To begin, let's discuss the above three definitions, the mode $\mathcal{M}$, $attracted$, and $distracted$, which play a crucial role in the mode collapse phenomenon~\cite{yang2019diversity}.  

Additionally, we propose implementing the gradient penalty in the discriminator to adjust the latent N-size and alleviate the mode collapse phenomenon.
If the neighborhood size is too large, a significant portion of the latent space vectors would be attracted to this specific image mode, leading to limited diversity in the synthetic samples. 
Conversely, if the neighborhood size is too small and contains only one vector, the latent space vectors cannot adequately cover all the modes in the image space.

The intuition behind our idea is illustrated in Fig.~\ref{Fig.mainidea}.
The top and bottom rows indicate the latent N-size for the discriminator with or without gradient penalty, respectively. 
The yellow line at the top indicates that a different latent vector is being drawn towards a new mode, distinct from $z_1$.
The blue line at the bottom row indicates that the same latent vector in the neighborhood of $\boldsymbol z_1$ is attracted to the same mode as $\boldsymbol z_1$. 
The top row represents improved sample diversity, while the bottom row indicates a mode collapse phenomenon.
\begin{wrapfigure}{r}{7cm}
\centering 
\includegraphics[width=0.5\textwidth]{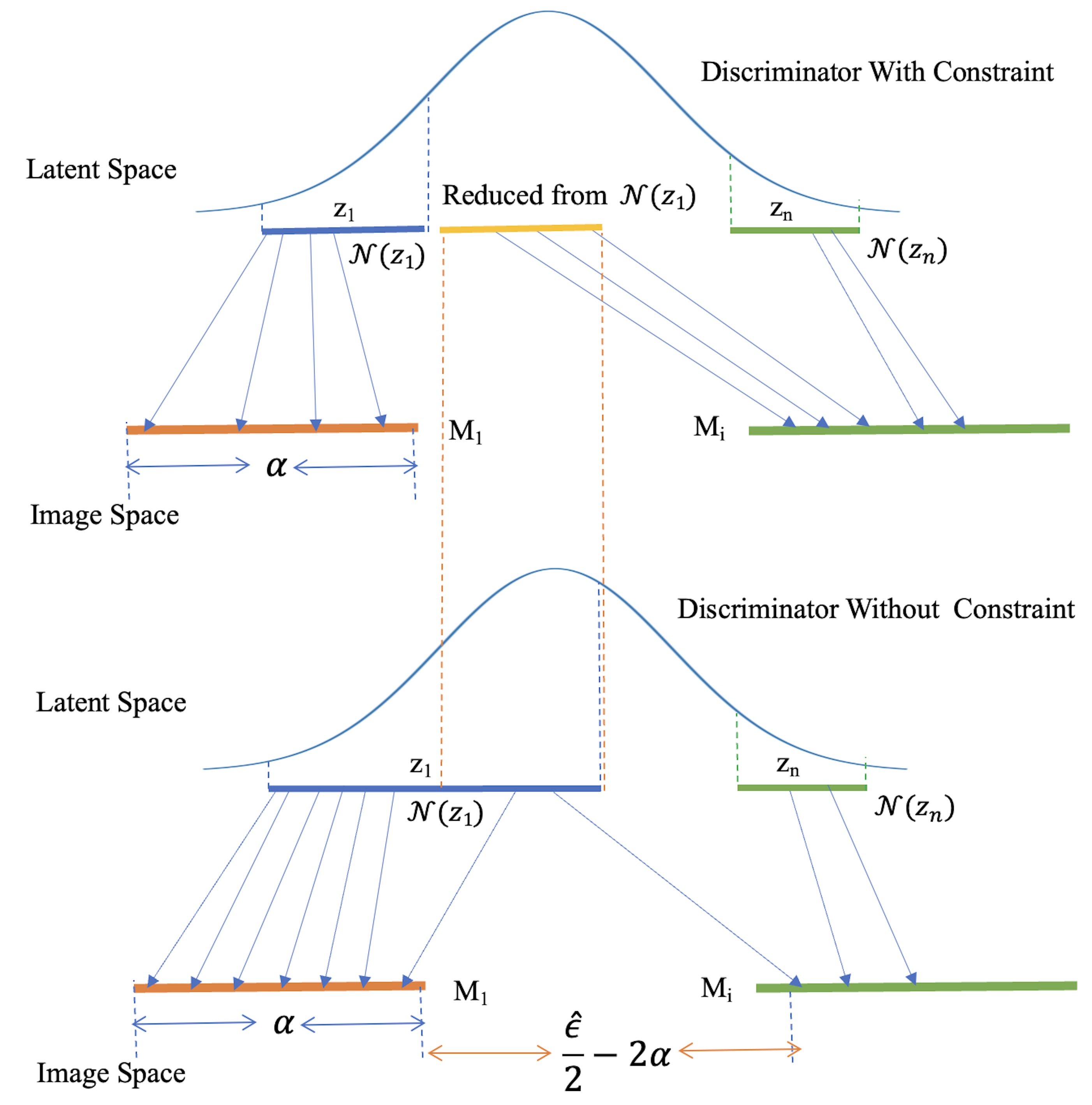}
\caption{{Main idea of the relationship between latent N-size and constraint .}\label{Fig.mainidea}}
\vspace{-0.3in}
\end{wrapfigure}

\begin{definition}[Modes in Image Space]\label{definition1}
There exist some modes $\mathcal{M}$ cover the image space $\mathcal{Y}$.
Mode $\mathcal{M}_{i} $ is a subset of $\mathcal{Y}$ satisfying $\max _{\boldsymbol{y_{k,j}} \in \mathcal{M}_{i}}\left\|\boldsymbol{y_{k}}-\boldsymbol{y_{j}}\right\|<\alpha$ and  
$\min_{\boldsymbol{y_{k}} \in \mathcal{M}_{i},\boldsymbol{y_{m}} \not\in \mathcal{M}_{i}}\alpha<\left\|\boldsymbol{y_{m}}-\boldsymbol{y_{k}}\right\|<2\alpha$, where $\boldsymbol{y_{k}}$ and $\boldsymbol{y_{j}}$ belong to the same mode $\mathcal{M}_{i}$,  $\boldsymbol{y_{m}}$ and $\boldsymbol{y_{k}}$ belong to different modes $\mathcal{M}_{i}$, and $\alpha>0$.
\end{definition}

Definition~\ref{definition1} asserts that images within the same mode exhibit minimal differences.
Conversely, images belonging to different modes exhibit more significant differences.
% \noindent \textbf{Definition II}
\begin{definition}[Modes Attracted]\label{definition2}
Let $\boldsymbol{z}_1$ be a sample in latent space, we say $\boldsymbol{z}_1$ is attracted to a mode $\mathcal{M}_{i}$ by $\hat{\epsilon}$ from a gradient step if $\left\|\boldsymbol{y_{k}}-G_{\theta_{t+1}}\left(\boldsymbol{z}_1\right)\right\|+\hat{\epsilon}<\left\|\boldsymbol{y_{k}}-G_{\theta_t}\left(\boldsymbol{z}_1\right)\right\|$, where $\boldsymbol{y_{k}} \in \mathcal{M}_{i}$ is an image in a mode $\mathcal{M}_{i}$, $\hat{\epsilon}$ denotes a small quantity, {$\theta_t$ and $\theta_{t+1}$} are the generator parameters before and after the gradient updates respectively.   
\end{definition}

Definition~\ref{definition2} establishes that a latent vector $\boldsymbol{z}_1$ is attracted to a specific mode $\mathcal{M}_{i}$ in the latent space.
As training progresses, the output corresponding to $\boldsymbol{z}_1$ will exhibit only minor deviations from images within that mode.

% \noindent \textbf{Definition III}
\begin{definition}[Modes Distracted]\label{definition3}
Let $\boldsymbol{z}_2$ be a sample in latent space, we say $\boldsymbol{z}_2$ is distracted from a mode $\mathcal{M}_{i}$ by $(\alpha+\hat{\epsilon})/2$ from a gradient step if $\left\|\boldsymbol{y_{m}}-G_{\theta_{t+1}}\left(\boldsymbol{z}_2\right)\right\|+(\hat\epsilon/ 2-2\alpha)<\left\|\boldsymbol{y_{k}}-G_{\theta_t}\left(\boldsymbol{z}_2\right)\right\|$, where $\boldsymbol{y_{k}} \in \mathcal{M}_{i}$ is an image in a mode $\mathcal{M}_{i}$, $\boldsymbol{y_{m}} \not\in \mathcal{M}_{i}$ is an image from other modes,  $\alpha$ keeps the same meaning as in {
Definition~\ref{definition1}, $\theta_t$ and $\theta_{t+1}$} 
are the generator parameters before and after the gradient updates respectively.
\end{definition}
Definition~\ref{definition3} explains that when a vector $\boldsymbol{z}_2$ close to $\boldsymbol{z}_1$ is drawn towards a particular mode in the image space, it is less likely to be attracted by a different mode. 
Therefore, it is crucial to decrease the latent N-size, as this encourages latent vectors to be attracted to various modes within the image space.

\noindent \textbf{Latent N-size with gradient penalty.}
We demonstrate the relationship between the latent N-size and the gradient penalty. 
According to Proposition~\ref{relation n d}, as $\|\nabla_{\boldsymbol x} D(\boldsymbol{x})\|$ increases, the latent N-size decreases, and vice versa.

\begin{proposition}\label{relation n d}
$\mathcal{N}_r\left(\boldsymbol{z}_1\right)$ can be defined with discriminator gradient penalty as follows: 
$$\small
r=\hat{\epsilon}\cdot\left(2 \inf _{\boldsymbol{z}}\left\{ \left(\frac{2\left\|G_{\theta_t}\left(\boldsymbol{z}_1\right)-G_{\theta_t}(\boldsymbol{z})\right\|+\eta_m \delta(\boldsymbol x) \sum\limits_{m=1}^{N}\left(\|\nabla_{\boldsymbol x} D_{m}(\mathcal{Y}_2)\|+\|\nabla_{\boldsymbol x} D_{m}(\mathcal{Y})\|\right)}{\left\|\boldsymbol{z}_1-\boldsymbol{z}\right\|}\right)\right\}\right)^{-1}
$$  
, where $\|\nabla_{\boldsymbol x} D_{m}(\mathcal{Y}_2)\|=$
$\|\nabla_{\boldsymbol x} D_{m}(G_{\theta_t}(\boldsymbol{z}_2))+R\|$ and   $\|\nabla_{\boldsymbol x} D_{m}(\mathcal{Y})\|=$
$\|\nabla_{\boldsymbol x} D_{m}(G_{\theta_t}(\boldsymbol{z}))+R\|$. $R$ stands for the discriminator gradient penalty.
\end{proposition}
We show that $G_{\theta_t}$ can be iterated computed as follows: $G_{\theta_t} = G_{\theta_{t-1}}+\eta_m\delta(\boldsymbol{x})\nabla_x D_m(G_{\theta_{t-1}}(\boldsymbol{z}))$. For the sake of simplicity, we present the expansion for the $t$-th term.

From Proposition~\ref{relation n d}, it's crucial to maintain the latent N-size within a specific range to generate diverse synthetic samples. 
To achieve this, it is necessary to control the magnitude of the gradient norm with a gradient penalty.

Based on the corollary $\nabla_{\boldsymbol{x}} D(\boldsymbol{x})\leq 0$, we propose subtracting a small value $\boldsymbol{\varepsilon}$ from $\nabla_{\boldsymbol{x}} D(\boldsymbol{x})$ to control the magnitude of the gradient norm.
This will effectively enhance the gradient norm, leading to a reduction in the latent N-size.
Additionally, we propose the $\boldsymbol{\varepsilon}$-centered gradient penalty based on the above insight.

\subsection{\texorpdfstring{$\boldsymbol\varepsilon$}{-}-centered GP}\label{our theory}
% \textcolor{red}{Add some comments}
We propose our $\boldsymbol\varepsilon$-centered gradient penalty in this section. We use notation $\boldsymbol\varepsilon$ in our penalty name and equation to differ from the hyper-parameter $\varepsilon'$.
 The $\boldsymbol{\varepsilon}$-centered GP is 
\begin{gather}
R(\theta, \psi)=\frac{\gamma}{2} \mathrm{E}_{\hat{\boldsymbol x}}\left(\left\|\nabla_{\boldsymbol x} D_{\psi}(\hat{\boldsymbol x})-\boldsymbol \varepsilon\right\|\right)^{2},\label{epsilon-center-GP}
\end{gather}
 where $\boldsymbol{\varepsilon}$ is a vector such that $\Vert \boldsymbol\varepsilon \Vert_{2}= \varepsilon'$
 with $\varepsilon'=\sqrt{C\cdot N^2\cdot \boldsymbol\varepsilon^{2}}$,  $N$ and $C$ are dimensions and channels of the real data respectively.
$\hat{\boldsymbol x}$ is sampled uniformly on the line segment between two random points vector $\boldsymbol x_{1} \sim p_{\theta}\left(\boldsymbol x_{1}\right), \boldsymbol x_{2} \sim p_{\mathcal{D}}\left(\boldsymbol x_{2}\right)$.
Combining the corollary $\nabla_{\boldsymbol{x}} D(\boldsymbol{x})\leq 0 $, our $\boldsymbol{\varepsilon}$-centered GP increases the $\|\nabla_{\boldsymbol{x}} D(\boldsymbol{x})\|$ as to achieve a better latent N-size which other two gradient penalty behaviors worse. 

When training the GAN model with the loss function  Eq.~(\ref{general-d}) and Eq.~(\ref{epsilon-center-GP}), our approach will control the gradient of the discriminator and result in a diversity of synthesized samples. 

\subsection{Latent N-size with different gradient penalties} 
 \label{relation theory}
In this section, we combine Definition~\ref{definition-radius}, Proposition~\ref{relation n d}, Lemma~\ref{lemma compare}, and our $\varepsilon$-centered gradient penalty to prove the main theorem.
First, we establish the relationship among the latent N-size with three gradient penalties in the following Lemma.
\begin{lemma}\label{lemma compare}
The norms of the three Gradient Penalties, which determine the latent N-size, are defined as follows: $\|R_1\|$ =  $\|\left(\left\|\nabla_{\boldsymbol x} D_{m}(G_{\theta_t}(\boldsymbol z))\right\|-g_0\right)\|$, $\|R_0\|=\|\left(\left\|\nabla_{\boldsymbol x} D_{m}(G_{\theta_t}(\boldsymbol z))\right\|\right)\| $, $ \|R_{\boldsymbol\varepsilon}\|$=$\|\left(\left\|\nabla_{\boldsymbol x} D_{m}(G_{\theta_t}(\boldsymbol z))\right\|+\|\boldsymbol\varepsilon\|\right)\|$, 
respectively. 
The order of magnitude between the norms of three Gradient Penalty is $\|R_1\|<\|R_0\|<\|R_{\boldsymbol\varepsilon}\|$. Consequently, the relationship between the latent N-size of three Gradient Penalty is $r_{R_1}>r_{R_0}>r_{R_{\boldsymbol\varepsilon}}$. 
\end{lemma}
Combining the Proposition~\ref{relation n d}, a larger value of $\|\nabla_{\boldsymbol{x}} D(\boldsymbol{x})\|$ will lead to a smaller latent N-size, thus enhancing the diversity of synthetic samples.

\begin{theorem}\label{theorem main}
Suppose $\boldsymbol z_1$ is attracted to the mode $\mathcal{M}_{i}$ by $\hat{\epsilon}$, then there exists a neighborhood $\mathcal{N}_r\left(\boldsymbol{z}_1\right)$ of $\boldsymbol{z}_1$ such that $\boldsymbol{z}_2$ is distracted to $\mathcal{M}_{i}$ by $(\hat\epsilon/ 2-2\alpha)$, for all $\boldsymbol{z}_2 \in \mathcal{N}_r\left(\boldsymbol{z}_1\right)$. The size of $\mathcal{N}_r\left(\boldsymbol{z}_1\right)$ can be arbitrarily large but is bounded by an open ball of radius $r$ where be controlled by Gradient Penalty terms of the discriminator. The relationship between the latent N-size corresponding to the three Gradient Penalty is $r_{R_1}>r_{R_0}>r_{R_{\boldsymbol\varepsilon}}$.
\end{theorem}
The theorem suggests that if a latent vector is pulled towards a specific mode in the image space, the size of the latent vector should be kept reasonable. 
If the latent N-size is excessively large, it could prevent latent vectors from attracting toward other modes, potentially causing the mode collapse phenomenon. 
The gradient penalty can be used to effectively adjust the latent N-size.
In descending order, the relationship between the latent N-size corresponding to the three Gradient Penalties can be summarized as follows: $r_{R_1} > r_{R_0} > r_{R_{\boldsymbol\varepsilon}}$.

\begin{figure*}
\centering
\includegraphics[width=0.85\textwidth]{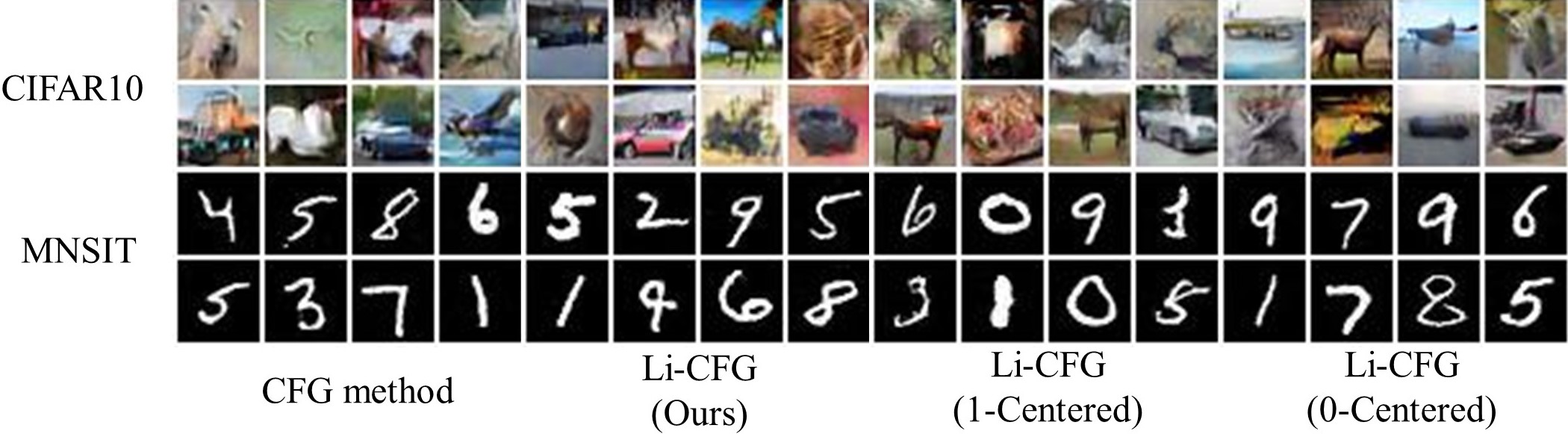} 
\caption{Results for CIFAR10, MNIST: The most left four columns are CFG method, the second four columns are Li-CFG with $\boldsymbol\varepsilon$-centered GP(ours), the third four columns are Li-CFG with $1$-centered, the most right four columns are  Li-CFG with $0$-centered. \label{Fig.main6} 
}
\end{figure*}

\begin{figure*}
\vspace{-0.1in}
\centering
\includegraphics[width=1\textwidth]{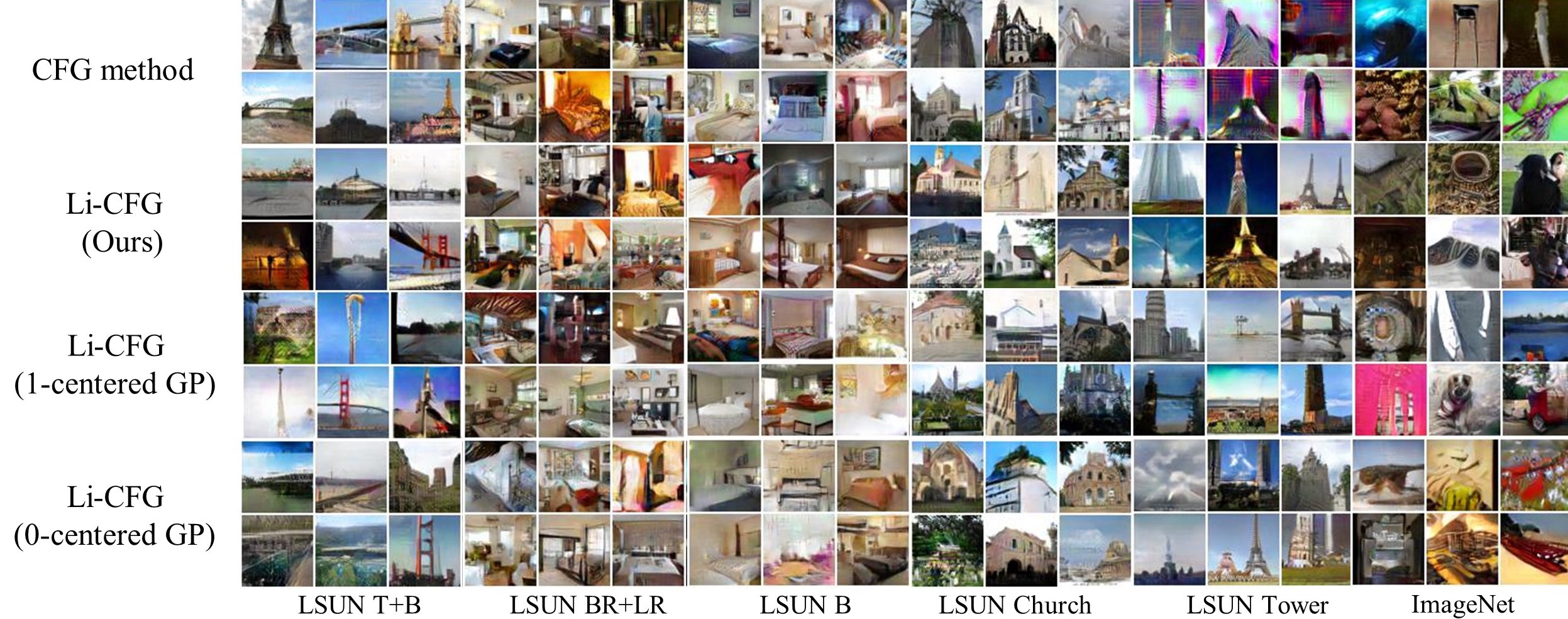}
\vspace{-0.1in}
\caption{Results for LSUN tower, Church, B, BR+LR, T+B, ImageNet: The method from top to bottom are CFG method, Li-CFG with $\boldsymbol\varepsilon$-centered GP(ours), Li-CFG with $1$-centered and Li-CFG with $0$-centered in each two rows.
\label{Fig.main7}
}
\vspace{-0.1in}
\end{figure*}

\section{Related Work}

\noindent \textbf{Generative Adversarial Network}.
GAN is optimized as a discriminator and a generator in a minimax game formulated as~\cite{goodfellow2014generative}. 
There are various variants of GAN for image generation, such as DCGAN~\cite{radford2015unsupervised},
SAGAN~\cite{zhang2019self}, Progressive Growing GAN~\cite{karras2017progressive}, BigGAN~\cite{brock2018large} and StyleGAN~\cite{karras2019style}. In general, the GANs are very difficult to be stably trained. Training GAN may suffer from various issues, including gradients vanishing, mode collapse, and so on~\cite{che2016mode,roth2017stabilizing,nowozin2016f}.
Numerous excellent works have been done in addressing these issues.
For example, by using the Wasserstein distance,  WGAN~\cite{arjovsky2017wasserstein} and its extensions
~\cite{gulrajani2017improved,nowozin2016f,nagarajan2017gradient,mescheder2017numerics,mescheder2018training} improve the training stability of GAN.

\noindent \textbf{Lipschitz Constrained for GANs}.
Applying Lipschitz constraint to CNNs have been widely explored~\cite{oberman2018lipschitz,scaman2018lipschitz,zhou2018understanding,zhou2019lipschitz,herrera2020estimating,kim2021lipschitz}. 
Such an idea of Lipschitz constraint has also been introduced in Wasserstein GAN (WGAN). 
Recently theoretical results proposed by~\citet{kim2021lipschitz} show that self-attention which is widely used in the transformer model does not meet the Lipschitz continuity. 
It has been proved in ~\citet{zhou2018understanding} that the Lipschitz-continuity is a general solution to make the gradient of the optimal discriminative function reliable. Unfortunately, directly applying the Lipschitz constraint to complicated neural networks is not easy. Previous works typically employ the  
 mechanisms of gradient penalty and weight normalization. The gradient penalty proposed by ~\citet{gulrajani2017improved} adds a function after the loss function to control the gradient varying of the discriminator. ~\citet{mescheder2018training} and ~\citet{nagarajan2017gradient} argue that the WGAN-GP is not stable near Nash-equilibrium and propose a new form of gradient penalty.
Weight normalization or spectrum normalization was first studied by 
~\citet{miyato2018spectral}. The normalization constructs a layer of neural networks to control the gradient of the discriminator. Recently, 
~\citet{bhaskara2022gran,wu2021gradient} propose a new form of normalization behaviors better than the spectrum normalization.

\noindent \textbf{Compare our method with existing methods}.
Our method appears similar to AdvLatGAN-div~\cite{li2022improving} and MSGAN~\cite{mao2019mode}, aiming to increase the pixel space distance ratio over latent space. 

{We have clarified the differences between our method and AdvLatGAN-div and MSGAN as follows: Firstly, we note that MSGAN, AdvLatGAN, and our method are related to Eq.~(\ref{diversz}).}
\begin{equation}
    \begin{aligned}  \left(\frac{d_{\boldsymbol{I}}\left(G\left(\boldsymbol{c}, \boldsymbol{z}_1\right), G\left(\boldsymbol{c}, \boldsymbol{z}_2\right)\right)}{d_{\boldsymbol{z}}\left(\boldsymbol{z}_1, \boldsymbol{z}_2\right)}\right)\end{aligned},\label{diversz}
\end{equation}
{
where $\boldsymbol{c}$ is the condition vector for generation, and $d_{\boldsymbol{I}}$ and $d_{\boldsymbol{z}}$ are the distance metrics in the target (image) space and latent space, respectively.}

{
Secondly, MSGAN maximizes Eq.~(\ref{diversz}) to train the generator, encouraging it to synthesize more diverse samples. Essentially, maximizing Eq.~(\ref{diversz}) means that for two different samples $\boldsymbol{z_1}$ and $\boldsymbol{z_2}$ from the latent space, the distance between their corresponding synthesized outputs $G(\boldsymbol{z_1})$ and $G(\boldsymbol{z_2})$ in the target space should be as large as possible. This leads to an increase in the diversity of the generated samples.
The AdvLatGAN-div method searches for pairs of $z_t$ that match hard sample pairs, which are more likely to collapse in image space. 
This search process is achieved by iteratively optimizing the latent samples $\boldsymbol{z}$ to minimize Eq.~(\ref{diversz}), similar to generating adversarial examples.
Then, the objective is to maximize Eq.~(\ref{diversz}) by using these hard sample pairs $z_t$ in the latent space as input again to optimize the generator. 
This process helps to generate diverse samples in the image space and avoid mode collapse.
Both MSGAN and AdvLatGAN-div, which utilize Eq.~(\ref{diversz}), are based on empirical observations and do not have a solid theoretical basis.}

{Unlike the previous two methods, we don't employ Eq.~(\ref{diversz}) as an additional loss function or an optimal target for training our model. Instead, we use Eq.~(\ref{diversz}) as the fundamental explanation of our theory. We have improved upon Eq.~(\ref{diversz}) by using the CFG method. This has helped us establish a connection between the distance of various $\boldsymbol{z}$ from the latent space and the discriminator's constraints. We have also explained how the Lipschitz constraint of the discriminator can enhance the generator's diversity based on our new theory.}

\section{Experiment}
\noindent \textbf{Datasets}. {
To demonstrate the efficiency of our approach, we use both synthetic datasets and real-world datasets.
In line with \cite{metz2016unrolled,li2021iid}, we simulate two synthetic datasets: (1) Ring
dataset. It consists of a mixture of eight 2-D Gaussians
with mean ${(2\cos(i\pi/4),2\sin(i\pi/4))},1\leq i\leq8$
and standard deviation 0.02. (2) Grid dataset. It consists of
a mixture of twenty-five 2-D isotropic Gaussians with mean ${(2i,2j)}, -2\leq i,j\leq2$
and standard deviation 0.02. 
For real-world datasets, we used MNIST~\cite{lecun1998gradient}, CIFAR10~\cite{krizhevsky2009learning},
% the Street View House Numbers dataset (SVHN)\citet{netzer2011reading}, 
and the large-scale scene understanding (LSUN)~\cite{yu2015lsun} and ImageNet~\cite{deng2009imagenet} dataset, to make fair comparison to original CFG method. Note that we for the first time present the ImageNet results of CFG methods.
The CFG method employs a balanced two-class dataset using the same number of training images from the ‘bedroom’ class and the ‘living room’ class (LSUN BR+LR) and a balanced dataset from ‘tower’ and ‘bridge’ (LSUN T+B). Besides, We choose to use the 'church'(LSUN C) and the LSUN tower(LSUN T) to do our experiment because the CFG method does not very well in these datasets.
We choose CIFAR10 and ImageNet to demonstrate the generalization performance of our algorithm on richer categories and larger-scale datasets.
We have included the high-resolution qualitative results of various datasets in the supplementary section.}
% CelebA~\cite{liu2015faceattributes}, LSUN Church ($256\times 256$)  
% LSUN Bedroom ($256\times 256$)  and ImageNet ($256\times 256$) in the supplementary.}

\noindent \textbf{Implementation Details.}\label{Implementation Details}
To ensure a fair comparison between our method and the CFG method, we maintain identical implementation settings. 
All the experiments were done using a single NVIDIA Tesla V100 or a single NVIDIA Tesla A100~\footnote{The computations in this research were performed using the CFFF platform of Fudan University}.
The hyper-parameter values for the CFG method were fixed to those in Table.~\ref{table:meta} unless otherwise specified.
CFG method behaves sensitively to the setting of $\delta(\boldsymbol x)$ for generating the image to approximate an appropriate value. 
In our work, we give the same $\delta(\boldsymbol x)$ settings with the CFG method; The hyper-parameter $\gamma$ is set to $0.1$, $1$ or $10$; $\boldsymbol\varepsilon$-centered hyper-parameters $\varepsilon'$ is set to $0.1$, $0.3$, or $1$ in the experiments. 

The explanation and ablation study of $\varepsilon'$ are in the supplementary. 
The base setting is presented in Table.~\ref{table:meta} and Table.~\ref{table:detail settings}. 

\noindent \textbf{Baselines.}
{
As a representative of comparison methods, we tested WGAN with the gradient penalty (WGAN-GP)~\cite{gulrajani2017improved}, the least square GAN~\cite{mao2017least}, the origin GAN~\cite{goodfellow2014generative} and the HingeGAN~\cite{lim2017geometric}. 
All of them always have been the baseline in other GAN models. For a fair comparison, we utilize the same network architecture as the CFG method. We have utilized the backbones of both simple  DCGAN~\cite{radford2015unsupervised}, and complex Resnet~\cite{he2016deep} in each dataset.
To evaluate the generalization of our $\boldsymbol\varepsilon$-centered method, we incorporate SOTA GAN models, such as BigGAN~\cite{brock2018large} and the Denoising Diffusion GAN (DDGAN)~\cite{xiao2021tackling}, with our $\boldsymbol\varepsilon$-centered gradient penalty.
}

\begin{wraptable}{r}{8.5cm} \small
\vspace{-0.3in}
\caption{\centering hyper-parameters.\label{table:meta}}
\begin{tabular}{cccc}
\hline
\multicolumn{1}{c}{\bf NAME}  &\multicolumn{1}{c}{\bf DESCRIPTION}\\
 \hline 
B=64         &training data batch size\\
U=1            &discriminator update per epoch\\
N=640             &examples for updating G per epoch\\
M=15             &number of generate step in CFG method\\
$\gamma$=0.1/1/10             &a hyper-parameters for L constant\\
$\varepsilon'$=0.1/0.3/1 & a hyper-parameters for $\boldsymbol\varepsilon$-centered\\
\hline
\end{tabular}
\vspace{-0.2in}
\end{wraptable}

\noindent \textbf{Evaluation Metrics.} %\label{Evaluation Metrics}
{
Generative adversarial models are known to be a challenge to make reliable likelihood estimates.
So we instead evaluated the visual quality of generated images by adopting the inception score~\cite{salimans2016improved}, the Fr\'{e}chet inception distance~\cite{heusel2017gans} and the Precision/Recall~\cite{kynkaanniemi2019improved}. 
The intuition behind the inception score is that high-quality generated images should lead to close to the real image. The Fr\'{e}chet inception distance indicates the similarity between the generated images and the real image. Moreover, by definition, precision is the fraction of the generated images that are
realistic, and recall is the fraction of real images within the manifold covered by the generator.}
\begin{wraptable}{r}{8.5cm} \small
% \vspace{-0.2in}
\caption{\centering Hyper-parameters. The $\delta(\boldsymbol x)$ value is the same as the CFG method. $\gamma$ is the parameter for the GP coefficient. \label{table:detail settings}}
\centering  
\begin{tabular}{lccll}  
\hline
\multicolumn{1}{c}{\bf DATASET}  &\multicolumn{1}{c}{\bf $\eta$} &\multicolumn{1}{c}{\bf $\delta(\boldsymbol x)$}&\multicolumn{1}{c}{\bf $\gamma$}&\multicolumn{1}{c}{\bf $\varepsilon'$} \\
\hline 
MNIST &2.5e-4 & 0.5 & 0.1/1/10& 0.1/0.3/1\\
% SVHN            &1& 0.25 & 1/10 \\
CIFAR             &2.5e-4& 1 & 0.1/1/10& 0.1/0.3/1 \\
LSUN         &2.5e-4& 1& 0.1/1/10&0.1/0.3/1\\
ImageNet         &2.5e-4& 1& 0.1/1/10&0.1/0.3/1\\
\hline 
\end{tabular}
\vspace{-0.2in}
\end{wraptable}
We note that the inception score is limited, as it fails to detect mode collapse or missing modes. Apart from that, we found that it generally corresponds well to human perception.

In addition, we used Fr\'{e}chet inception distance (FID). FID measures the distance between the distribution of $f(x_*)$ for real data $x_*$and the distribution of $f(x)$ for generated data $x$, where function $f$ extract the feature of an image.  One advantage of this metric is that it would be high (poor) if mode collapse occurs, and a disadvantage is that its computation is relatively expensive.
However, FID does not differentiate the fidelity and diversity aspects of the generated images, so we used the Precision/Recall to diagnose these properties. A high precision value indicates a high quality of generated samples, and a high recall value implies that the generator can generate many realistic samples that can be found in the "real" image distribution.
In the results below, we call these metrics the (inception) score, the Fr\'{e}chet distance, and the Precision/Recall.

\subsection{Experimental Results}
In this section, we present a detailed comparison between our Li-CFG experiments and three gradient penalty (GP) methods within the context of CFG and other GAN-based models.
we report the results of our model using the same hyper-parameters against the CFG method. Our comparison shows that our method achieves better results. Furthermore, we demonstrate that the $\boldsymbol\varepsilon$-centered gradient penalty is versatile and can be effectively integrated into various GAN architectures.
\begin{table*}[htb!]
\vspace{-0.1in}
\caption{\centering 
{To assess the generalization of our $\boldsymbol\varepsilon$-centered gradient penalty, we apply it to various GANs baselines on the various datasets. When compared to the original GAN, WGAN, LSGAN, and HingeGAN, each utilizing different gradient penalties, our $\boldsymbol\varepsilon$-centered gradient penalty consistently achieves the best Fréchet Inception Distance (FID) score and the Inception Score (IS). This outcome serves as evidence that our $\boldsymbol\varepsilon$-centered gradient penalty is not only applicable to the CFG mechanism but can also be effectively employed in common GAN models. The red font number indicates a correction made in the previous manuscript.}\label{table:isfid}
}
\centering
\resizebox{\textwidth}{75mm}{
\begin{tabular}{ccccccccc}
\hline \multicolumn{6}{c}{IS $\uparrow$ / FID$\downarrow$}                           \\
             \textbf{MNIST}               &    \multicolumn{1}{c}{origin GAN }       & \multicolumn{1}{c}{WGAN}                      & \multicolumn{1}{c}{LSGAN}                      & \multicolumn{1}{c}{HingeGAN}   & \multicolumn{1}{c}{Li-CFG}\\
\hline
Unconstrained          &  2.18/34.28       &   2.23/31.05    &     2.22/36.83                                       &   2.21/30.72                                &  2.29/4.04                                   \\
ours($\boldsymbol\varepsilon$-centered)           &   \textbf{2.21/25.62}        &    \textbf{2.24/23.79}     &      \textbf{2.24/28.54}                                        &   \textbf{2.23/24.65}           &  \textbf{2.31/3.29}     \\
% \textbf{SVHN}              & \textbf{0.913}            & \textbf{0.87}&\textbf{1.39}  &    &   \textbf{1.41} &                                   \\
$0$-centered           &2.21/30.28         & 2.22/24.98   &    2.2/27.32                                        &    2.19/26.89   & 2.31/3.54                           \\

$1$-centered          & 2.19/31.48          & 2.20/31.27  &    2.23/29.23                                     &   2.22/24.72    & 2.3/3.64                                                   \\                            
\hline 
 \textbf{CIFAR10}               &
 \\
\hline
Unconstrained          &   3.48/37.83        &    3.53/36.24    &      3.41/30.94                                       &   3.58/32.31         &    4.02/19.41\\
ours($\boldsymbol\varepsilon$-centered)           &   \textbf{ 3.82/30.52}        &    \textbf{ 3.87/28.45}     &      \textbf{ 3.92/24.73}                                        &   \textbf{ 4.01/20.63}         &   \textbf{4.83/14.96}      \\
% \textbf{SVHN}              & \textbf{0.913}            & \textbf{0.87}&\textbf{1.39}  &    &   \textbf{1.41} &                                   \\
$0$-centered           & 3.76/35.72         &  3.85/29.21   &     3.84/29.01                                        &    3.89/23.95       & 4.72/18.5          \\

$1$-centered          &  3.78/31.48          &  3.75/31.27  &     3.80/29.23                                     &     3.82/24.72   & 4.71/19.3       \\        
\hline 
 \textbf{LSUN B}               &
 \\
\hline
Unconstrained          &   2.35/19.28       &   2.38/18.72    &       2.31/19.53                                      &   2.42/17.85         &   3.014/10.79  \\
ours($\boldsymbol\varepsilon$-centered)           &   \textbf{2.58/17.38}        &    \textbf{2.61/16.98}     &      \textbf{2.67/16.21}                                        &   \textbf{2.73/15.43}                                    &    \textbf{8.78}                                         \\
% \textbf{SVHN}              & \textbf{0.913}            & \textbf{0.87}&\textbf{1.39}  &    &   \textbf{1.41} &                                   \\
$0$-centered           &2.55/18.72         & 2.59/17.88  &   2.65/17.42                                       &   2.68/15.96   & 3.067/10.54     \\

$1$-centered          & 2.56/18.89         & 2.57/18.29  &    2.66/17.69                                     &   2.65/16.83   & 3.154/11.5   \\                            
 
\hline 
 \textbf{LSUN T}               &
 \\
\hline
Unconstrained          &   3.59/25.87        &   3.67/22.76     &      3.52/28.14                                       &  3.63/23.57 &   4.38/13.54        \\
ours($\boldsymbol\varepsilon$-centered)           &   \textbf{3.93/20.62}        &    \textbf{4.06/17.89}     &      \textbf{3.89/22.67}                                        &   \textbf{4.02/18.72}                                  &    \textbf{4.6/11.41}                       \\
% \textbf{SVHN}              & \textbf{0.913}            & \textbf{0.87}&\textbf{1.39}  &    &   \textbf{1.41} &                                   \\
$0$-centered           &3.87/21.65         & 4.02/18.03 &   3.83/23.88                                       &    3.98/20.23   & 4.57/12.33      \\

$1$-centered          & 3.85/22.03          & 3.99/18.49  &   3.81/24.33                                    &   3.94/20.81   &4.47/12.81     \\                            

\hline 
 \textbf{LSUN C}               &
 \\
\hline
Unconstrained          &   2.18/34.69        &   2.25/32.31    &       2.3/31.47                                        &   2.43/30.77   &   3.17/11.49              \\
ours($\boldsymbol\varepsilon$-centered)           &   \textbf{2.51/29.48}        &    \textbf{2.63/26.46}     &      \textbf{2.61/27.82}                                        &   \textbf{2.69/25.58}                           &    \textbf{3.18/9.38}          \\
% \textbf{SVHN}              & \textbf{0.913}            & \textbf{0.87}&\textbf{1.39}  &    &   \textbf{1.41} &                                   \\
$0$-centered           &2.49/30.13         & 2.60/27.86   &    2.57/28.76                                       &    2.67/26.63  & 3.08/12.33     \\

$1$-centered          & 2.45/30.79          & 2.58/28.79  &    2.53/29.15                                     &   2.65/ 26.97    & 3.08/12.81   \\                            

\hline 
 \textbf{LSUN B+L}               &
 \\
\hline
Unconstrained          &   3.18/21.92       &    3.01/23.47    &      3.12/{25.93}                                      &   3.29/20.33    &   3.49/{15.76}                           \\
ours($\boldsymbol\varepsilon$-centered)           &   \textbf{3.40/18.39}        &    \textbf{3.38/19.34}     &      \textbf{3.31/21.82}                                        &   \textbf{ 3.42/17.87}          &    \textbf{3.50/14.09}                              \\
% \textbf{SVHN}              & \textbf{0.913}            & \textbf{0.87}&\textbf{1.39}  &    &   \textbf{1.41} &                                   \\
$0$-centered           &3.36/19.12         &3.37/19.98   &    3.29/22.51                                        &    3.38/18.63 & 3.47/15.39        \\

$1$-centered          &3.30/19.87         & 3.35/20.65  &    3.26/23.32                                     &    3.31/19.48 & 3.42/15.53     \\                            
\hline  
 \textbf{LSUN T+B}               &
 \\
\hline
Unconstrained          &   4.65/27.61        &   4.77/25.4    &      4.68/25.87                                       &  4.81/24.66 &   5.01/16.9             \\
ours($\boldsymbol\varepsilon$-centered)           &   \textbf{ 4.78/23.72}        &    \textbf{ 4.85/21.84}     &      \textbf{ 4.89/21.9}                                        &   \textbf{ 4.92/19.83}            &    \textbf{5.07/15.72}                                     \\
% \textbf{SVHN}              & \textbf{0.913}            & \textbf{0.87}&\textbf{1.39}  &    &   \textbf{1.41} &                                   \\
$0$-centered           &4.73/24.19         & 4.81/22.98   &    4.85/23.4                                        &    4.85/21.31  & 5.05/16.01     \\

$1$-centered          & 4.74/24.55          & 4.76/23.59  &   4.79/24.02                                     &    4.80/22.76   & 5.04/16.32       \\                            
\hline 
 \textbf{ImageNet}               &
 \\
\hline
Unconstrained          &   6.83/54.68        &   6.94/48.54     &      7.00/53.7                                       &  7.13/51.28 &    8.98/29.73          \\
ours($\boldsymbol\varepsilon$-centered)           &   \textbf{ 7.38/45.84}        &    \textbf{7.69/40.72}     &      \textbf{7.52/41.25}                                        &   \textbf{7.66/40.36}           &    \textbf{9.15/29.05}       \\
% \textbf{SVHN}              & \textbf{0.913}            & \textbf{0.87}&\textbf{1.39}  &    &   \textbf{1.41} &                                   \\
$0$-centered           &7.25/46.13         & 7.48/42.88   &   7.49/41.93                                        &    7.58/41.65          & 8.79/30.04                                                 \\

$1$-centered          & 7.16/47.28         & 7.21/44.71  &    7.44/42.84                                     &    7.47/42.37   & 8.96/29.55     \\                            
                        
\hline 
\end{tabular}}
\vspace{-0.1in}
\end{table*}

\noindent \textbf{Inception Score Results.}
% \subsubsection{Inception Score Results}
Table.~\ref{table:isfid} presents the Inception Score (IS) values for different GAN models.
Since the exact codes and models used to compute IS scores in~\citet{johnson2018composite} are not publicly available, we employed the standard PyTorch implementation of the Inception Score function~\footnote{https://github.com/sbarratt/inception-score-pytorch}.
The measured IS scores for the real datasets are as follows: 2.58(MNIST), 9.56(CIFAR10), 4.78(LSUN T), 3.72(LSUN B), 3.72(LSUN B+L), 3.79(LSUN T+B),  5.8(LSUN C), 37.99(ImageNet). 
We can find that the CFG method and Li-CFG scores are very close, and the WGAN and LSGAN performance is not so good in all the datasets.
The relative differences also shows that Li-CFG archives a better image quality than the CFG method in the same dataset.

\begin{table*}[htb!]\small
\vspace{-0.1in}
\caption{\centering 
{To assess the generalization of our $\boldsymbol\varepsilon$-centered gradient penalty, we apply it to various GANs baselines on the various datasets. When compared to the original GAN, WGAN, LSGAN, and HingeGAN baselines, each utilizing different gradient penalties, our $\boldsymbol\varepsilon$-centered gradient penalty consistently achieves the best Recall score. This outcome serves as evidence that our $\boldsymbol\varepsilon$-centered gradient penalty is not only applicable to the CFG mechanism but can also be effectively employed in common GAN models.}\label{table:pr}
}
\centering
\resizebox{\textwidth}{75mm}{
\begin{tabular}{ccccccccc}
\hline
\multicolumn{6}{c}{Precision/Recall$\uparrow$} \\  
 \textbf{CIFAR10}                &    \multicolumn{1}{c}{origin GAN }       & \multicolumn{1}{c}{WGAN}                      & \multicolumn{1}{c}{LSGAN}                      & \multicolumn{1}{c}{HingeGAN}   & \multicolumn{1}{c}{Li-CFG}\\
\hline
Unconstrained          &    0.78/0.45       &    0.78/0.41    &       0.80/0.49                                     &   0.81/0.47    &    0.77/0.59                              \\
ours($\boldsymbol\varepsilon$-centered)           &   \textbf{0.77/0.56}        &   \textbf{0.79/0.56}     &      \textbf{0.79/0.55}                                        &   \textbf{0.78/0.59}   &    \textbf{0.75/0.66}          \\
% \textbf{SVHN}              & \textbf{0.913}            & \textbf{0.87}&\textbf{1.39}  &    &   \textbf{1.41} &                                   \\
$0$-centered           &0.77/0.53        & 0.79/0.53   &    0.78/0.53                                        &    0.78/0.58                & 0.76/0.64  &                 \\

$1$-centered          & 0.78/0.54          &  0.78/0.52   &   0.78/0.51                                     &   0.79/0.57                    & 0.77/0.62                         \\        
\hline 
 \textbf{LSUN B}               &
 \\
\hline
Unconstrained          &  0.75/0.37         &    0.75/0.39    &      0.77/0.38                                     &   0.73/0.4   &   0.65/0.49          \\
ours($\boldsymbol\varepsilon$-centered)           &   \textbf{0.71/0.43}        &    \textbf{0.69/0.47}     &      \textbf{0.67/0.46}                                        &   \textbf{0.69/0.47}  &    \textbf{0.61/0.53}                             \\
% \textbf{SVHN}              & \textbf{0.913}            & \textbf{0.87}&\textbf{1.39}  &    &   \textbf{1.41} &                                   \\
$0$-centered           &0.73/0.41         &0.71/0.45   &    0.67/0.45                                        &    0.71/0.47                   & 0.62/0.51              \\

$1$-centered          & 0.74/0.41          & 0.71/0.44  &   0.69/0.44                                     &    0.73/0.45                 & 0.64/0.50                \\                            
 
\hline 
 \textbf{LSUN T}               &
 \\
\hline
Unconstrained          &   0.80/0.45        &    0.79/0.49    &       0.81/0.43                                        &   0.79/0.47           &    0.71/0.58                          \\
ours($\boldsymbol\varepsilon$-centered)           &   \textbf{0.76/0.49}        &    \textbf{0.72/0.56}     &      \textbf{0.78/0.45}                                        &   \textbf{0.74/0.53}                                    &    \textbf{0.62/0.65 }      \\
% \textbf{SVHN}              & \textbf{0.913}            & \textbf{0.87}&\textbf{1.39}  &    &   \textbf{1.41} &                                   \\
$0$-centered           &0.77/0.48        & 0.73/0.54   &    0.79/0.44                                        &    0.75/0.51                & 0.64/0.63                 \\

$1$-centered          & 0.75/0.48          & 0.76/0.53  &    0.78/0.43                                    &   0.74/0.52                    & 0.69/0.62               \\                            

\hline 
 \textbf{LSUN C}               &
 \\
\hline
Unconstrained          &    0.82/0.36        &     0.84/0.32   &      0.83/0.35                                    &    0.79/0.41          &    0.75/0.51                               \\
ours($\boldsymbol\varepsilon$-centered)           &   \textbf{0.79/0.43}        &    \textbf{0.80/0.45}     &      \textbf{ 0.81/0.44}                                        &   \textbf{0.78/0.47}  &    \textbf{0.72/0.58}                             \\
% \textbf{SVHN}              & \textbf{0.913}            & \textbf{0.87}&\textbf{1.39}  &    &   \textbf{1.41} &                                   \\
$0$-centered           &0.78/0.42         & 0.81/0.44  &    0.82/0.43                                        &    0.79/0.45                 &0.76/0.55             \\

$1$-centered          & 0.79/0.41         & 0.82/0.42 &     0.81/0.41                                     &     0.79/0.42                      & 0.77/0.54          \\                            

\hline 
 \textbf{LSUN B+L}               &
 \\
\hline
Unconstrained          &  0.81/0.46       &   0.85/0.43    &       0.86/0.42                                       & 0.81/0.47            &    0.77/0.51                            \\
ours($\boldsymbol\varepsilon$-centered)           &   \textbf{0.79/0.51}        &    \textbf{0.79/0.49}     &      \textbf{0.83/0.46}                                        &   \textbf{0.78/0.51}  &     \textbf{0.72/0.61}                              \\
% \textbf{SVHN}              & \textbf{0.913}            & \textbf{0.87}&\textbf{1.39}  &    &   \textbf{1.41} &                                   \\
$0$-centered           &0.79/0.49         & 0.80/0.47   &    0.85/0.45                                       &    0.80/0.48               & 0.74/0.59                \\

$1$-centered          & 0.77/0.48          & 0.82/0.46  &     0.86/0.43                                     &   0.80/0.49                 &  0.74/0.58   
\\                            
\hline  
 \textbf{LSUN T+B}               &
 \\
\hline
Unconstrained          &   0.82/0.38         &   0.82/0.39    &      0.84/0.37                                       &   0.81/0.42         &    0.78/0.46                              \\
ours($\boldsymbol\varepsilon$-centered)           &   \textbf{0.80/0.43}        &    \textbf{0.78/0.45}     &      \textbf{0.79/0.45}                                        &   \textbf{0.78/0.46}  &    \textbf{0.74/0.53}                               \\
% \textbf{SVHN}              & \textbf{0.913}            & \textbf{0.87}&\textbf{1.39}  &    &   \textbf{1.41} &                                   \\
$0$-centered           &0.83/0.41         & 0.79/0.43   &    0.81/0.42                                        &    0.79/0.44       & 0.75/0.52             \\

$1$-centered          & 0.83/0.40         & 0.81/0.42  &    0.81/0.41                                      &    0.8/0.44                     & 0.75/0.51              \\                            
\hline 
 \textbf{ImageNet}               &
 \\
\hline
Unconstrained          &    0.79/0.39         &    0.72/0.45     &      0.76/0.41                                       &   0.76/0.42           &     0.62/0.61                                   \\
ours($\boldsymbol\varepsilon$-centered)           &   \textbf{0.73/0.47}        &    \textbf{0.69/0.51}     &      \textbf{0.7/0.49}                                        &   \textbf{0.7/0.51}   &    \textbf{0.61/0.62}                            \\
% \textbf{SVHN}              & \textbf{0.913}            & \textbf{0.87}&\textbf{1.39}  &    &   \textbf{1.41} &                                   \\
$0$-centered           &0.75/0.46         & 0.7/0.47   &    0.72/0.47                                        &    0.71/0.50                 & 0.62/0.62                 \\

$1$-centered          & 0.76/0.45         & 0.72/0.46  &    0.72/0.46                                       &    0.71/0.49                      &  0.61/0.61                \\                            
                        
\hline 
\end{tabular}}
\vspace{-0.1in}
\end{table*}

\noindent \textbf{Fr\'{e}chet Distance Results.}\label{FID-explain}
We compute the Fr\'{e}chet Distance with 50k generative images and all real images from datasets with the standard implementation~\footnote{https://github.com/mseitzer/pytorch-fid}. 
However, we observed that the FID scores computed in our environment were significantly higher than those reported for the CFG method, even though we generated the images using the same CFG technique. The difference in FID scores suggest potential differences in environmental factors or implementation details.
The results, presented in Table.~\ref{table:isfid},  show that Li-CFG archives the best in the MNIST, CIFAR10, and the LSUN T, T+B, and B+L datasets. The WGAN and the LSGAN exhibit consistently weaker performance. 
The reason is that, for a fair comparison with other methods, we do not use tuning tricks, and these methods are also sensitive to varying hyper-parameters.
\begin{table*}[htb!]\small
\vspace{-0.1in}
\caption{\centering Different $\varepsilon'$ settings of our Li-CFG trained in MNIST. We use the FID and IS scores to compare the generated effect. The other two penalties do not have the parameter $\varepsilon'$ so that all the cells fill the same value. Untrained means the loss function does not converge.\label{table:ablation}
}
\centering
\setlength{\tabcolsep}{0.75mm}{
\begin{tabular}{lcccccccc}
\hline
{\textbf{}} & \textbf{} & \multicolumn{2}{c}{FID } & \multicolumn{1}{l}{\textbf{}}
& \multicolumn{4}{c}{IS}                                          \\
             \textbf{MNIST}               &    \multicolumn{1}{l}{\textbf{$\varepsilon'=0.1$} }       & \multicolumn{1}{l}{\textbf{$\varepsilon'=0.3$}}                      & \multicolumn{1}{l}{\textbf{$\varepsilon'=1$}}                      & \multicolumn{1}{l}{\textbf{$\varepsilon'=5$}} & \multicolumn{1}{l}{\textbf{$\varepsilon'=0.1$}}& \multicolumn{1}{l}{\textbf{$\varepsilon'=0.3$}}                      & \multicolumn{1}{l}{\textbf{$\varepsilon'=1$}}                      & \multicolumn{1}{l}{\textbf{$\varepsilon'=5$}}\\
\hline
ours($\boldsymbol\varepsilon$-centered)           &   \textbf{2.99}        &    \textbf{2.88}     &      \textbf{2.85}                                        &   untrained                            &                  2.28             &     \textbf{2.32} 
 &                  2.29               &     untrained \\
$0$-centered           &3.54          & 3.54   &    3.54                                        &    3.54                             & \textbf{2.31}                                &                2.31
 & \textbf{2.31}                                 &                 \textbf{2.31} \\

$1$-centered          & 3.64          & 3.64   &    3.64                                        &    3.64                              & 2.3                               &               2.3 & 2.3                                &                2.3 \\
\hline
  \textbf{LSUN Bedroom}               &   \\
\hline
ours($\boldsymbol\varepsilon$-centered)           &            \textbf{9.94} &   \textbf{8.78}     &     \textbf{ 9.73}                                        &   untrained                              &                 2.97                &     2.94
 &                  2.97                  &     untrained \\
$0$-centered       &10.54            &10.54  &    10.54                                        &    10.54                             & 3.067                                &                3.067
 & 3.067                                  &                3.067\\

$1$-centered         & 11.5           & 11.5   &    11.5                                        &    11.5                              & \textbf{3.154  }                                &              \textbf{3.154  }& \textbf{3.154  }                                 &                \textbf{3.154  }\\                             
\hline
{\textbf{}} & \textbf{} & \multicolumn{2}{c}{FID } & \multicolumn{1}{l}{\textbf{}}
& \multicolumn{4}{c}{IS}                                           \\
             \textbf{LSUN T+B}               &    \multicolumn{1}{l}{\textbf{$\gamma=0.1$}}                      & \multicolumn{1}{l}{\textbf{$\gamma=1$}}                      & \multicolumn{1}{l}{\textbf{$\gamma=10$}} &
             \multicolumn{1}{l}{\textbf{}}                      &\multicolumn{1}{l}{\textbf{$\gamma=0.1$}}                      & \multicolumn{1}{l}{\textbf{$\gamma=1$}}                      & \multicolumn{1}{l}{\textbf{$\gamma=10$}}
             &\multicolumn{1}{l}{\textbf{}}                      \\
\hline
\textbf{ours($\boldsymbol\varepsilon$-centered)}                      &    \textbf{15.72}      &      \textbf{16.85}                                         &  \textbf{16.67}                              &   &                             \textbf{5.07}       &     5.06
 &  \textbf{5.08} &                  \textbf{} \\
\textbf{$0$-centered}          & 16.01   &    17.4                                        &   19.79  &   \textbf{}                           & 5.05                                &                5.08
 &5.05                        &   \textbf{}    \\

\textbf{$1$-centered}           & 16.32    &    18.73                                      &  34.28  &   \textbf{}                            &5.04                                &                 \textbf{ 5.18} & 4.71    &   \textbf{}                \\
\hline      
\end{tabular}}
\vspace{-0.1in}
\end{table*} 
As we maintained the original network configurations for both models without applying any additional training optimizations, our results provide a fair comparison.

By maintaining the same network architecture as the CFG method, we were able to achieve the best scores in six out of seven datasets. The FID results demonstrate that our Li-CFG approach is effective and capable of generating high-quality images.

\noindent\textbf{Precision and Recall Results.}\label{pr}
{
We generate synthetic  10,000  samples compared with 10,000 real samples to compute precision and recall, utilizing the codes of Precision and Recall functions.\footnote{https://github.com/blandocs/improved-precision-and-recall-metric-pytorch}}. 
{
Except for the five GAN variants mentioned above, we additionally utilized BigGAN in both CIFAR10 and ImageNet and DDGAN on the CIFAR10 dataset.  We present these results in Table.~\ref{table:pr}.}

\noindent \textbf{Ablation Study about $\varepsilon'$ And $\delta(\boldsymbol x)$}.
\label{ablation study}
Our method introduces a controllable parameter that adjusts the gradient penalty within a specified range. As this parameter varies, the latent N-size adjusts accordingly. 
Through experiments, we demonstrate that a reasonable latent N-size is crucial, as shown in Table.~\ref{table:ablation}. A too small latent N-size, resulting from a large value of our parameter, causes the loss function to fail to converge. This reflects a trade-off between the diversity of synthetic samples and the model's training stability.
For the values of $\delta(\boldsymbol x)$ where the CFG method performs well, $\gamma=0.1$ consistently yields better results.
On the other hand, for cases where the CFG method performs poorly, $\gamma=10$ tends to improve performance.

\noindent \textbf{Generalization of $\boldsymbol\varepsilon$-centered Gradient Penalty.}
\label{generalization study}
{
The conception of our $\boldsymbol\varepsilon$-centered gradient penalty emerged from the CFG mechanism. The Column 'Li-CFG' in Table.~\ref{table:isfid} and Table.~\ref{table:pr} present the effectiveness of the CFG method. However, a fundamental question arises: How well does the generalization of our $\boldsymbol\varepsilon$-centered gradient penalty extend beyond the CFG framework?
To address this, we investigate the applicability of our mechanism to various GAN models. While the CFG method employs a distinct formulation, it shares equivalence with common GAN in dynamic theory. This prompts us to assess the performance of our $\boldsymbol\varepsilon$-centered gradient penalty across various models. %
The results in Table.~\ref{table:isfid} and Table.~\ref{table:pr}  reveal that our $\boldsymbol\varepsilon$-centered gradient penalty consistently achieves the best FID score across all these GAN models.}

{
To demonstrate the effectiveness and generalization of our $\boldsymbol\varepsilon$-centered gradient penalty, we compare various methods, focusing on the diversity of synthesized samples. The results of the compared methods are tabulated in Table.~\ref{table:comparemodels}.}

{
We have discovered an interesting phenomenon where spectral normalization, weight gradient, and gradient penalty can effectively work together to improve the diversity of synthesized samples. However, it can be argued that all these methods rely on Lipschitz constraints and, therefore may compete with each other. For instance, BigGAN and denoising diffusion GAN employ spectral normalization as a strong Lipschitz constraint for the varying weight of the neural networks. Given such a strong Lipschitz constraint, our gradient penalty should not affect the synthesis samples. However, when both models are applied with the gradient penalty, there is still a significant improvement in the diversity of the synthesis samples. 
This situation emphasizes that our theory is valid. It suggests that interpreting the gradient penalty only as a Lipschitz constraint may not be sufficient. Our neighborhood theory provides a new perspective to understand how the gradient penalty can improve the diversity of the model.}

{
This evidence highlights the strong generalization capability of our proposed method. It showcases its broad applicability by seamlessly integrating with standard GAN models, offering advantages that extend beyond the CFG mechanism.}

\begin{figure*} [htb]
	\centering
       \subfloat[\label{fig:ddgan-cifar10}]{
		\includegraphics[scale=0.615]{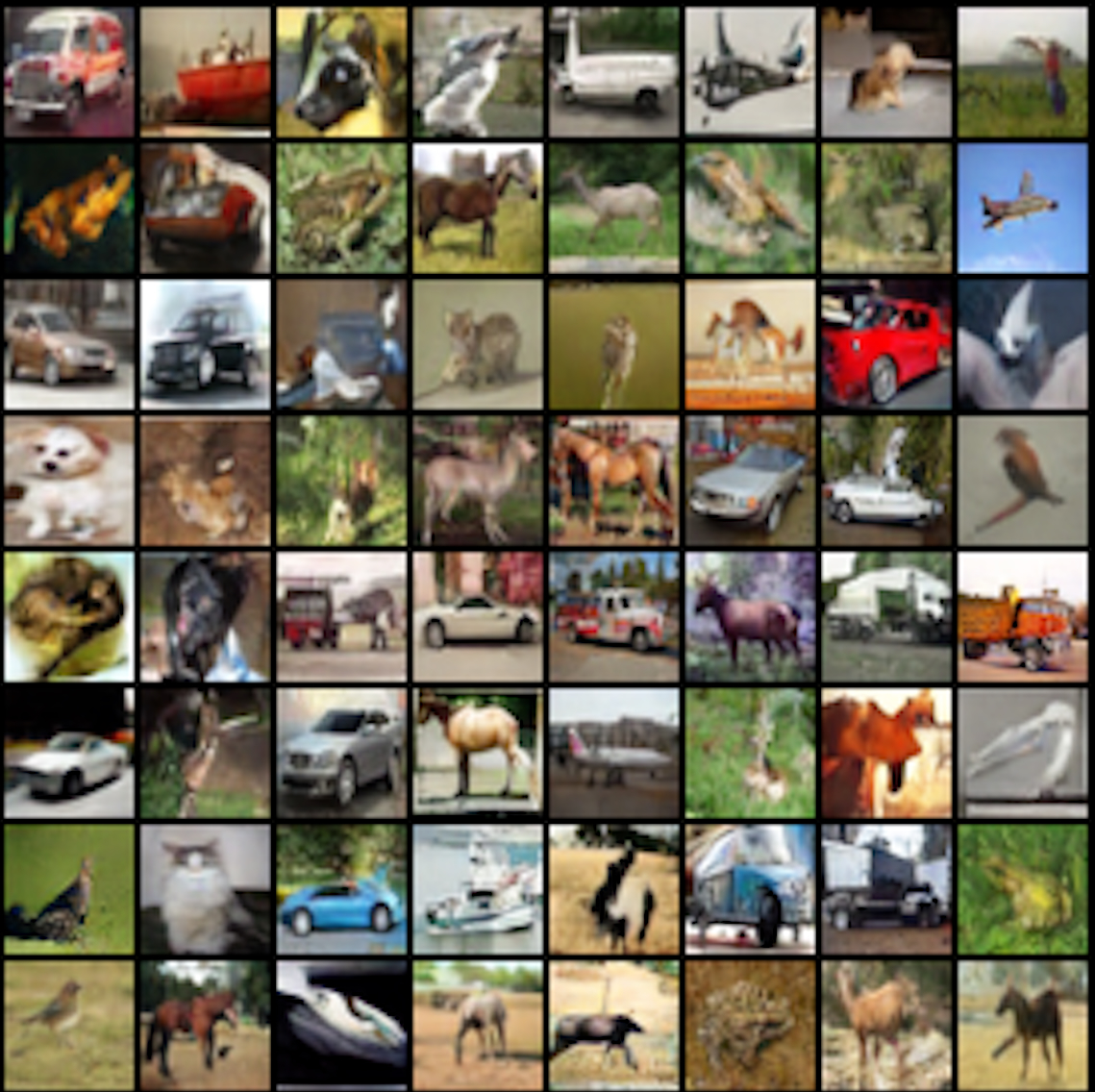}}
	\subfloat[\label{fig:biggan-cifar10}]{
		\includegraphics[scale=0.7]{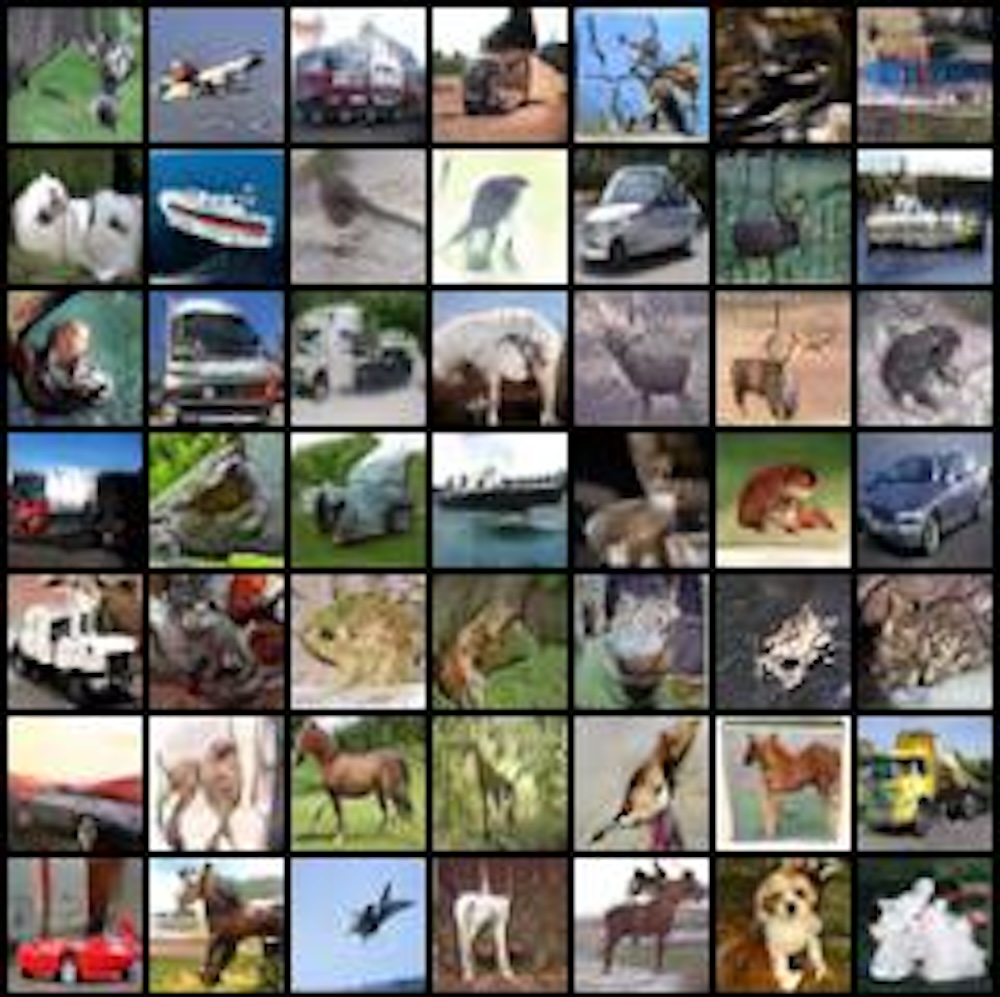}}
	\caption{{Column (a): Result for CIFAR10 from DDGAN with our $\boldsymbol\varepsilon$-centered gradient penalty.  Column (b): Result for CIFAR10 from BigGAN with our $\boldsymbol\varepsilon$-centered gradient penalty.}\label{fig:sota-cifar10}}
 \vspace{-0.2in}
\end{figure*}

\begin{table*}[htb!]\ra{1.3}
\vspace{-0.6in}
\caption{\centering 
{To assess the effeteness of our $\boldsymbol\varepsilon$-centered gradient penalty to the diversity target, we compare it to various GANs baselines focused on the diversity of synthesis samples. The improved FID refers to the difference in FID results between a method and a baseline. The symbol '-' indicates that we do not recalculate the experimental results in local computing environments. The 'Baseline' refers to the method that other methods are compared to. In order to make a fair comparison, all the methods being compared are based on the same baseline.
Our method is the one that uses the bold ($\boldsymbol\varepsilon$-centered).}
\label{table:comparemodels} 
}
\centering
\begin{tabular}{@{}p{0.38\textwidth}<{\raggedright}
p{0.12\textwidth}<{\centering}p{0.15\textwidth}<{\centering}@{}p{0.12\textwidth}<{\centering}@{}p{0.12\textwidth}<{\centering}@{}}
\toprule 
             \textbf{CIFAR10}               &    FID$\downarrow$      &  Improved $\uparrow$  FID                     & Precision                      & Recall$\uparrow$ \\
\midrule
Baseline: WGAN         &   36.24        &   -   &     0.78                                       &   0.41  \\
% \textbf{SVHN}              & \textbf{0.913}            & \textbf{0.87}&\textbf{1.39}  &    &   \textbf{1.41} &                                   \\
AdvLatGAN-qua~\cite{li2022improving}          & 30.21        &  6.03   &    0.69                                       &    0.45                                         \\
AdvLatGAN-qua+~\cite{li2022improving}           &  29.73        & 6.51 &  0.69   &     0.46                                         \\
WGAN-Unroll~\cite{metz2016unrolled}       & 30.28        &  5.98   &    0.7                                       &   0.45                                           \\
IID-GAN~\cite{li2021iid}         & 28.63         &  7.61   &     -                                        &    -         \\
\textbf{WGAN($\boldsymbol\varepsilon$-centered)}&   28.45       &     7.79    &     0.67                                        &   0.46                                                \\
\midrule
{Baseline: Origin GAN}           &  37.83        & - &  0.78  &     0.45                                         \\
{MSGAN~\cite{mao2019mode}}           &  32.38       & 5.45 &  0.77   &     0.51                                        \\
{AdvLatGAN-div+~\cite{li2022improving}}           &  30.92       & 6.91 &  0.78   &     0.54                                        \\
{\textbf{Origin GAN($\boldsymbol\varepsilon$-centered)} }           &  30.52        & 7.31 &  0.77   &     0.56                                         \\
\midrule           
Baseline: SNGAN-RES~\cite{miyato2018spectral}        &  15.93         &  - &     0.8                                   &     0.75                                       \\  
AdvLatGAN-qua~\cite{li2022improving}            & 20.75        &  -4.62   &     0.82                                       &    0.68                                                   \\
AdvLatGAN-qua+~\cite{li2022improving}            & 15.87        &  0.06   &     0.79                                      &    0.76        \\
GN-GAN~\cite{wu2021gradient}     &    15.31        &  0.63 &    0.77                                     &     0.75                                         \\
GraNC-GAN~\cite{bhaskara2022gran}         &  14.82          &  1.11  &    -                                   &     -    \\ 
aw-SN-GAN~\cite{zadorozhnyy2021adaptive}      &  8.9         &  7.03  &     -                                     &     -   \\ 
\textbf{SNGAN-RES($\boldsymbol\varepsilon$-centered)}           &   13.4       &   2.53     &     0.74                                       &  0.79                                     \\
\midrule           
Baseline: SNGAN-CNN~\cite{miyato2018spectral}       &  18.95         &  -  &     0.785                                    &     0.63    \\  
GN-GAN~\cite{wu2021gradient}        &  19.31          &  -0.36  &    0.81                                    &     0.59   \\ 
\textbf{SNGAN-CNN($\boldsymbol\varepsilon$-centered)}           &   
  17.92        &   1.03     &      0.79                                        &   0.64       \\

\midrule           
Baseline: BigGAN        & 8.25       &  - &     0.76                                    &     0.62   \\  
GN-BigGAN~\cite{wu2021gradient}        &  7.89         & 0.36  &    0.77                                     &     0.62                                          \\  
aw-BigGAN~\cite{zadorozhnyy2021adaptive}        &       7.03   & 1.22  &    -                                     &     -   \\
\textbf{BigGAN($\boldsymbol\varepsilon$-centered)}           &   5.18       &    3.07     &       0.75                                       &   0.67                                              \\ 
\textbf{DDGAN($\boldsymbol\varepsilon$-centered)}       &  2.38       &  5.87  & 0.74 &    0.69                                                  \\

\midrule 
 \textbf{ImageNet}               &
 \\
\midrule
Baseline: BigGAN         &   17.33       &  -     &      0.53                                      &  0.71    \\
\textbf{BigGAN($\boldsymbol\varepsilon$-centered)}           &   12.68      &   4.65    &      0.45                                       &   0.76                         \\                            
                        
\bottomrule 
\end{tabular}
\vspace{-0.1in}
\end{table*}

\begin{figure*}
\centering 
\includegraphics[width=0.9\textwidth]{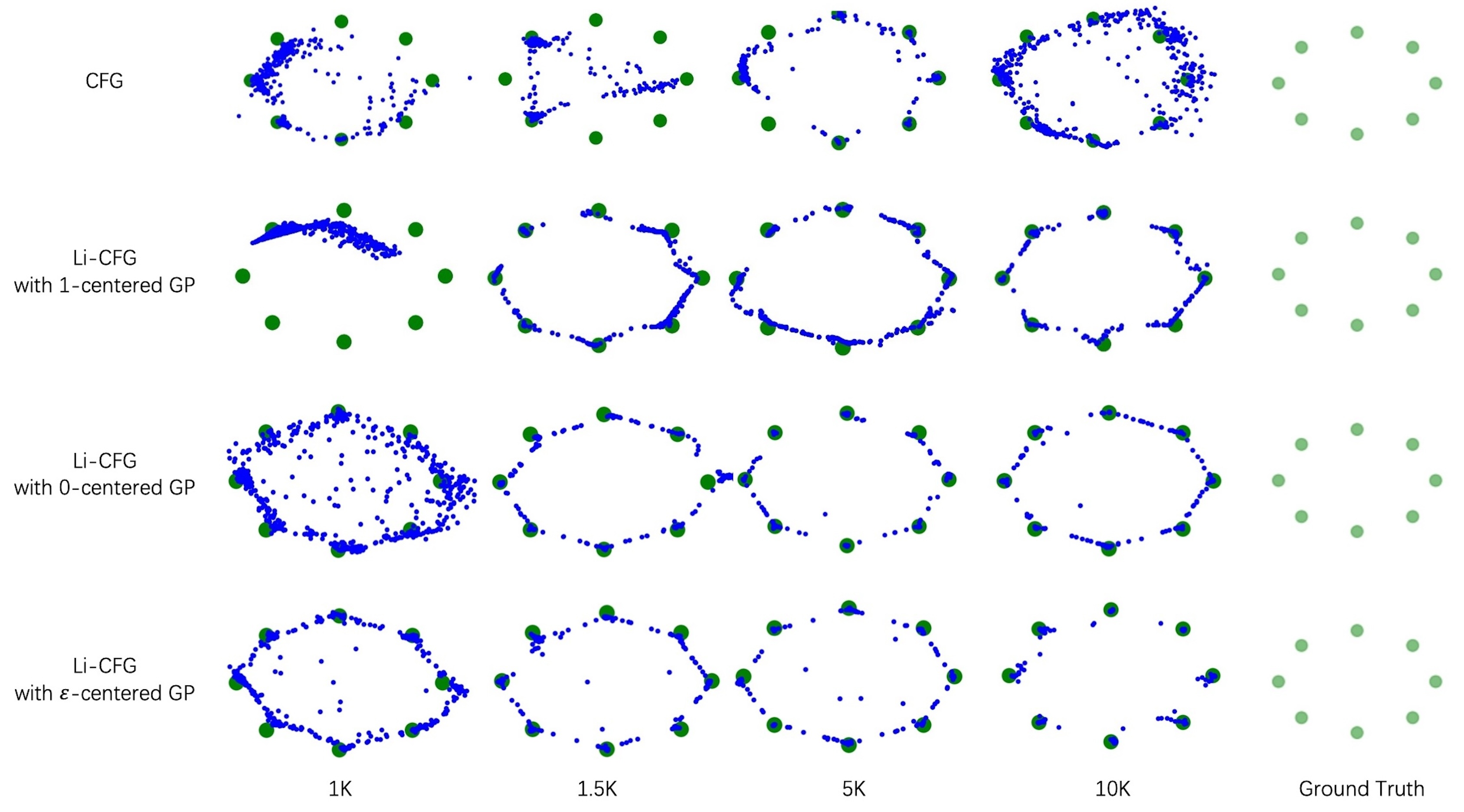} 
\caption{{Results for the CFG with varying gradient penalties on the Ring dataset are displayed as follows: From top to bottom, the sequence includes the  CFG, the Li-CFG with $1$-centered gradient penalty, the Li-CFG with $0$-centered gradient penalty, and Li-CFG with $\boldsymbol\varepsilon$-centered gradient penalty. Progressing from left to right, each column represents outcomes from different stages of training. The far-right column displays the ground truth data for comparison.}\label{Fig.licfg-ring}}
\vspace{-0.1in}
\end{figure*}

\begin{figure*}
\centering 
\includegraphics[width=0.9\textwidth]{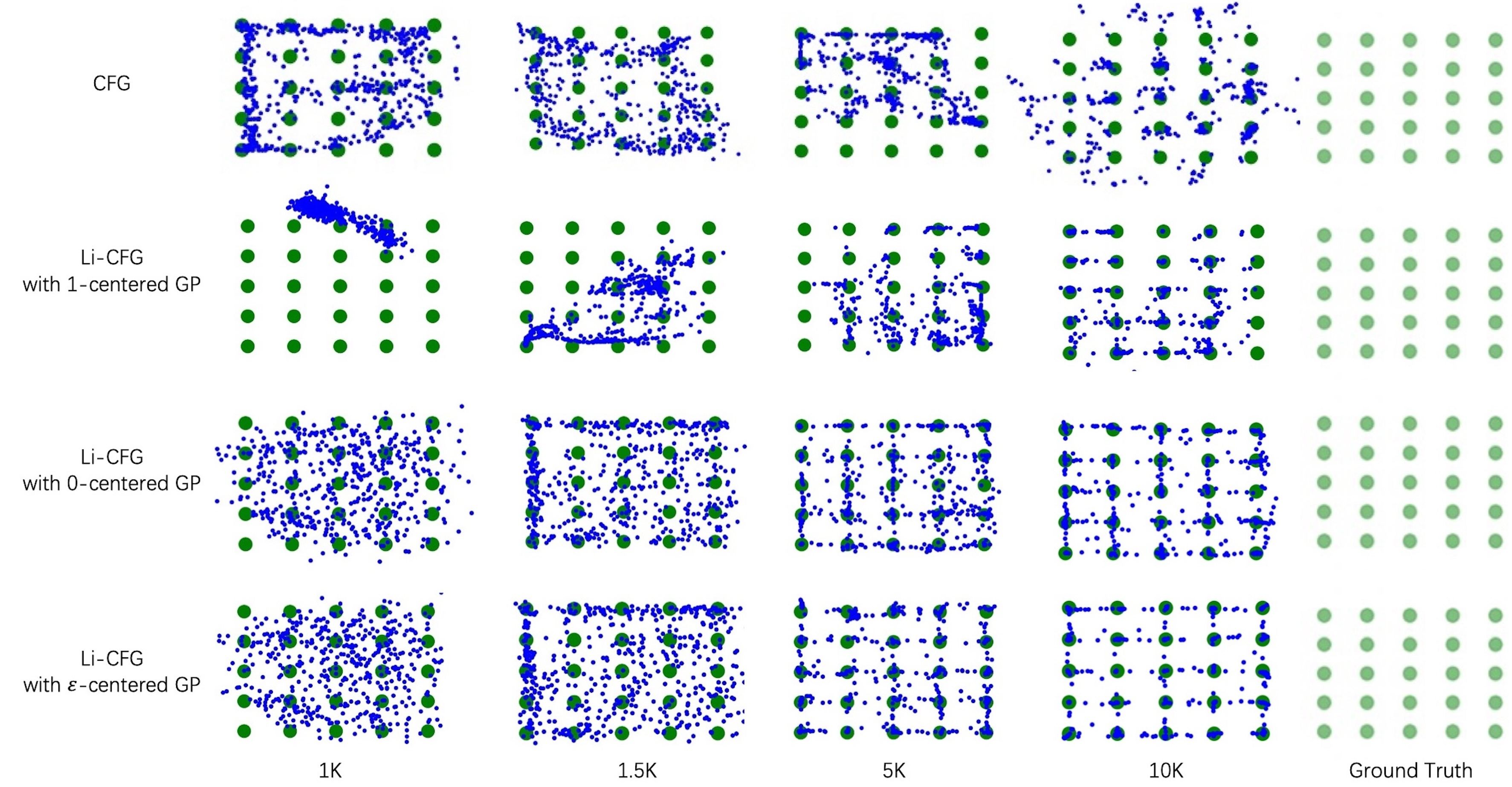} 
\caption{{Results for the CFG with varying gradient penalties on the Grid dataset are displayed as follows: From top to bottom, the sequence includes the  CFG, the Li-CFG with $1$-centered gradient penalty, the Li-CFG with $0$-centered gradient penalty, and Li-CFG with $\boldsymbol\varepsilon$-centered gradient penalty. Progressing from left to right, each column represents outcomes from different stages of training. The far-right column displays the ground truth data for comparison.} \label{Fig.licfg-square}}
\vspace{-0.1in}
\end{figure*}

\noindent \textbf{Visualization of synthesized images.}
The results of the experiment using Li-CFG and CFG methods can be seen in Fig.~\ref{Fig.main6} and Fig.~\ref{Fig.main7}. 
Our Li-CFG method can achieve the same or better results than the CFG method. In some datasets, the image generated by the CFG method has already collapsed, while the image generated by Li-CFG still performs well. 
{Furthermore, we also present synthesis samples generated with the state-of-the-art GAN model in CIFAR10, as shown in Fig.~\ref{fig:sota-cifar10}. 
Additional synthesis samples from various datasets can be found in Section~\ref{appedix.figures} of the supplementary materials.}

{The following results were obtained for the synthetic datasets presented in Fig.~\ref{Fig.licfg-ring},~\ref{Fig.licfg-square}. Our observation is that the unconstrained CFG method has difficulty in converging to all modes of the ring or grid datasets. However, when supplemented with gradient penalty, these methods show an enhanced ability to converge to a mixture of Gaussians. Out of the three types of gradient penalties tested, the $1$-centered gradient penalty showed inferior convergence compared to the $0$-centered gradient penalty and our proposed $\boldsymbol\varepsilon$-centered gradient penalty. Notably, our $\boldsymbol\varepsilon$-centered gradient penalty demonstrated a higher efficacy in driving more sample points to converge to the Gaussian points compared to the $0$-centered gradient penalty.}

\section{Conclusion}\label{conclusion}
In this paper, We provide a novel perspective to analyze the relationship between constraint and the diversity of synthetic samples. 
We assert that the constraint can effectively influence the latent N-size that is strongly associated with the mode collapse phenomenon. To modify the latent N-size efficiently, we propose a new form of Gradient Penalty called $\boldsymbol\varepsilon$-centered GP.
The experiments demonstrate that our method is more stable and achieves more diverse synthetic samples compared with the CFG method. Additionally, our method can be applied not only in the CFG method but also in common GAN models. Inception score, Fr\'{e}chet Distance, Precision/Recall, and  visual quality of generated images
show that our method is more powerful. 
In future work, we plan
to investigate the metric function of the generator from KL diversity to the Wasserstein distance to achieve a more stable and efficient GAN architecture.

\noindent \textbf{Limitations and future work.}
In this study, we achieve good FID results in the considered data sets. However, if given higher resolution data sets, it might lead to a different choice of our neural network architecture and hyper-parameter. On the other hand, although the hyper-parameter $\delta(\boldsymbol x)$ has an important impact on the algorithm, we do not discuss the relationship between hyper-parameter $\delta(\boldsymbol x)$ in CFG and our constraint hyper-parameter $\boldsymbol\varepsilon$ from a theoretical perspective.
These effects should be systematically studied in future work.

\newpage
\section{Declarations}
declaration - statement
\begin{itemize}
\item Conflicts of interest/Competing interests

Not applicable

\item Funding

This work was funded by the National Natural Science Foundation of China U22A20102, 62272419. Natural Science Foundation of Zhejiang Province ZJNSFLZ22F020010.

\item Ethics approval

Not applicable

\item Consent to participate

Not applicable

\item Consent for publication

Not applicable

\item Availability of data and material

All the data used in our work are sourced from publicly available datasets that have been established by prior research. References for all the data are provided in our main document.

\item Code availability

Our code is some kind of custom code and is available for use.

\item Authors' contributions

We list the Authors' contributions as follows:

Conceptualization: [Chang Wan], [Yanwei Fu], [Xinwei Sun]; 

Methodology: [Chang Wan], [Xinwei Sun], [Ke Fan]; 

Experiments: [Chang Wan], [Ke Fan], [Yunliang Jiang]; 

Writing  ‐  original draft preparation: [Chang Wan], [Yanwei Fu]; 

Writing  ‐  review and editing: [Xinwei Sun], [Zhonglong Zheng], [Minglu Li]; 

Funding acquisition: [Yunliang Jiang], [Zhonglong Zheng];

Resources: [Zhonglong Zheng]; 

Supervision: [Minglu Li], [Yunliang Jiang];

All authors read and approved the final manuscript.

\end{itemize}

\bibliography{sn-bibliography}


\begin{thebibliography}{42}
\ifx \bisbn   \undefined \def \bisbn  #1{ISBN #1}\fi
\ifx \binits  \undefined \def \binits#1{#1}\fi
\ifx \bauthor  \undefined \def \bauthor#1{#1}\fi
\ifx \batitle  \undefined \def \batitle#1{#1}\fi
\ifx \bjtitle  \undefined \def \bjtitle#1{#1}\fi
\ifx \bvolume  \undefined \def \bvolume#1{\textbf{#1}}\fi
\ifx \byear  \undefined \def \byear#1{#1}\fi
\ifx \bissue  \undefined \def \bissue#1{#1}\fi
\ifx \bfpage  \undefined \def \bfpage#1{#1}\fi
\ifx \blpage  \undefined \def \blpage #1{#1}\fi
\ifx \burl  \undefined \def \burl#1{\textsf{#1}}\fi
\ifx \doiurl  \undefined \def \doiurl#1{\url{https://doi.org/#1}}\fi
\ifx \betal  \undefined \def \betal{\textit{et al.}}\fi
\ifx \binstitute  \undefined \def \binstitute#1{#1}\fi
\ifx \binstitutionaled  \undefined \def \binstitutionaled#1{#1}\fi
\ifx \bctitle  \undefined \def \bctitle#1{#1}\fi
\ifx \beditor  \undefined \def \beditor#1{#1}\fi
\ifx \bpublisher  \undefined \def \bpublisher#1{#1}\fi
\ifx \bbtitle  \undefined \def \bbtitle#1{#1}\fi
\ifx \bedition  \undefined \def \bedition#1{#1}\fi
\ifx \bseriesno  \undefined \def \bseriesno#1{#1}\fi
\ifx \blocation  \undefined \def \blocation#1{#1}\fi
\ifx \bsertitle  \undefined \def \bsertitle#1{#1}\fi
\ifx \bsnm \undefined \def \bsnm#1{#1}\fi
\ifx \bsuffix \undefined \def \bsuffix#1{#1}\fi
\ifx \bparticle \undefined \def \bparticle#1{#1}\fi
\ifx \barticle \undefined \def \barticle#1{#1}\fi
\bibcommenthead
\ifx \bconfdate \undefined \def \bconfdate #1{#1}\fi
\ifx \botherref \undefined \def \botherref #1{#1}\fi
\ifx \url \undefined \def \url#1{\textsf{#1}}\fi
\ifx \bchapter \undefined \def \bchapter#1{#1}\fi
\ifx \bbook \undefined \def \bbook#1{#1}\fi
\ifx \bcomment \undefined \def \bcomment#1{#1}\fi
\ifx \oauthor \undefined \def \oauthor#1{#1}\fi
\ifx \citeauthoryear \undefined \def \citeauthoryear#1{#1}\fi
\ifx \endbibitem  \undefined \def \endbibitem {}\fi
\ifx \bconflocation  \undefined \def \bconflocation#1{#1}\fi
\ifx \arxivurl  \undefined \def \arxivurl#1{\textsf{#1}}\fi
\csname PreBibitemsHook\endcsname

%%% 1
\bibitem[\protect\citeauthoryear{Goodfellow et~al.}{2014}]{goodfellow2014generative}
\begin{botherref}
\oauthor{\bsnm{Goodfellow}, \binits{I.}},
\oauthor{\bsnm{Pouget-Abadie}, \binits{J.}},
\oauthor{\bsnm{Mirza}, \binits{M.}},
\oauthor{\bsnm{Xu}, \binits{B.}},
\oauthor{\bsnm{Warde-Farley}, \binits{D.}},
\oauthor{\bsnm{Ozair}, \binits{S.}},
\oauthor{\bsnm{Courville}, \binits{A.}},
\oauthor{\bsnm{Bengio}, \binits{Y.}}:
Generative adversarial nets.
Conference on Neural Information Processing Systems
\textbf{27}
(2014)
\end{botherref}
\endbibitem

%%% 2
\bibitem[\protect\citeauthoryear{Arjovsky et~al.}{2017}]{arjovsky2017wasserstein}
\begin{bchapter}
\bauthor{\bsnm{Arjovsky}, \binits{M.}},
\bauthor{\bsnm{Chintala}, \binits{S.}},
\bauthor{\bsnm{Bottou}, \binits{L.}}:
\bctitle{Wasserstein generative adversarial networks}.
In: \bbtitle{International Conference on Machine Learning},
pp. \bfpage{214}--\blpage{223}
(\byear{2017}).
\bcomment{PMLR}
\end{bchapter}
\endbibitem

%%% 3
\bibitem[\protect\citeauthoryear{Gulrajani et~al.}{2017}]{gulrajani2017improved}
\begin{botherref}
\oauthor{\bsnm{Gulrajani}, \binits{I.}},
\oauthor{\bsnm{Ahmed}, \binits{F.}},
\oauthor{\bsnm{Arjovsky}, \binits{M.}},
\oauthor{\bsnm{Dumoulin}, \binits{V.}},
\oauthor{\bsnm{Courville}, \binits{A.}}:
Improved training of wasserstein gans.
arXiv preprint arXiv:1704.00028
(2017)
\end{botherref}
\endbibitem

%%% 4
\bibitem[\protect\citeauthoryear{Johnson and Zhang}{2018}]{johnson2018composite}
\begin{bchapter}
\bauthor{\bsnm{Johnson}, \binits{R.}},
\bauthor{\bsnm{Zhang}, \binits{T.}}:
\bctitle{Composite functional gradient learning of generative adversarial models}.
In: \bbtitle{International Conference on Machine Learning},
pp. \bfpage{2371}--\blpage{2379}
(\byear{2018}).
\bcomment{PMLR}
\end{bchapter}
\endbibitem

%%% 5
\bibitem[\protect\citeauthoryear{Mescheder et~al.}{2018}]{mescheder2018training}
\begin{bchapter}
\bauthor{\bsnm{Mescheder}, \binits{L.}},
\bauthor{\bsnm{Geiger}, \binits{A.}},
\bauthor{\bsnm{Nowozin}, \binits{S.}}:
\bctitle{Which training methods for gans do actually converge?}
In: \bbtitle{International Conference on Machine Learning},
pp. \bfpage{3481}--\blpage{3490}
(\byear{2018}).
\bcomment{PMLR}
\end{bchapter}
\endbibitem

%%% 6
\bibitem[\protect\citeauthoryear{Johnson and Zhang}{2019}]{johnson2019framework}
\begin{barticle}
\bauthor{\bsnm{Johnson}, \binits{R.}},
\bauthor{\bsnm{Zhang}, \binits{T.}}:
\batitle{A framework of composite functional gradient methods for generative adversarial models}.
\bjtitle{IEEE Transactions on Pattern Analysis and Machine Intelligence}
\bvolume{43}(\bissue{1}),
\bfpage{17}--\blpage{32}
(\byear{2019})
\end{barticle}
\endbibitem

%%% 7
\bibitem[\protect\citeauthoryear{Roth et~al.}{2017}]{roth2017stabilizing}
\begin{botherref}
\oauthor{\bsnm{Roth}, \binits{K.}},
\oauthor{\bsnm{Lucchi}, \binits{A.}},
\oauthor{\bsnm{Nowozin}, \binits{S.}},
\oauthor{\bsnm{Hofmann}, \binits{T.}}:
Stabilizing training of generative adversarial networks through regularization.
arXiv preprint arXiv:1705.09367
(2017)
\end{botherref}
\endbibitem

%%% 8
\bibitem[\protect\citeauthoryear{Yang et~al.}{2019}]{yang2019diversity}
\begin{botherref}
\oauthor{\bsnm{Yang}, \binits{D.}},
\oauthor{\bsnm{Hong}, \binits{S.}},
\oauthor{\bsnm{Jang}, \binits{Y.}},
\oauthor{\bsnm{Zhao}, \binits{T.}},
\oauthor{\bsnm{Lee}, \binits{H.}}:
Diversity-sensitive conditional generative adversarial networks.
arXiv preprint arXiv:1901.09024
(2019)
\end{botherref}
\endbibitem

%%% 9
\bibitem[\protect\citeauthoryear{Radford et~al.}{2015}]{radford2015unsupervised}
\begin{botherref}
\oauthor{\bsnm{Radford}, \binits{A.}},
\oauthor{\bsnm{Metz}, \binits{L.}},
\oauthor{\bsnm{Chintala}, \binits{S.}}:
Unsupervised representation learning with deep convolutional generative adversarial networks.
arXiv preprint arXiv:1511.06434
(2015)
\end{botherref}
\endbibitem

%%% 10
\bibitem[\protect\citeauthoryear{Zhang et~al.}{2019}]{zhang2019self}
\begin{bchapter}
\bauthor{\bsnm{Zhang}, \binits{H.}},
\bauthor{\bsnm{Goodfellow}, \binits{I.}},
\bauthor{\bsnm{Metaxas}, \binits{D.}},
\bauthor{\bsnm{Odena}, \binits{A.}}:
\bctitle{Self-attention generative adversarial networks}.
In: \bbtitle{International Conference on Machine Learning},
pp. \bfpage{7354}--\blpage{7363}
(\byear{2019}).
\bcomment{PMLR}
\end{bchapter}
\endbibitem

%%% 11
\bibitem[\protect\citeauthoryear{Karras et~al.}{2017}]{karras2017progressive}
\begin{botherref}
\oauthor{\bsnm{Karras}, \binits{T.}},
\oauthor{\bsnm{Aila}, \binits{T.}},
\oauthor{\bsnm{Laine}, \binits{S.}},
\oauthor{\bsnm{Lehtinen}, \binits{J.}}:
Progressive growing of gans for improved quality, stability, and variation.
arXiv preprint arXiv:1710.10196
(2017)
\end{botherref}
\endbibitem

%%% 12
\bibitem[\protect\citeauthoryear{Brock et~al.}{2018}]{brock2018large}
\begin{botherref}
\oauthor{\bsnm{Brock}, \binits{A.}},
\oauthor{\bsnm{Donahue}, \binits{J.}},
\oauthor{\bsnm{Simonyan}, \binits{K.}}:
Large scale gan training for high fidelity natural image synthesis.
arXiv preprint arXiv:1809.11096
(2018)
\end{botherref}
\endbibitem

%%% 13
\bibitem[\protect\citeauthoryear{Karras et~al.}{2019}]{karras2019style}
\begin{bchapter}
\bauthor{\bsnm{Karras}, \binits{T.}},
\bauthor{\bsnm{Laine}, \binits{S.}},
\bauthor{\bsnm{Aila}, \binits{T.}}:
\bctitle{A style-based generator architecture for generative adversarial networks}.
In: \bbtitle{IEEE/CVF Computer Vision and Pattern Recognition Conference},
pp. \bfpage{4401}--\blpage{4410}
(\byear{2019})
\end{bchapter}
\endbibitem

%%% 14
\bibitem[\protect\citeauthoryear{Che et~al.}{2016}]{che2016mode}
\begin{botherref}
\oauthor{\bsnm{Che}, \binits{T.}},
\oauthor{\bsnm{Li}, \binits{Y.}},
\oauthor{\bsnm{Jacob}, \binits{A.P.}},
\oauthor{\bsnm{Bengio}, \binits{Y.}},
\oauthor{\bsnm{Li}, \binits{W.}}:
Mode regularized generative adversarial networks.
arXiv preprint arXiv:1612.02136
(2016)
\end{botherref}
\endbibitem

%%% 15
\bibitem[\protect\citeauthoryear{Nowozin et~al.}{2016}]{nowozin2016f}
\begin{bchapter}
\bauthor{\bsnm{Nowozin}, \binits{S.}},
\bauthor{\bsnm{Cseke}, \binits{B.}},
\bauthor{\bsnm{Tomioka}, \binits{R.}}:
\bctitle{f-gan: Training generative neural samplers using variational divergence minimization}.
In: \bbtitle{Conference on Neural Information Processing Systems},
pp. \bfpage{271}--\blpage{279}
(\byear{2016})
\end{bchapter}
\endbibitem

%%% 16
\bibitem[\protect\citeauthoryear{Nagarajan and Kolter}{2017}]{nagarajan2017gradient}
\begin{botherref}
\oauthor{\bsnm{Nagarajan}, \binits{V.}},
\oauthor{\bsnm{Kolter}, \binits{J.Z.}}:
Gradient descent gan optimization is locally stable.
arXiv preprint arXiv:1706.04156
(2017)
\end{botherref}
\endbibitem

%%% 17
\bibitem[\protect\citeauthoryear{Mescheder et~al.}{2017}]{mescheder2017numerics}
\begin{botherref}
\oauthor{\bsnm{Mescheder}, \binits{L.}},
\oauthor{\bsnm{Nowozin}, \binits{S.}},
\oauthor{\bsnm{Geiger}, \binits{A.}}:
The numerics of gans.
arXiv preprint arXiv:1705.10461
(2017)
\end{botherref}
\endbibitem

%%% 18
\bibitem[\protect\citeauthoryear{Oberman and Calder}{2018}]{oberman2018lipschitz}
\begin{botherref}
\oauthor{\bsnm{Oberman}, \binits{A.M.}},
\oauthor{\bsnm{Calder}, \binits{J.}}:
Lipschitz regularized deep neural networks converge and generalize.
arXiv preprint arXiv:1808.09540
(2018)
\end{botherref}
\endbibitem

%%% 19
\bibitem[\protect\citeauthoryear{Scaman and Virmaux}{2018}]{scaman2018lipschitz}
\begin{botherref}
\oauthor{\bsnm{Scaman}, \binits{K.}},
\oauthor{\bsnm{Virmaux}, \binits{A.}}:
Lipschitz regularity of deep neural networks: analysis and efficient estimation.
arXiv preprint arXiv:1805.10965
(2018)
\end{botherref}
\endbibitem

%%% 20
\bibitem[\protect\citeauthoryear{Zhou et~al.}{2018}]{zhou2018understanding}
\begin{botherref}
\oauthor{\bsnm{Zhou}, \binits{Z.}},
\oauthor{\bsnm{Song}, \binits{Y.}},
\oauthor{\bsnm{Yu}, \binits{L.}},
\oauthor{\bsnm{Wang}, \binits{H.}},
\oauthor{\bsnm{Liang}, \binits{J.}},
\oauthor{\bsnm{Zhang}, \binits{W.}},
\oauthor{\bsnm{Zhang}, \binits{Z.}},
\oauthor{\bsnm{Yu}, \binits{Y.}}:
Understanding the effectiveness of lipschitz-continuity in generative adversarial nets.
arXiv preprint arXiv:1807.00751
(2018)
\end{botherref}
\endbibitem

%%% 21
\bibitem[\protect\citeauthoryear{Zhou et~al.}{2019}]{zhou2019lipschitz}
\begin{bchapter}
\bauthor{\bsnm{Zhou}, \binits{Z.}},
\bauthor{\bsnm{Liang}, \binits{J.}},
\bauthor{\bsnm{Song}, \binits{Y.}},
\bauthor{\bsnm{Yu}, \binits{L.}},
\bauthor{\bsnm{Wang}, \binits{H.}},
\bauthor{\bsnm{Zhang}, \binits{W.}},
\bauthor{\bsnm{Yu}, \binits{Y.}},
\bauthor{\bsnm{Zhang}, \binits{Z.}}:
\bctitle{Lipschitz generative adversarial nets}.
In: \bbtitle{International Conference on Machine Learning},
pp. \bfpage{7584}--\blpage{7593}
(\byear{2019}).
\bcomment{PMLR}
\end{bchapter}
\endbibitem

%%% 22
\bibitem[\protect\citeauthoryear{Herrera et~al.}{2020}]{herrera2020estimating}
\begin{botherref}
\oauthor{\bsnm{Herrera}, \binits{C.}},
\oauthor{\bsnm{Krach}, \binits{F.}},
\oauthor{\bsnm{Teichmann}, \binits{J.}}:
Estimating full lipschitz constants of deep neural networks.
arXiv preprint arXiv:2004.13135
(2020)
\end{botherref}
\endbibitem

%%% 23
\bibitem[\protect\citeauthoryear{Kim et~al.}{2021}]{kim2021lipschitz}
\begin{bchapter}
\bauthor{\bsnm{Kim}, \binits{H.}},
\bauthor{\bsnm{Papamakarios}, \binits{G.}},
\bauthor{\bsnm{Mnih}, \binits{A.}}:
\bctitle{The lipschitz constant of self-attention}.
In: \bbtitle{International Conference on Machine Learning},
pp. \bfpage{5562}--\blpage{5571}
(\byear{2021}).
\bcomment{PMLR}
\end{bchapter}
\endbibitem

%%% 24
\bibitem[\protect\citeauthoryear{Miyato et~al.}{2018}]{miyato2018spectral}
\begin{botherref}
\oauthor{\bsnm{Miyato}, \binits{T.}},
\oauthor{\bsnm{Kataoka}, \binits{T.}},
\oauthor{\bsnm{Koyama}, \binits{M.}},
\oauthor{\bsnm{Yoshida}, \binits{Y.}}:
Spectral normalization for generative adversarial networks.
arXiv preprint arXiv:1802.05957
(2018)
\end{botherref}
\endbibitem

%%% 25
\bibitem[\protect\citeauthoryear{Bhaskara et~al.}{2022}]{bhaskara2022gran}
\begin{bchapter}
\bauthor{\bsnm{Bhaskara}, \binits{V.S.}},
\bauthor{\bsnm{Aumentado-Armstrong}, \binits{T.}},
\bauthor{\bsnm{Jepson}, \binits{A.D.}},
\bauthor{\bsnm{Levinshtein}, \binits{A.}}:
\bctitle{Gran-gan: Piecewise gradient normalization for generative adversarial networks}.
In: \bbtitle{Proceedings of the IEEE/CVF Winter Conference on Applications of Computer Vision},
pp. \bfpage{3821}--\blpage{3830}
(\byear{2022})
\end{bchapter}
\endbibitem

%%% 26
\bibitem[\protect\citeauthoryear{Wu et~al.}{2021}]{wu2021gradient}
\begin{bchapter}
\bauthor{\bsnm{Wu}, \binits{Y.-L.}},
\bauthor{\bsnm{Shuai}, \binits{H.-H.}},
\bauthor{\bsnm{Tam}, \binits{Z.-R.}},
\bauthor{\bsnm{Chiu}, \binits{H.-Y.}}:
\bctitle{Gradient normalization for generative adversarial networks}.
In: \bbtitle{InternationalConference on Computer Vision},
pp. \bfpage{6373}--\blpage{6382}
(\byear{2021})
\end{bchapter}
\endbibitem

%%% 27
\bibitem[\protect\citeauthoryear{Li et~al.}{2022}]{li2022improving}
\begin{barticle}
\bauthor{\bsnm{Li}, \binits{Y.}},
\bauthor{\bsnm{Mo}, \binits{Y.}},
\bauthor{\bsnm{Shi}, \binits{L.}},
\bauthor{\bsnm{Yan}, \binits{J.}}:
\batitle{Improving generative adversarial networks via adversarial learning in latent space}.
\bjtitle{Conference on Neural Information Processing Systems}
\bvolume{35},
\bfpage{8868}--\blpage{8881}
(\byear{2022})
\end{barticle}
\endbibitem

%%% 28
\bibitem[\protect\citeauthoryear{Mao et~al.}{2019}]{mao2019mode}
\begin{bchapter}
\bauthor{\bsnm{Mao}, \binits{Q.}},
\bauthor{\bsnm{Lee}, \binits{H.-Y.}},
\bauthor{\bsnm{Tseng}, \binits{H.-Y.}},
\bauthor{\bsnm{Ma}, \binits{S.}},
\bauthor{\bsnm{Yang}, \binits{M.-H.}}:
\bctitle{Mode seeking generative adversarial networks for diverse image synthesis}.
In: \bbtitle{IEEE/CVF Computer Vision and Pattern Recognition Conference},
pp. \bfpage{1429}--\blpage{1437}
(\byear{2019})
\end{bchapter}
\endbibitem

%%% 29
\bibitem[\protect\citeauthoryear{Metz et~al.}{2016}]{metz2016unrolled}
\begin{botherref}
\oauthor{\bsnm{Metz}, \binits{L.}},
\oauthor{\bsnm{Poole}, \binits{B.}},
\oauthor{\bsnm{Pfau}, \binits{D.}},
\oauthor{\bsnm{Sohl-Dickstein}, \binits{J.}}:
Unrolled generative adversarial networks.
arXiv preprint arXiv:1611.02163
(2016)
\end{botherref}
\endbibitem

%%% 30
\bibitem[\protect\citeauthoryear{Li et~al.}{2021}]{li2021iid}
\begin{botherref}
\oauthor{\bsnm{Li}, \binits{Y.}},
\oauthor{\bsnm{Shi}, \binits{L.}},
\oauthor{\bsnm{Yan}, \binits{J.}}:
Iid-gan: an iid sampling perspective for regularizing mode collapse.
arXiv preprint arXiv:2106.00563
(2021)
\end{botherref}
\endbibitem

%%% 31
\bibitem[\protect\citeauthoryear{LeCun et~al.}{1998}]{lecun1998gradient}
\begin{barticle}
\bauthor{\bsnm{LeCun}, \binits{Y.}},
\bauthor{\bsnm{Bottou}, \binits{L.}},
\bauthor{\bsnm{Bengio}, \binits{Y.}},
\bauthor{\bsnm{Haffner}, \binits{P.}}:
\batitle{Gradient-based learning applied to document recognition}.
\bjtitle{Proceedings of the IEEE}
\bvolume{86}(\bissue{11}),
\bfpage{2278}--\blpage{2324}
(\byear{1998})
\end{barticle}
\endbibitem

%%% 32
\bibitem[\protect\citeauthoryear{Krizhevsky et~al.}{2009}]{krizhevsky2009learning}
\begin{botherref}
\oauthor{\bsnm{Krizhevsky}, \binits{A.}},
\oauthor{\bsnm{Hinton}, \binits{G.}}, et al.:
Learning multiple layers of features from tiny images.
Master's thesis, Department of Computer Science, University of Toronto
(2009)
\end{botherref}
\endbibitem

%%% 33
\bibitem[\protect\citeauthoryear{Yu et~al.}{2015}]{yu2015lsun}
\begin{botherref}
\oauthor{\bsnm{Yu}, \binits{F.}},
\oauthor{\bsnm{Seff}, \binits{A.}},
\oauthor{\bsnm{Zhang}, \binits{Y.}},
\oauthor{\bsnm{Song}, \binits{S.}},
\oauthor{\bsnm{Funkhouser}, \binits{T.}},
\oauthor{\bsnm{Xiao}, \binits{J.}}:
Lsun: Construction of a large-scale image dataset using deep learning with humans in the loop.
arXiv preprint arXiv:1506.03365
(2015)
\end{botherref}
\endbibitem

%%% 34
\bibitem[\protect\citeauthoryear{Deng et~al.}{2009}]{deng2009imagenet}
\begin{bchapter}
\bauthor{\bsnm{Deng}, \binits{J.}},
\bauthor{\bsnm{Dong}, \binits{W.}},
\bauthor{\bsnm{Socher}, \binits{R.}},
\bauthor{\bsnm{Li}, \binits{L.-J.}},
\bauthor{\bsnm{Li}, \binits{K.}},
\bauthor{\bsnm{Fei-Fei}, \binits{L.}}:
\bctitle{Imagenet: A large-scale hierarchical image database}.
In: \bbtitle{IEEE/CVF Computer Vision and Pattern Recognition Conference},
pp. \bfpage{248}--\blpage{255}
(\byear{2009}).
\bcomment{Ieee}
\end{bchapter}
\endbibitem

%%% 35
\bibitem[\protect\citeauthoryear{Mao et~al.}{2017}]{mao2017least}
\begin{bchapter}
\bauthor{\bsnm{Mao}, \binits{X.}},
\bauthor{\bsnm{Li}, \binits{Q.}},
\bauthor{\bsnm{Xie}, \binits{H.}},
\bauthor{\bsnm{Lau}, \binits{R.Y.}},
\bauthor{\bsnm{Wang}, \binits{Z.}},
\bauthor{\bsnm{Paul~Smolley}, \binits{S.}}:
\bctitle{Least squares generative adversarial networks}.
In: \bbtitle{InternationalConference on Computer Vision},
pp. \bfpage{2794}--\blpage{2802}
(\byear{2017})
\end{bchapter}
\endbibitem

%%% 36
\bibitem[\protect\citeauthoryear{Lim and Ye}{2017}]{lim2017geometric}
\begin{botherref}
\oauthor{\bsnm{Lim}, \binits{J.H.}},
\oauthor{\bsnm{Ye}, \binits{J.C.}}:
Geometric gan.
arXiv preprint arXiv:1705.02894
(2017)
\end{botherref}
\endbibitem

%%% 37
\bibitem[\protect\citeauthoryear{He et~al.}{2016}]{he2016deep}
\begin{bchapter}
\bauthor{\bsnm{He}, \binits{K.}},
\bauthor{\bsnm{Zhang}, \binits{X.}},
\bauthor{\bsnm{Ren}, \binits{S.}},
\bauthor{\bsnm{Sun}, \binits{J.}}:
\bctitle{Deep residual learning for image recognition}.
In: \bbtitle{IEEE/CVF Computer Vision and Pattern Recognition Conference},
pp. \bfpage{770}--\blpage{778}
(\byear{2016})
\end{bchapter}
\endbibitem

%%% 38
\bibitem[\protect\citeauthoryear{Xiao et~al.}{2021}]{xiao2021tackling}
\begin{botherref}
\oauthor{\bsnm{Xiao}, \binits{Z.}},
\oauthor{\bsnm{Kreis}, \binits{K.}},
\oauthor{\bsnm{Vahdat}, \binits{A.}}:
Tackling the generative learning trilemma with denoising diffusion gans.
arXiv preprint arXiv:2112.07804
(2021)
\end{botherref}
\endbibitem

%%% 39
\bibitem[\protect\citeauthoryear{Salimans et~al.}{2016}]{salimans2016improved}
\begin{barticle}
\bauthor{\bsnm{Salimans}, \binits{T.}},
\bauthor{\bsnm{Goodfellow}, \binits{I.}},
\bauthor{\bsnm{Zaremba}, \binits{W.}},
\bauthor{\bsnm{Cheung}, \binits{V.}},
\bauthor{\bsnm{Radford}, \binits{A.}},
\bauthor{\bsnm{Chen}, \binits{X.}}:
\batitle{Improved techniques for training gans}.
\bjtitle{Conference on Neural Information Processing Systems}
\bvolume{29},
\bfpage{2234}--\blpage{2242}
(\byear{2016})
\end{barticle}
\endbibitem

%%% 40
\bibitem[\protect\citeauthoryear{Heusel et~al.}{2017}]{heusel2017gans}
\begin{botherref}
\oauthor{\bsnm{Heusel}, \binits{M.}},
\oauthor{\bsnm{Ramsauer}, \binits{H.}},
\oauthor{\bsnm{Unterthiner}, \binits{T.}},
\oauthor{\bsnm{Nessler}, \binits{B.}},
\oauthor{\bsnm{Hochreiter}, \binits{S.}}:
Gans trained by a two time-scale update rule converge to a local nash equilibrium.
Conference on Neural Information Processing Systems
\textbf{30}
(2017)
\end{botherref}
\endbibitem

%%% 41
\bibitem[\protect\citeauthoryear{Kynk{\"a}{\"a}nniemi et~al.}{2019}]{kynkaanniemi2019improved}
\begin{botherref}
\oauthor{\bsnm{Kynk{\"a}{\"a}nniemi}, \binits{T.}},
\oauthor{\bsnm{Karras}, \binits{T.}},
\oauthor{\bsnm{Laine}, \binits{S.}},
\oauthor{\bsnm{Lehtinen}, \binits{J.}},
\oauthor{\bsnm{Aila}, \binits{T.}}:
Improved precision and recall metric for assessing generative models.
Conference on Neural Information Processing Systems
\textbf{32}
(2019)
\end{botherref}
\endbibitem

%%% 42
\bibitem[\protect\citeauthoryear{Zadorozhnyy et~al.}{2021}]{zadorozhnyy2021adaptive}
\begin{bchapter}
\bauthor{\bsnm{Zadorozhnyy}, \binits{V.}},
\bauthor{\bsnm{Cheng}, \binits{Q.}},
\bauthor{\bsnm{Ye}, \binits{Q.}}:
\bctitle{Adaptive weighted discriminator for training generative adversarial networks}.
In: \bbtitle{IEEE/CVF Computer Vision and Pattern Recognition Conference},
pp. \bfpage{4781}--\blpage{4790}
(\byear{2021})
\end{bchapter}
\endbibitem

\end{thebibliography}

\newpage
\addcontentsline{toc}{section}{Appendix} % Add the appendix text to the document TOC
\part{Appendix} % Start the appendix part
\parttoc % Insert the appendix TOC
% \tableofcontents

\newpage
\begin{appendices}

\section{Overview}
% \noindent \textbf{Overview.}
In this section, we outline the contents of our supplementary material, which is divided into four main sections. Firstly, we introduce the supplementary material.
Secondly, we delve into the analysis of the dynamic theory for the CFG method. This section covers essential concepts related to CFG, the Lipschitz constraint, and the analysis of the dynamic theory of CFG. Thirdly, the theoretical analysis of our theory section contains the proof of our definition, Lemma and Theorem. Finally, we showcase the results of our experiments in the last section.

In the analysis of dynamic theory for the CFG method section, we provide foundational knowledge about CFG and Lipschitz constraints. Furthermore, we provide theoretical proofs that establish certain equivalences between the CFG method and conventional GAN theory, specifically focusing on dynamic theory principles.

{In the theoretical analysis of our theory section, we initiate our exploration by presenting the proofs for both Definition~\ref{definition-radius} and the Proposition~\ref{relation n d}. The Definition~\ref{definition-radius} is a base definition for Proposition~\ref{relation n d}, so we prove them together. The Proposition~\ref{relation n d} gives the formulation for the latent N-size with gradient penalty.
Moving forward, we provide the proof of the corollary that $\nabla_{\boldsymbol{x}} D(\boldsymbol{x})\leq 0 $, which is an important corollary that deviates our $\boldsymbol\varepsilon$-centered gradient penalty. 
Subsequently, we provide the proof of the Lemma~\ref{lemma compare} which deviates the relationship between the value of the discriminator norm under three different gradient penalties. 
Lastly, we provide the proof of our main Theorem~\ref{theorem main} which summarizes the relationship between the latent N-size and three gradient penalties under the above definitions and Lemmas.}

{In the Experiments section, we provide additional experiments and present more results of our Li-CFG and $\boldsymbol\varepsilon$-centered method with common GAN models.
When working with synthetic datasets, we include visual inspection figures compared our $\boldsymbol\varepsilon$-centered method to other GAN models that use different gradient penalties.
In our work with real-world datasets, we will create visual comparison figures to evaluate the performance of the CFG method and Li-CFG on different databases, including CeleBA, LSUN Bedroom, and ImageNet. We will compare these methods at resolutions of 128x128 and 256x256 pixels. Moreover, we'll also present visual comparison figures of our $\boldsymbol\varepsilon$-centered method with DDGAN and BigGAN in LSUN Church (256x256) and ImageNet (64x64).}

\section{Analysis of the dynamic theory for the CFG}\label{adtftcm}

\noindent \textbf{Dynamic Theory for the CFG.}
This section encompasses the motivation behind and proofs for the relationship between the CFG method and the common GAN method.
In terms of the dynamic theory, the generator and discriminator in the CFG method function equivalently to the corresponding components in common GAN.
Based on this analysis, the CFG method still surfers unstable and not locally convergent problems at the Nash-equilibrium point. We visualize the gradient vector of the discriminator for CFG and Li-CFG in Fig.~\ref{Fig.convergence}. This serves as the impetus for us to introduce the Lipschitz constraint to the CFG method, aiming to enforce convergence. The dynamic theory is imported from~\citet{mescheder2018training} and~\citet{nagarajan2017gradient}.

\begin{figure}
\centering
\includegraphics[width=1\textwidth]{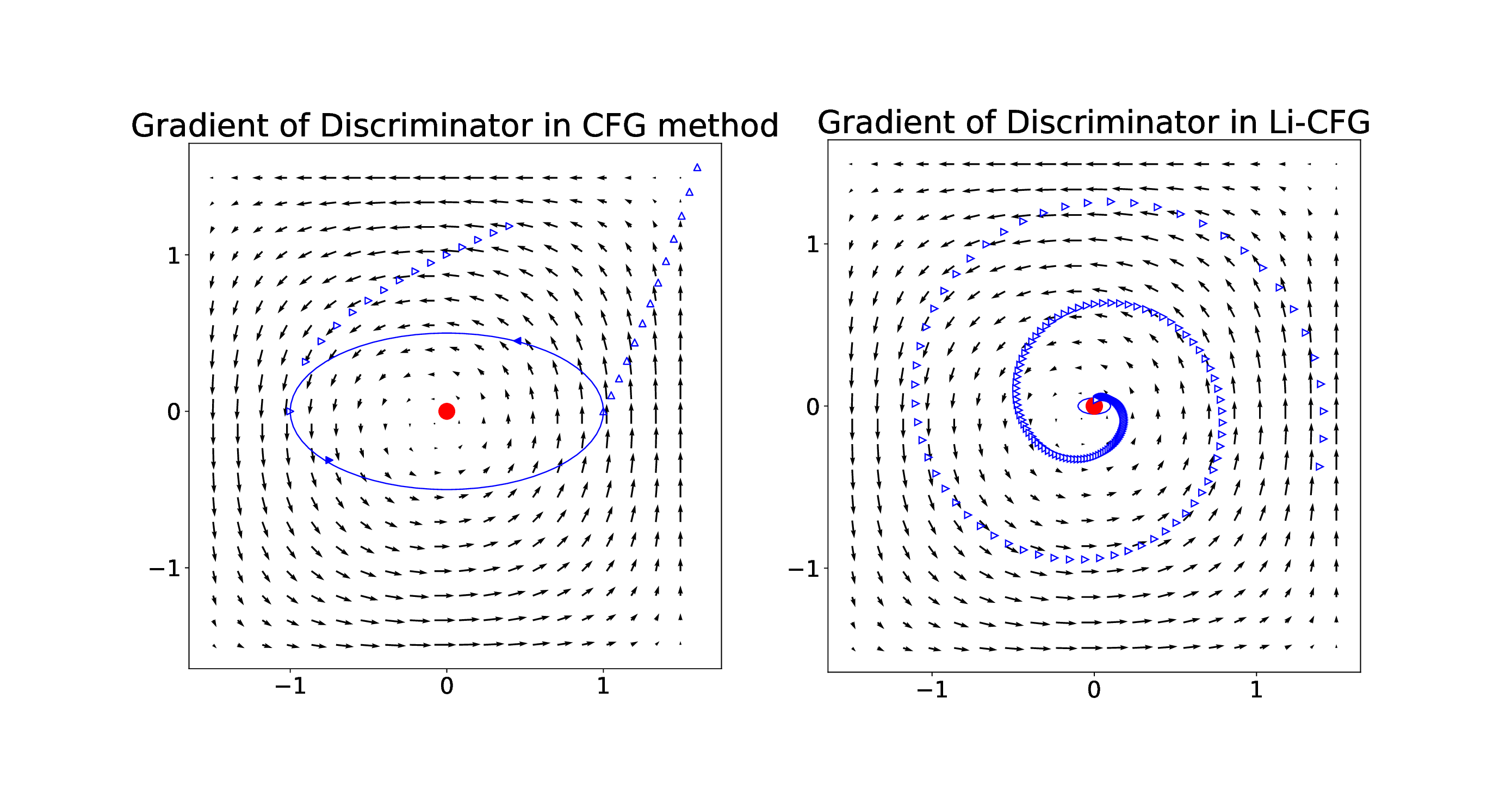}
\setlength{\abovecaptionskip}{-1.2cm}
\caption{ \label{Fig.convergence} Left and right images compare CFG and our Li-CFG, individually. 
% is CFG method. The right image is CFG method with GP.
From the conclusions of Lemma B.1 and Lemma B.2, CFG suffers the non-locally convergent near the Nash-equilibrium problems as common GAN in the dynamic theory. In contrast, our Li-CFG method can 
% the GP regularization, we know that the CFG method 
converges to the small neighborhood around the Nash equilibrium. }
\end{figure}
\begin{lemma}\label{lemma:discriminator}
The \textbf{loss functions of the discriminators} of CFG 
and common GAN learned by minimax games exhibit the same form and optimization objective. The form of the discriminator loss function is:
$$
\mathbb{E}_{\boldsymbol z\sim  p_{z}}\left[f\left(-D_{\psi}\left(G_{\theta}(\boldsymbol z)\right)\right)\right]
+ 
\mathbb{E}_{\boldsymbol x\sim p_{*}}\left[f\left(D_{\psi}(\boldsymbol x)\right)\right]
$$
\end{lemma}

\begin{proof} 
We can establish the equivalence through two main perspectives. Firstly, we analyze the CFG method. The loss function of the discriminator can be described as follows:
\begin{gather}
\begin{aligned}
L_{CFG}=\mathbb{E}_{\boldsymbol x\sim p_{*}}\mathrm{ln}\left(1+\exp(-D(\boldsymbol x))\right)+\mathbb{E}_{\boldsymbol x\sim p_{z}}\mathrm{ln}\left(1+\exp(D(\boldsymbol x))\right).
\end{aligned}
\end{gather}
To unify the symbolic expression, we use $\exp(D(\boldsymbol x))$ instead of $e^{D(\boldsymbol x)}$ from the CFG article in our paper.
The common GAN loss function is
\begin{gather}
L_{GAN}=\mathbb{E}_{\boldsymbol x\sim p_{*}}[\ln d(\boldsymbol x)]+\mathbb{E}_{\boldsymbol z\sim p_{\boldsymbol z}} \ln [1-d(G(\boldsymbol z))].
\end{gather}
If we set the $d(\boldsymbol x) = \frac{1}{1+\exp(-D(\boldsymbol x))}$
and substitute it into the common GAN loss function,
we can get the Loss function of the CFG method. This demonstrates the equivalence between the two loss functions.
From the second part, we can use the notation of the loss function proposed by~\citet{nagarajan2017gradient} to prove the lemma. The loss function is
\begin{gather}
\begin{aligned}
L_{COMM}(\theta, \psi)=\mathbb{E}_{\boldsymbol z\sim  p_{z}}\left[f\left(D_{\psi}\left(G_{\theta}(\boldsymbol z)\right)\right)\right]+ 
\mathbb{E}_{\boldsymbol x\sim p_{*}}\left[f\left(-D_{\psi}(\boldsymbol x)\right)\right].
\end{aligned}
\end{gather}
If we choose the f function to be $f(t)=-\ln \left(1+\exp(-t)\right)$, it leads to the loss function of the CFG method.
\end{proof}
We demonstrate the process with more detailed proof. First, we list the loss functions of the three discriminators for comparison. The discriminator's loss function proposed by ~\citet{nagarajan2017gradient} is
\begin{gather}\label{eq:common_gans}
L_{COMM}(\theta, \psi)=\mathbb{E}_{\boldsymbol z\sim  p_{z}}\left[f\left(-D_{\psi}\left(G_{\theta}(\boldsymbol z)\right)\right)\right]
+ 
\mathbb{E}_{\boldsymbol x\sim p_{*}}\left[f\left(D_{\psi}(\boldsymbol x)\right)\right].
\end{gather}
But in the notation of~\citet{mescheder2018training}, the loss function is 
\begin{gather}\label{eq:common_gans2}
L_{COMM}(\theta, \psi)=\mathbb{E}_{\boldsymbol z\sim  p_{z}}\left[f\left(D_{\psi}\left(G_{\theta}(\boldsymbol z)\right)\right)\right]
+
\mathbb{E}_{x\sim p_{*}}\left[f\left(-D_{\psi}(\boldsymbol x)\right)\right].
\end{gather}
We think both the equation are correct, and choose the Eq.~(\ref{eq:common_gans}) in our proof process. As for simplification, we omit the symbol $(\theta, \psi)$. Furthermore, we use the symbols $D$ and $G$ to denote the discriminator and generator, respectively.
The loss function of the discriminator of the CFG method is
\begin{gather}
L_{CFG}= \mathbb{E}_{\boldsymbol x\sim p_{*}}\mathrm{ln}\left(1+\exp(-D(\boldsymbol x))\right)
+\mathbb{E}_{\boldsymbol z\sim p_{z}}\mathrm{ln}\left(1+\exp(D(G(\boldsymbol z)))\right).
\label{eq:discriminator}
\end{gather}
The discriminator loss function  of common GAN is
\begin{gather}\label{eq:gans}
L_{GAN}=\mathbb{E}_{x\sim p_{*}}[\ln d(\boldsymbol x)]+\mathbb{E}_{\boldsymbol z\sim p_{z}} \ln [1-d(G(\boldsymbol z))].
\end{gather}
Our goal is to prove that the three loss functions are equivalent. One loss function can be transformed into others under certain conditions. 

The proof of Eq.~(\ref{eq:gans})
convert to Eq.~(\ref{eq:discriminator}).

We set the $d(\boldsymbol x) = \frac{1}{1+\exp(-D(\boldsymbol x))}$
and substitute it into the Eq.~(\ref{eq:gans}). Then we have
\begin{gather}
\begin{aligned}
L_{CFG} &= \mathbb{E}_{\boldsymbol x\sim p_{*}}\ln \left[\frac{1}{1+\exp(-D(\boldsymbol x))}\right] 
+\mathbb{E}_{\boldsymbol z\sim p_{z}} \ln\left[1-\frac{1}{1+\exp(-D(G(\boldsymbol z)))}\right]\nonumber\\
&=\mathbb{E}_{\boldsymbol x\sim p_{*}}\ln\left[1+\exp(-D(\boldsymbol x))\right]^{-1}
+
\mathbb{E}_{\boldsymbol z\sim p_{z}} \ln\left[\frac{\exp(-D(G(\boldsymbol z)))}{1+\exp(-D(G(\boldsymbol z)))}\right]\nonumber\\
&=\mathbb{E}_{\boldsymbol x\sim p_{*}}\ln \left[1+\exp(-D(\boldsymbol x))\right]^{-1} 
+
\mathbb{E}_{\boldsymbol z\sim p_{z}} \ln \left[\frac{\frac{1}{\exp(D(G(\boldsymbol z)))}}{\frac{1+\exp(D(G(\boldsymbol z)))}{\exp(D(G(\boldsymbol z)))}}\right]\\\nonumber
&=\mathbb{E}_{\boldsymbol x\sim p_{*}}\ln\left[1+\exp(-D(\boldsymbol x))\right]^{-1}
+
\mathbb{E}_{\boldsymbol z\sim p_{z}} \ln \left[\frac{1}{1+\exp(D(G(\boldsymbol z)))}\right]\\\nonumber
&=\mathbb{E}_{\boldsymbol x\sim p_{*}}\ln\left[1+\exp(-D(\boldsymbol x))\right]^{-1}
+
\mathbb{E}_{\boldsymbol z\sim p_{z}} \ln \left[1+\exp(D(G(\boldsymbol z)))\right]^{-1}\\\nonumber
&=-\bigg[\mathbb{E}_{\boldsymbol x\sim p_{*}}\ln \left[1+\exp(-D(\boldsymbol x))\right]
+
\mathbb{E}_{\boldsymbol z\sim p_{z}} \ln [1+\exp(D(G(\boldsymbol z)))]\bigg]\\\nonumber
&=-\bigg[\mathbb{E}_{\boldsymbol x\sim p_{*}}\ln\left(1+\exp(-D(\boldsymbol x))\right)
+
\mathbb{E}_{\boldsymbol z\sim p_{z}} \ln[1+\exp(D(G(\boldsymbol z)))]\bigg].\nonumber
\end{aligned}
\end{gather}
The equation in the bracket is the loss function of the CFG method. $-\min L_{CFG}$ has the same optimization objectives as the $\max L_{GAN}$. So both loss functions are equivalent. It is easy to transform the Eq.~(\ref{eq:discriminator}) to Eq.~(\ref{eq:gans}) as we set the $\frac{1}{1+\exp(-D(\boldsymbol x))} = d(\boldsymbol x)$.
 
The proof of Eq.~(\ref{eq:common_gans})
convert to Eq.~(\ref{eq:discriminator}).

We choice the f function to be $f(t)=-\ln \left(1+\exp(-t)\right)$. The equation is similar to the $\ln d(\boldsymbol x) = \ln\left[\frac{1}{1+\exp(-D(\boldsymbol x))}\right] = -\ln\left[1+\exp(-D(\boldsymbol x))\right]$ if we substitute the $f$ into Eq.~(\ref{eq:common_gans}), it leads to the loss function of CFG method. Both equations are equivalent. 
So the Lemma~\ref{lemma:discriminator} is proved.
\begin{lemma}\label{appendix:gvf}
The \textbf{gradient vector field of generators} of CFG and common GAN learned by minimax games exhibit the same form in the dynamic theory. The form of the gradient vector field is:
$$
\left(\begin{array}{cc}
\left[-\eta \delta(G(\boldsymbol z))\left[\nabla_{\boldsymbol x} D(G(\boldsymbol z))\right]\frac{\partial G(\boldsymbol z)}{\partial \boldsymbol \theta}\right]
&\\ 
&\\
0 &\end{array}\right)
$$ 
% where $\widetilde v_{G}(\boldsymbol \theta)$, $\overline v_{G}(\boldsymbol \theta)$ are the gradients of common GAN and CFG, respectively.
\end{lemma}
We want to prove that
$
\widetilde v_{G}(\boldsymbol \theta)=\overline v_{G}(\boldsymbol \theta)
$. $\widetilde v_{G}$ denotes the gradient vector field of the common GAN generator, while $\overline v_{G}$ represents the generator gradient vector field of the CFG method. $\boldsymbol \theta$ denotes the weight vector of the generator instead of $\theta$. If the two gradient vector fields are equal to each other, we can say that proof of Lemma~\ref{appendix:gvf} is done.

% $
% \widetilde F_{G}(\theta)=\overline F_{G}(\theta)
% $ stands for the update operator of the generator for GANs and CFG methods respectively. 
\begin{proof} we expand the $\widetilde v_{G}(\boldsymbol\theta)$ , the gradient vector field of generator for common GAN  
% $L=\ln \left(1-d\left(G(z)\right)\right.$
\begin{gather}
\begin{aligned}
\widetilde v_{G}(\boldsymbol\theta)=\left(\begin{array}{cc}\nabla_{\boldsymbol 
\theta} \ln (1-d(G(\boldsymbol z)) &\\ 
&\\
0 & \end{array}\right).
\end{aligned}
\end{gather}
The value of $d$ is $d(\boldsymbol x)=\frac{1}{1+\exp(-D(\boldsymbol x)})$
% $L=\ln \left(\frac{1}{1+e^{D(G(z)}}\right)$
. We substitute  the $d(\boldsymbol x)$ into $\widetilde v_{G}(\boldsymbol \theta)$, then 
\begin{gather}
\begin{aligned}
\widetilde v_{G}(\boldsymbol \theta)=\left(
\begin{array}{cc}
\left[\nabla_{\boldsymbol\theta} \ln \frac{\exp(-D(G(\boldsymbol z)))}{1+\exp(-D(G(\boldsymbol  z)))}\right]
&\\ 
&\\
0 &
\end{array}\right). \\
\end{aligned}
\end{gather}

We extract the $[\nabla_{\boldsymbol\theta} \ln \frac{\exp(-D(G(\boldsymbol z)))}{1+\exp(-D(G(\boldsymbol z)))}]$ out and derive it as follows:
\begin{gather}\label{eq:derivate}
\begin{aligned}
& \nabla_{\boldsymbol \theta} \ln  \frac{\exp(-D(G(\boldsymbol z)))}{1+\exp(-D(G(\boldsymbol z)))}
\\
&=\nabla_{\boldsymbol\theta} \ln \frac{\frac{1}{\exp(D(G(\boldsymbol z)))}}{\frac{1+\exp(D(G(\boldsymbol z)))}{\exp(D(G(\boldsymbol z)))}} \\
&= \nabla_{\boldsymbol\theta} \ln \frac{1}{1+\exp(D(G(\boldsymbol z)))}\\
&= -\frac{\exp(D(G(\boldsymbol z)))}{1+\exp(D(G(\boldsymbol z)))}\left[\nabla_{\boldsymbol x} D(G(\boldsymbol z))\right]\frac{\partial G(\boldsymbol z)}{\partial \boldsymbol\theta}\\
&= -\frac{1}{1+\exp(-D(G(\boldsymbol z)))}\left[\nabla_{\boldsymbol x} D(G(\boldsymbol z))\right]\frac{\partial G(\boldsymbol z)}{\partial \boldsymbol\theta}.\\
\end{aligned}
\end{gather}
We bring it back into the $\widetilde v_{G}(\boldsymbol\theta)$, we will get
\begin{gather}
\begin{aligned}\label{eq:dynamic result}
\widetilde v_{G}(\boldsymbol\theta)=\left(\begin{array}{cc}
\left[-\frac{1}{1+\exp(-D(G(\boldsymbol z)))}\left[\nabla_{\boldsymbol x} D(G(\boldsymbol z))\right]\frac{\partial G(\boldsymbol z)}{\partial \boldsymbol\theta}\right]
&\\&\\  0 &\end{array}\right). \\
\end{aligned}
\end{gather}
This is the gradient vector field of common GAN. 

Now we check the gradient vector field of the  CFG method. For the CFG method, a hyper-parameter $M$ has been used to control the varying steps of generators
\begin{gather}
G_{m+1}(\boldsymbol z)=G_{m}(\boldsymbol 
 z)+\eta_{m}g_{m}\left(G_{m}(\boldsymbol 
 z)\right).
\label{eq:G}
\end{gather}
The variables $M$ and $m$ have the same meaning.
we will discuss the M in two cases to analyze the gradient vector field of the generator for the CFG method.
The loss function of the generator for the CFG method is 
\begin{gather}
L(\boldsymbol \theta)=\frac{1}{2}\left[G_{m+1}(\boldsymbol 
 z)-G_{1}(\boldsymbol z)\right]^{2}.
 \end{gather}

First, when we consider $M$=1. In this case, the loss function can be written as 
\begin{gather}
L(\boldsymbol \theta)=\frac{1}{2}\left[G_{2}(\boldsymbol z)-G_{1}(\boldsymbol z)\right]^{2}.
\end{gather}
. Then the gradient vector field can be written as below:
\begin{gather}
\begin{aligned}
\overline v_{G}(\boldsymbol \theta)&=\left(\begin{array}{cc}
\left[\left[G_{2}(\boldsymbol z)-G_{1}(\boldsymbol z)\right] \frac{\partial G(\boldsymbol z)}{\partial \boldsymbol \theta} \right]
&\\ 
&\\
0 &\end{array}\right) \\
&=\left(\begin{array}{cc}
\left[\left[\eta_1 g_1(G(\boldsymbol z))\right] \frac{\partial G(\boldsymbol z)}{\partial \boldsymbol\theta}\right]
&\\ 
&\\
0 &\end{array}\right) \\
&=\left(\begin{array}{cc}
\left[-\eta_1 \delta(G(\boldsymbol z))\left[\nabla_{\boldsymbol x} D(G(\boldsymbol z))\right]\frac{\partial G(\boldsymbol z)}{\partial \boldsymbol \theta}\right]
&\\ 
&\\
0 &\end{array}\right).
\end{aligned}
\end{gather}
% $$
% L(\theta, \psi)=f(\psi(\theta)+0 \cdot f(\psi)
% $$

% $d(x)=\frac{1}{1+e^{-D(x)}}$

% $G_{t+1}(z)=G_{t}(z)+y_{t} g_{t}\left(G_{t}(z)\right)$
For $g(\boldsymbol x)=\delta(\boldsymbol x) \cdot \nabla_{\boldsymbol x} D(\boldsymbol x)$, we archive the above result. Now we compare the form of
$
\widetilde v_{G}(\theta)$ and $\overline v_{G}(\theta)
$. If we set the $-\frac{1}{1+\exp(-D(G(\boldsymbol z)))}$ = $-\eta_1 \delta(G(\boldsymbol z)$, the two function are both scaling factor.
We find two equations get exactly the same form.
 So the proof is done.

Next, When we consider the second case, $M>1$ and $M \rightarrow \infty$. The expansion of Eq.~(\ref{eq:G}) can be written as below:
\begin{gather}
\begin{aligned}
G_{m+1}(\boldsymbol z)=&G_{m}(\boldsymbol z)+\eta_{m} g_{m}\left(G_{m-1}(\boldsymbol z)\right) \\
&\vdots\\
G_{3}(\boldsymbol z)=&G_{2}(\boldsymbol z)+\eta_2 g_2\left(G_{2}(\boldsymbol z)\right) \\
G_{2}(\boldsymbol z)=&G_1(\boldsymbol z)+\eta_{1} g_{1}\left(G_{1}(\boldsymbol z)\right).
\end{aligned}
\end{gather}
Let us sum the right side of the equation list, and we get the form of $G_{m+1}(z)$ by $G_{1}$ as:
\begin{gather}
G_{m+1}(\boldsymbol z)=G_{1}(\boldsymbol z)+\sum_{m=1}^{M}\eta_{m}g_{m}\left(G_{m}(\boldsymbol z)\right).
\end{gather}
Now we consider the $\overline v_{G}(\boldsymbol \theta)$. When $M>1$, the gradient vector field can be written as
\begin{gather}
\begin{aligned}
\overline v_{G}(\boldsymbol \theta)&=\left(\begin{array}{cc}
\left[\nabla_{\boldsymbol\theta}\frac{1}{2}\left[G_{m+1}(\boldsymbol z)-G_{1}(\boldsymbol z)\right]^{2} \right]
&\\ 
&\\
0 &\end{array}\right) \\
% &=\left(\begin{array}{c}
% \left[\nabla_{\theta}\|G_{2}(z)-G_{1}(z)+\sum_{t=2}^{t}\eta_{t}g_{t}\left(G_{t}(z)\right)\|\right]
% &\\ 0 &\end{array}\right) \\
&=\left(\begin{array}{cc}
\left[\nabla_{\boldsymbol \theta}\frac{1}{2}\left[\sum\limits_{m=1}^{M}\eta_{m}g_{m}\left(G_{m}(\boldsymbol z)\right)\right]^2\right]
&\\ 
&\\
0 &\end{array}\right) \\
&=\left(\begin{array}{cc}
\left[\sum\limits_{m=1}^{M}\eta_{m}g_{m}\left(G_{m}(\boldsymbol z)\right)\right]\nabla_{\boldsymbol \theta}\left[\sum\limits_{j=1}^{M}G_{j}(\boldsymbol z)\right]
&\\ 
&\\
0 &\end{array}\right) \\
&=\left(\begin{array}{cc}
\left[\sum\limits_{m=1}^{M}\eta_{m}g_{m}\left(G_{m}(\boldsymbol z)\right)\right]\left[\sum\limits_{j=1}^{M}\frac{\partial G_{j}(\boldsymbol z)}{\partial \boldsymbol \theta}\right]
&\\ 
&\\
0 &\end{array}\right) \\
&=\left(\begin{array}{cc}
\left[\sum\limits_{m=1}^{M}\sum\limits_{j=1}^{M}\eta_{m}g_{m}\left(G_{m}(\boldsymbol z)\right)\frac{\partial G_{j}(\boldsymbol z)}{\partial \boldsymbol \theta}\right]
&\\ 
&\\
0 &\end{array}\right) \\
&=\left(\begin{array}{cc}
\left[\sum\limits_{m=1}^{M}\sum\limits_{j=1}^{M}-\eta_m \delta(G_{m}(\boldsymbol z))\left[\nabla_{\boldsymbol x} D(G_{m}(\boldsymbol z))\right]
\frac{\partial G_{j}(\boldsymbol z)}{\partial \boldsymbol \theta}\right]
&\\ 
&\\
0 &\end{array}\right).\\
\end{aligned}
\end{gather}

So we get the gradient vector field of the CFG method when $T>1$. Every sub-item in the equation is a gradient vector that has the same form as the gradient vector when $T=1$, which is equivalent to common GAN. The Sum of these gradient vectors is CFG method $T>1$ which is also equivalent to common GAN. So Lemma 3.2 is proved.
 \end{proof}

\section{Theoretical analysis of our theory}
\label{appendix proof}
\subsection{Latent N-size with gradient penalty}
In this section, we will provide a detailed proof procedure for Definition~\ref{definition-radius} and Proposition~\ref{relation n d}. Since Definition~\ref{definition-radius} is the foundational definition for Proposition~\ref{relation n d}, the detailed proof procedure for Definition~\ref{definition-radius} will be presented in the proof of Proposition~\ref{relation n d}.

\noindent \textbf{Latent N-size.}
We illustrate the concept of latent N-size, which serves as the foundational definition for the subsequent theory.
% \noindent \textbf{Lemma \uppercase\expandafter{\romannumeral1}}
\begingroup
\def\thetheorem{\ref{definition-radius}}
\begin{definition}
Let $\boldsymbol{z}_1$, $\boldsymbol{z}_2$ be two samples in the latent space. Suppose $\boldsymbol z_1$ is attracted to the mode $\mathcal{M}_{i}$ by $\hat{\epsilon}$, then there exists a neighborhood $\mathcal{N}_r\left(\boldsymbol{z}_1\right)$ of $\boldsymbol{z}_1$ such that $\boldsymbol{z}_2$ is distracted to $\mathcal{M}_{i}$ by $(\hat\epsilon/ 2-2\alpha)$, for all $\boldsymbol{z}_2 \in \mathcal{N}_r\left(\boldsymbol{z}_1\right)$. The size of $\mathcal{N}_r\left(\boldsymbol{z}_1\right)$ can be arbitrarily large but is bounded by an open ball of radius $r$. The $r$ is defined as
$$
  r=\hat\epsilon \cdot\left(2 \inf _{\boldsymbol{z}}\left\{\left(\frac{\left\|G_{\theta_t}\left(\boldsymbol z_1\right)-G_{\theta_t}(\boldsymbol z)\right\|}{\left\|\boldsymbol z_1-\boldsymbol z\right\|}+ \frac{\left\|G_{\theta_{t+1}}\left(\boldsymbol z_1\right)-G_{\theta_{t+1}}(\boldsymbol{z})\right\|}{\left\|\boldsymbol z_1-\boldsymbol{z}\right\|}\right)\right\}\right)^{-1},$$
\end{definition}
\addtocounter{theorem}{-1}
\endgroup
where the meaningful of mode $\mathcal{M}$, $attracted$ and $distracted$ are present in definition~\ref{definition1},~\ref{definition2},~\ref{definition3}, respectively.

According to this definition, the radius $r$ is inversely proportional to the discrepancy between the preceding and subsequent outputs of the generator, given a similar latent vector $\boldsymbol z$. 
A large value of $r$ results in a small difference between the previous and subsequent generator outputs, leading to mode collapse. 
Conversely, a small value of $r$ leads to a large difference, resulting in diverse synthesis.

\noindent \textbf{Latent N-size and the diversity.}
We offer three fundamental definitions of the latent N-size and demonstrate its relationship to the diversity of synthetic samples.

\begingroup
\def\thetheorem{\ref{definition1}}
\begin{definition}[Modes in Image Space]
There exist some modes $\mathcal{M}$ cover the image space $\mathcal{Y}$.
Mode $\mathcal{M}_{i} $ is a subset of $\mathcal{Y}$ satisfying $\max _{\boldsymbol{y_{k,j}} \in \mathcal{M}_{i}}\left\|\boldsymbol{y_{k}}-\boldsymbol{y_{j}}\right\|<\alpha$ and  
$\min_{\boldsymbol{y_{k}} \in \mathcal{M}_{i},\boldsymbol{y_{m}} \not\in \mathcal{M}_{i}}\alpha<\left\|\boldsymbol{y_{m}}-\boldsymbol{y_{k}}\right\|<2\alpha$, where $\boldsymbol{y_{k}}$ and $\boldsymbol{y_{j}}$ belong to the same mode $\mathcal{M}_{i}$,  $\boldsymbol{y_{m}}$ and $\boldsymbol{y_{k}}$ belong to different modes $\mathcal{M}_{i}$, and $\alpha>0$.
\end{definition}
\addtocounter{theorem}{-1}
\endgroup

\begingroup
\def\thetheorem{\ref{definition2}}
\begin{definition}[Modes Attracted]
  Let $\boldsymbol{z}_1$ be a sample in latent space, we say $\boldsymbol{z}_1$ is attracted to a mode $\mathcal{M}_{i}$ by $\hat{\epsilon}$ from a gradient step if $\left\|\boldsymbol{y_{k}}-G_{\theta_{t+1}}\left(\boldsymbol{z}_1\right)\right\|+\hat{\epsilon}<\left\|\boldsymbol{y_{k}}-G_{\theta_t}\left(\boldsymbol{z}_1\right)\right\|$, where $\boldsymbol{y_{k}} \in \mathcal{M}_{i}$ is an image in a mode $\mathcal{M}_{i}$, $\hat{\epsilon}$ denotes a small quantity, $\theta_{t}$ and $\theta_{t+1}$ are the generator parameters before and after the gradient updates respectively.  
\end{definition}
\addtocounter{theorem}{-1}
\endgroup

\begingroup
\def\thetheorem{\ref{definition3}}
\begin{definition}[Modes Distracted]
Let $\boldsymbol{z}_2$ be a sample in latent space, we say $\boldsymbol{z}_2$ is distracted to a mode $\mathcal{M}_{i}$ by $(\alpha+\hat{\epsilon})/2$ from a gradient step if $\left\|\boldsymbol{y_{m}}-G_{\theta_{t+1}}\left(\boldsymbol{z}_2\right)\right\|+(\hat\epsilon/ 2-2\alpha)<\left\|\boldsymbol{y_{k}}-G_{\theta_t}\left(\boldsymbol{z}_2\right)\right\|$, where $\boldsymbol{y_{k}} \in \mathcal{M}_{i}$ is an image in a mode $\mathcal{M}_{i}$, $\boldsymbol{y_{m}} \not\in \mathcal{M}_{i}$ is an image from other modes,  $\alpha$ keeps the same meaning as in Definition~\ref{definition1}, $\theta_{t}$ and $\theta_{t+1}$ are the generator parameters before and after the gradient updates respectively.     
\end{definition}
\addtocounter{theorem}{-1}
\endgroup

\noindent \textbf{Latent N-size with gradient penalty.}
We demonstrate the relationship between the latent N-size and the gradient penalty. 
According to Proposition~\ref{relation n d}, as $\|\nabla_{\boldsymbol x} D(\boldsymbol{x})\|$ increases, the latent N-size decreases, and vice versa.

\begingroup
\def\thetheorem{\ref{relation n d}}
\begin{proposition}
$\mathcal{N}_r\left(\boldsymbol{z}_1\right)$ can be defined with discriminator gradient penalty as follows: 
\begin{small}
$$
r=\hat{\epsilon}\cdot\left(2 \inf _{\boldsymbol{z}}\left\{ \left(\frac{2\left\|G_{\theta_t}\left(\boldsymbol{z}_1\right)-G_{\theta_t}(\boldsymbol{z})\right\|+\eta_m \delta(\boldsymbol x) \sum\limits_{m=1}^{N}\left(\|\nabla_{\boldsymbol x} D_{m}(\mathcal{Y}_2)\|+\|\nabla_{\boldsymbol x} D_{m}(\mathcal{Y})\|\right)}{\left\|\boldsymbol{z}_1-\boldsymbol{z}\right\|}\right)\right\}\right)^{-1}
$$
\end{small}
, where $\|\nabla_{\boldsymbol x} D_{m}(\mathcal{Y}_2)\|=$
$\|\nabla_{\boldsymbol x} D_{m}(G_{\theta_t}(\boldsymbol{z}_2))+R\|$ and   $\|\nabla_{\boldsymbol x} D_{m}(\mathcal{Y})\|=$
$\|\nabla_{\boldsymbol x} D_{m}(G_{\theta_t}(\boldsymbol{z}))+R\|$. $R$ stands for the discriminator gradient penalty.
\end{proposition}
\addtocounter{theorem}{-1}
\endgroup

\begin{proof}
    With the gradient penalty of CFG formulation
\begin{gather}\label{eq:cfg}
\begin{aligned}
        G_{\theta_{t+1}}\left(\boldsymbol{z}\right) &= G_{\theta_{t}}\left(\boldsymbol{z}\right) + \sum\limits_{m=1}^{N}\eta_m \delta(\boldsymbol{x})\left(\nabla_{\boldsymbol{x}} D_{m}(G_{\theta_t}(\boldsymbol{z}))\right) \\
&=G_{\theta_{t}}\left(\boldsymbol{z}\right) + \sum\limits_{m=1}^{N}\eta_m \delta(\boldsymbol{x})\left(\nabla_{\boldsymbol{x}} D_{m}(G_{\theta_t}(\boldsymbol{z}))+ R  \right), 
\end{aligned}
\end{gather}
we are now ready to prove the Proposition~\ref{relation n d}.
\begin{gather}
\begin{aligned}
\left\|\boldsymbol{y_m}-G_{\theta_{t+1}}\left(\boldsymbol{z}_2\right)\right\| & 
\leq\left\|\boldsymbol{y_m}-\boldsymbol{y_k}\right\|+\left\|\boldsymbol{y_k}-G_{\theta_{t+1}}\left(\boldsymbol{z}_2\right)\right\| \\
&\leq\left\|\boldsymbol{y_m}-\boldsymbol{y_k}\right\|+\left\|\boldsymbol{y_k}-G_{\theta_{t+1}}\left(\boldsymbol{z}_1\right)\right\|+\left\|G_{\theta_{t+1}}\left(\boldsymbol{z}_1\right)-G_{\theta_{t+1}}\left(\boldsymbol{z}_2\right)\right\| \\
& <\left\|\boldsymbol{y_m}-\boldsymbol{y_k}\right\|+ \left\|\boldsymbol{y_k}-G_{\theta_t}\left(\boldsymbol{z}_1\right)\right\|+\left\|G_{\theta_{t+1}}\left(\boldsymbol{z}_1\right)-G_{\theta_{t+1}}\left(\boldsymbol{z}_2\right)\right\|-\hat\epsilon \quad \\
& \leq \left\|\boldsymbol{y_m}-\boldsymbol{y_k}\right\|+ \left\|\boldsymbol{y_k}-G_{\theta_t}\left(\boldsymbol{z}_2\right)\right\|+\left\|G_{\theta_t}\left(\boldsymbol{z}_2\right)-G_{\theta_t}\left(\boldsymbol{z}_1\right)\right\|
\\
&+\left\|G_{\theta_{t+1}}\left(\boldsymbol{z}_1\right)-G_{\theta_{t+1}}\left(\boldsymbol{z}_2\right)\right\|-\hat\epsilon \\
& =\left(\frac{\left\|G_{\theta_t}\left(\boldsymbol{z}_1\right)-G_{\theta_t}\left(\boldsymbol{z}_2\right)\right\|}{\left\|\boldsymbol{z}_1-\boldsymbol{z}_2\right\|}+\frac{\left\|G_{\theta_{t+1}}\left(\boldsymbol{z}_1\right)-G_{\theta_{t+1}}\left(\boldsymbol{z}_2\right)\right\|}{\left\|\boldsymbol{z}_1-\boldsymbol{z}_2\right\|}\right)\left\|\boldsymbol{z}_1-\boldsymbol{z}_2\right\| \\
&+\left\|\boldsymbol{y_m}-\boldsymbol{y_k}\right\|+ \left\|\boldsymbol{y_k}-G_{\theta_t}\left(\boldsymbol{z}_2\right)\right\|-\hat\epsilon .
\end{aligned}
\end{gather}
This implies
$$
\left\|\boldsymbol{y_m}-G_{\theta_{t+1}}\left(\boldsymbol{z}_2\right)\right\|+(\frac{\hat\epsilon}{2}-2\alpha )<\left\|\boldsymbol{y_k}-G_{\theta_t}\left(\boldsymbol{z}_2\right)\right\|,
$$ which is the content of definition \ref{definition3}
and
$$
\left(\frac{\left\|G_{\theta_t}\left(\boldsymbol{z}_1\right)-G_{\theta_t}\left(\boldsymbol{z}_2\right)\right\|}{\left\|\boldsymbol{z}_1-\boldsymbol{z}_2\right\|}+\frac{\left\|G_{\theta_{t+1}}\left(\boldsymbol{z}_1\right)-G_{\theta_{t+1}}\left(\boldsymbol{z}_2\right)\right\|}{\left\|\boldsymbol{z}_1-\boldsymbol{z}_2\right\|}\right)\left\|\boldsymbol{z}_1-\boldsymbol{z}_2\right\| \leq \frac{\hat\epsilon}{2} .
$$
We define the latent N-size of $\boldsymbol{z}_1$ is
$$
\mathcal{N}_\tau\left(\boldsymbol{z}_1\right)=\left\{\boldsymbol{z}:\left(\frac{\left\|G_{\theta_t}\left(\boldsymbol z_1\right)-G_{\theta_t}(\boldsymbol z)\right\|}{\left\|\boldsymbol z_1- \boldsymbol z\right\|}+ \frac{\left\|G_{\theta_{t+1}}\left(\boldsymbol z_1\right)-G_{\theta_{t+1}}(\boldsymbol{z})\right\|}{\left\|\boldsymbol z_1-\boldsymbol{z}\right\|}\right) \leq \tau\right\}.
$$
Then, the maximum latent N-size 
\begin{align}\label{radius}
  r=\hat\epsilon \cdot\left(2 \inf _{\boldsymbol{z}}\left\{\left(\frac{\left\|G_{\theta_t}\left(\boldsymbol z_1\right)-G_{\theta_t}(\boldsymbol z)\right\|}{\left\|\boldsymbol z_1-\boldsymbol z\right\|}+ \frac{\left\|G_{\theta_{t+1}}\left(\boldsymbol z_1\right)-G_{\theta_{t+1}}(\boldsymbol{z})\right\|}{\left\|\boldsymbol z_1-\boldsymbol{z}\right\|}\right)\right\}\right)^{-1} . 
\end{align}
Here, we finish the proof of the definition~\ref{definition-radius}.

Then, we bring the CFG definition of generator Eq.~(\ref{eq:cfg})
into the Eq.~(\ref{radius}), and expand the sum of the additions in parentheses as below:
\begin{gather}\label{eq:stand}
\begin{aligned}
&\frac{\left\|G_{\theta_t}\left(\boldsymbol z_1\right)-G_{\theta_t}(\boldsymbol z)\right\|}{\left\|\boldsymbol z_1-\boldsymbol z\right\|}+ \frac{\left\|G_{\theta_{t+1}}\left(\boldsymbol z_1\right)-G_{\theta_{t+1}}(\boldsymbol z)\right\|}{\left\|\boldsymbol z_1-\boldsymbol z\right\|}\\
&=\frac{\left\|G_{\theta_t}\left(\boldsymbol z_1\right)-G_{\theta_t}(\boldsymbol z)\right\|}{\left\|\boldsymbol z_1-\boldsymbol z\right\|}\\
&+ \frac{\left\|G_{\theta_t}\left(\boldsymbol z_1\right)-G_{\theta_t}(\boldsymbol z)+\sum\limits_{m=1}^{N}\eta_m \delta(\boldsymbol x)[(\nabla_{\boldsymbol x}  D_{m}(G_{\theta_t}(\boldsymbol z_1))+R)-(\nabla_{\boldsymbol x} D_{m}(G_{\theta_t}(\boldsymbol z))+R)]\right\|}{\left\|\boldsymbol z_1-\boldsymbol z\right\|}\\
&\leq \frac{2\left\|G_{\theta_t}\left(\boldsymbol z_1\right)-G_{\theta_t}(\boldsymbol z)\right\|}{\left\|\boldsymbol z_1-\boldsymbol z\right\|}\\
&+\frac{\left\|\sum\limits_{m=1}^{N}\eta_m \delta(\boldsymbol x)[(\nabla_{\boldsymbol x} D_{m}(G_{\theta_t}(\boldsymbol z_1))+R)-(\nabla_{\boldsymbol x} D_{m}(G_{\theta_t}(\boldsymbol z))+R)]\right\|}{\left\|\boldsymbol z_1-\boldsymbol z\right\|}\\
&\leq \frac{2\left\|G_{\theta_t}\left(\boldsymbol z_1\right)-G_{\theta_t}(\boldsymbol z)\right\|}{\left\|\boldsymbol z_1-\boldsymbol z\right\|}\\
&+\frac{\left\|\sum\limits_{m=1}^{N}\eta_m \delta(\boldsymbol x)(\nabla_{\boldsymbol x} D_{m}(G_{\theta_t}(\boldsymbol z_1))+R)\right\|+\left\|\sum\limits_{m=1}^{N}\eta_m \delta(\boldsymbol x)(\nabla_{\boldsymbol x} D_{m}(G_{\theta_t}(\boldsymbol z))+R)\right\|}{\left\|\boldsymbol z_1-\boldsymbol z\right\|}.\\
\end{aligned}
\end{gather}
As for simple, we define 
$\|\nabla_{\boldsymbol x} D_{m}(\mathcal{Y}_2)\|=$
$\|\nabla_{\boldsymbol x} D_{m}(G_{\theta_t}(\boldsymbol{z}_2))+R\|$ and   $\|\nabla_{\boldsymbol x} D_{m}(\mathcal{Y})\|=$
$\|\nabla_{\boldsymbol x} D_{m}(G_{\theta_t}(\boldsymbol{z}))+R\|$, and extract the symbol $\eta_m\delta(\boldsymbol x) \sum\limits_{m=1}^{N}$ out of the norm. So we finish the proof of Proposition~\ref{relation n d}.
\end{proof}

\subsection{\texorpdfstring{$\boldsymbol\varepsilon$}{-}-centered GP}

Observe from  Eq.~(\ref{eq:theorem 2.1}) and Eq.~(\ref{eq:empirical}) in CFG method, we can derive the following conclusion: $\nabla_{\boldsymbol{x}} D(\boldsymbol{x})\leq 0 $.
This is an important corollary derived from the CFG method. With this corollary, we can prove that the latent N-size corresponding to our $\boldsymbol{\varepsilon}$-centered GP is the smallest among all three gradient penalties in the following section.
\begin{proof}
The Eq.~(\ref{eq:theorem 2.1}) from CFG implies that for $\frac{d L\left(p_m\right)}{d m}$ to be negative so that the distance $L$ 
decreases, we should choose $g_m(\boldsymbol x)$ to be:
$$g_m(\boldsymbol x)=-s_m(\boldsymbol x) \phi_0\left(\nabla_{\boldsymbol x} \ell_2^{\prime}\left(p_*(\boldsymbol x), p_m(\boldsymbol x)\right)\right),$$
where $s_m(\boldsymbol x)>0$ is an arbitrary scaling factor. $\phi_0(u)$ is a vector function such that $\phi(u)=u \cdot \phi_0(u) \geq 0$ and that $\phi(u)=0$ if and only if $u=0$, e.g., $\left(\phi_0(u)=u, \phi(u)=\|u\|_2^2\right)$ or $\left(\phi_0(u)=\operatorname{sign}(u), \phi(u)=\|u\|_1\right)$. With this choice of $g_m(\boldsymbol x)$, we obtain
$$
\frac{d L\left(p_m\right)}{d m}=-\int s_m(\boldsymbol x) p_m(\boldsymbol x) \phi\left(\nabla_{\boldsymbol x} \ell_2^{\prime}\left(p_*(\boldsymbol x), p_m(\boldsymbol x)\right)\right) d \boldsymbol x \leq 0 ,
$$
that is, the distance $L$ is guaranteed to decrease unless the equality holds.
So, we can know that $g_m(\boldsymbol x)\leq 0$.
Then, we scrutiny the Eq.~(\ref{eq:empirical})  from which can deviate the equation:
\begin{gather}
\begin{aligned}
  g_m(\boldsymbol{x})&=-s_m(\boldsymbol{x}) \phi_0\left(\nabla_x \ell_2^{\prime}\left(p_*(\boldsymbol{x}), p_m(\boldsymbol{x})\right)\right) \\
  &=-s_m(\boldsymbol{x})\nabla_x \ell_2^{\prime}\left(p_*(\boldsymbol{x}), p_m(\boldsymbol{x})\right)\\
  &=-s_m(\boldsymbol{x})f^{\prime \prime}\left(r_m(\boldsymbol{x})\right) \nabla r_m(\boldsymbol{x})\\
  &\approx s_m(\boldsymbol{x})f^{\prime \prime}\left(\tilde{r}_m(\boldsymbol{x})\right) \tilde{r}_m(\boldsymbol{x})\nabla_x D(\boldsymbol{x}),\nonumber
\end{aligned}
\end{gather}
where $s_m(x)>0$ is an arbitrary scaling factor, $\ell_2\left(\rho_*, \rho\right)=\rho_* f\left(\rho / \rho_*\right)$, $\nabla_x \ell_2^{\prime}\left(p_*(x), p_m(x)\right)=f^{\prime \prime}\left(r_m(x)\right) \nabla r_m(x)$,  $f^{\prime\prime}=1/x^2$ when $f$ is KL-divergence and $r_m(\boldsymbol{x})=\exp(-D(\boldsymbol{x}))\approx p_m(x)/p_*(x)=\tilde{r}_m(\boldsymbol{x})$ when $D(x)\approx \ln \frac{p_*(x)}{p_m(x)}$, which is the analytic solution of the CFG discriminator. Based on the above formulation, we can find that
$s(\boldsymbol x)$ is a scaling function that is always greater than 0,
$\tilde{r}(\boldsymbol x)$ is a exponential function, and $f_{kl}''(\tilde{r}(\boldsymbol x))$ is a non negative function composed of exponential function $\tilde{r}(\boldsymbol x)$. So if equation $g_m(\boldsymbol x)\leq 0$ holds, the $ \nabla_{\boldsymbol x} D(\boldsymbol x)\leq 0$ also holds. 
\end{proof}

We define our method as the $\boldsymbol\varepsilon$-centered Gradient Penalty. We use notation $\boldsymbol\varepsilon$ in our penalty name and equation to differ from the hyper-parameter $\varepsilon'$.
The $\boldsymbol{\varepsilon}$-centered GP is
\begin{gather}
R(\theta, \psi)=\frac{\gamma}{2} \mathrm{E}_{\hat{\boldsymbol x}}\left(\left\|\nabla_{\boldsymbol x} D_{\psi}(\hat{\boldsymbol x})-\boldsymbol \varepsilon\right\|\right)^{2},
\end{gather}
 where $\boldsymbol{\varepsilon}$ is a vector such that $\Vert \boldsymbol\varepsilon \Vert_{2}= \varepsilon'$
 with $\varepsilon'=\sqrt{C\cdot N^2\cdot \boldsymbol\varepsilon^{2}}$,  $N$ and $C$ are dimensions and channels of the real data respectively. $\hat{\boldsymbol x}$ is sampled uniformly on the line segment between two random points vector $\boldsymbol x_{1} \sim p_{\theta}\left(\boldsymbol x_{1}\right), \boldsymbol x_{2} \sim p_{\mathcal{D}}\left(\boldsymbol x_{2}\right)$.

Combining the corollary $\nabla_{\boldsymbol{x}} D(\boldsymbol{x})\leq 0 $, our $\boldsymbol{\varepsilon}$-centered GP increases the $\|\nabla_{\boldsymbol{x}} D(\boldsymbol{x})\|$ as to achieve a better latent N-size which other two gradient penalty behaviors worse.

When training the GAN model with the loss function  Eq.~(\ref{general-d}) and Eq.~(\ref{epsilon-center-GP}), our approach will control the gradient of the discriminator and result in a diversity of synthesized samples.

\subsection{Latent N-size with different gradient penalty}

In this section, we will derive our final theorem about the latent N-size with different gradient penalties.

First, we establish the relationship among the latent N-size for three gradient penalties in the following Lemma. 
We will substitute different types of gradient penalty into the Proposition~\ref{relation n d}.
% when $z_1 \neq z$, $\|G_{\theta_t}\left(z_1\right)-G_{\theta_t}(z)\| \rightarrow 0 $, $\|\nabla_{x} D_{m}G_{\theta_t}\left(z_1\right)-\nabla_{x} D_{m}G_{\theta_t}(z)\| \rightarrow 0$. 

\begingroup
\def\thetheorem{\ref{lemma compare}}
\begin{lemma}
The norms of the three Gradient Penalties, which dictate the latent N-size, are defined as follows: $\|R_1\|$ =  $\|\left(\left\|\nabla_{\boldsymbol x} D_{m}(G_{\theta_t}(\boldsymbol z))\right\|-g_0\right)\|$, $\|R_0\|=\|\left(\left\|\nabla_{\boldsymbol x} D_{m}(G_{\theta_t}(\boldsymbol z))\right\|\right)\| $, $ \|R_{\boldsymbol\varepsilon}\|$=$\|\left(\left\|\nabla_{\boldsymbol x} D_{m}(G_{\theta_t}(\boldsymbol z))\right\|+\|\boldsymbol\varepsilon\|\right)\|$, 
respectively. 
The order of magnitude between the norms of three Gradient Penalty is $\|R_1\|<\|R_0\|<\|R_{\boldsymbol\varepsilon}\|$. Consequently, the relationship between the latent N-size of three Gradient Penalty is $r_{R_1}>r_{R_0}>r_{R_{\boldsymbol\varepsilon}}$.
\end{lemma}
\addtocounter{theorem}{-1}
\endgroup
\begin{proof}
Let us start with the result from the Eq~.\ref{eq:stand}. We will insert three distinct Gradient Penalties into it and determine the latent N-size for each of these penalties. Next, we will focus on the primary component of the three equations and compare the values among them.

The equation of the $1$-centered gradient penalty is 
\begin{gather}\label{eq:gp1}
\begin{aligned}
R_{g_0}&=\frac{\gamma}{2} \mathrm{E}_{\hat{\boldsymbol x}}\left(\left\|\nabla_{\boldsymbol x} D_{m}(\hat{\boldsymbol x})\right\|-g_{0}\right)^{2}. 
% \\
% \nabla_{x} R_{g_0}&=\gamma\mathrm{E}_{\hat{x}}\left(\left\|\nabla_{x} D_{m}(\hat{x})\right\|-g_{0}\right)(\nabla_{x}^2 D_{m}(\hat{x}))
\end{aligned}
\end{gather}
The equation of $0$-centered gradient penalty is
\begin{gather}\label{eq:gp0}
\begin{aligned}
R_{0} &=\frac{\gamma}{2} \mathrm{E}_{\hat{\boldsymbol x}}\left\|\nabla_{\hat{\boldsymbol x}} D_{\psi}(\hat{\boldsymbol x})\right\|^{2}.
% \\
% \nabla_{x} R_{0}&=\gamma\mathrm{E}_{\hat{x}}\left\|\nabla_{x} D_{m}(\hat{x})\right\|(\nabla_{x}^2 D_{m}(\hat{x}))
\end{aligned}
\end{gather}
The equation of $\varepsilon$-centered gradient penalty is
\begin{gather}\label{eq:gp our}
\begin{aligned}
R_{\varepsilon}&=\frac{\gamma}{2} \mathrm{E}_{\hat{\boldsymbol x}}\left\|\nabla_{\boldsymbol x} D_{\psi}(\hat{\boldsymbol x})-\boldsymbol \varepsilon\right\|^{2}.
% \\
% \nabla_{x} R_{\varepsilon}&=\gamma\mathrm{E}_{\hat{x}}\left(\left\|\nabla_{x} D_{m}(\hat{x})-\boldsymbol\varepsilon\right\|\right)(\nabla_{x}^2 D_{m}(\hat{x}))
\end{aligned}
\end{gather}
We bring them back to the Eq~.(\ref{eq:cfg}) and expand the Gradient Penalty expression. We omit the expectation symbol and focus on the gradient. Then,  we plug the above results 
% into Eq~.(\ref{eq:gp1}), (\ref{eq:gp0}) (\ref{eq:gp our}) 
into Eq~.\ref{eq:stand} and we have the different gradient penalty equation for the sum of the additions in latent N-size parentheses. As for simplicity, we just focus on it and omit other symbols. We called it the determined part of the latent N-size.

The determined part of the latent N-size of $1$-centered gradient penalty is
\begin{gather}\label{dpgp1}
\begin{aligned}
&\frac{\left\|G_{\theta_t}\left(\boldsymbol z_1\right)-G_{\theta_t}(\boldsymbol z)\right\|}{\left\|\boldsymbol z_1-\boldsymbol z\right\|}+ \frac{\left\|G_{\theta_{t+1}}\left(\boldsymbol z_1\right)-G_{\theta_{t+1}}(\boldsymbol z)\right\|}{\left\|\boldsymbol z_1-\boldsymbol z\right\|}\\
&\leq \frac{2\left\|G_{\theta_t}\left(\boldsymbol z_1\right)-G_{\theta_t}(\boldsymbol z)\right\|}{\left\|\boldsymbol z_1-\boldsymbol z\right\|}\\
&+\frac{\left\|\sum\limits_{m=1}^{N}\eta_m \delta(\boldsymbol x)(\nabla_{\boldsymbol x} D_{m}(G_{\theta_t}(\boldsymbol z_1))+ R\right\|+\left\|\sum\limits_{m=1}^{N}\eta_m \delta(\boldsymbol x)(\nabla_{\boldsymbol x} D_{m}(G_{\theta_t}(\boldsymbol z))+ R\right\|}{\left\|\boldsymbol z_1-\boldsymbol z\right\|}\\
&=\frac{2\left\|G_{\theta_t}\left(\boldsymbol z_1\right)-G_{\theta_t}(\boldsymbol z)\right\|}{\left\|\boldsymbol z_1-\boldsymbol z\right\|}\\
&+\frac{\left\|\sum\limits_{m=1}^{N}\eta_m \delta(\boldsymbol x)\left(\nabla_{\boldsymbol x} D_{m}(G_{\theta_t}(\boldsymbol z_1))+ \left(\left\|\nabla_{\boldsymbol x} D_{m}(G_{\theta_t}(\boldsymbol z_1))\right\|-g_{0}\right)^2\right)\right\|}{\left\|\boldsymbol z_1-\boldsymbol z\right\|}\\
&+\frac{\left\|\sum\limits_{m=1}^{N}\eta_m \delta(\boldsymbol x)\left(\nabla_{\boldsymbol x} D_{m}(G_{\theta_t}(\boldsymbol z))+ \left(\left\|\nabla_{\boldsymbol x} D_{m}(G_{\theta_t}(\boldsymbol z))\right\|-g_{0}\right)^2\right)\right\|}{\left\|\boldsymbol z_1-\boldsymbol z\right\|}.\\
\end{aligned}
\end{gather}
The determined part of the latent N-size of $0$-centered gradient penalty is
\begin{gather}
\begin{aligned}\label{dpgp0}
&\frac{\left\|G_{\theta_t}\left(\boldsymbol z_1\right)-G_{\theta_t}(\boldsymbol z)\right\|}{\left\|\boldsymbol z_1-\boldsymbol z\right\|}+ \frac{\left\|G_{\theta_{t+1}}\left(\boldsymbol z_1\right)-G_{\theta_{t+1}}(\boldsymbol z)\right\|}{\left\|\boldsymbol z_1-\boldsymbol z\right\|}\\
&\leq \frac{2\left\|G_{\theta_t}\left(\boldsymbol z_1\right)-G_{\theta_t}(\boldsymbol z)\right\|}{\left\|\boldsymbol z_1-\boldsymbol z\right\|}\\
&+\frac{\left\|\sum\limits_{m=1}^{N}\eta_m \delta(\boldsymbol x) (\nabla_{\boldsymbol x} D_{m}(G_{\theta_t}(\boldsymbol z_1))+  R\right\|+\left\|\sum\limits_{m=1}^{N}\eta_m\delta(\boldsymbol x)(\nabla_{\boldsymbol x} D_{m}(G_{\theta_t}(\boldsymbol z))+ R\right\|}{\left\|\boldsymbol z_1-\boldsymbol z\right\|}\\
&=\frac{2\left\|G_{\theta_t}\left(\boldsymbol z_1\right)-G_{\theta_t}(\boldsymbol z)\right\|}{\left\|\boldsymbol z_1-\boldsymbol z\right\|}\\
&+\frac{\left\|\sum\limits_{m=1}^{N}\eta_m \delta(\boldsymbol x)\left(\nabla_{\boldsymbol x} D_{m}(G_{\theta_t}(\boldsymbol z_1))+ \left(\left\|\nabla_{\boldsymbol x} D_{m}(G_{\theta_t}(\boldsymbol z_1))\right\|\right)^2\right)\right\|}{\left\|\boldsymbol z_1-\boldsymbol z\right\|}\\
&+\frac{\left\|\sum\limits_{m=1}^{N}\eta_m \delta(\boldsymbol x)\left(\nabla_{\boldsymbol x} D_{m}(G_{\theta_t}(\boldsymbol z))+ \left(\left\|\nabla_{\boldsymbol x} D_{m}(G_{\theta_t}(\boldsymbol z))\right\|\right)^2\right)\right\|}{\left\|\boldsymbol z_1-\boldsymbol z\right\|}.\\    
\end{aligned}  
\end{gather}
With prior knowledge from CFG that $\nabla_x D_{m}(x)\leq 0$, 
The determined part of the latent N-size of $\boldsymbol\varepsilon$-centered gradient penalty is
\begin{gather}
\begin{aligned}\label{dpgp2}
&\frac{\left\|G_{\theta_t}\left(\boldsymbol z_1\right)-G_{\theta_t}(\boldsymbol z)\right\|}{\left\|\boldsymbol z_1-\boldsymbol z\right\|}+ \frac{\left\|G_{\theta_{t+1}}\left(\boldsymbol z_1\right)-G_{\theta_{t+1}}(\boldsymbol z)\right\|}{\left\|\boldsymbol z_1-\boldsymbol z\right\|}\\
&\leq \frac{2\left\|G_{\theta_t}\left(\boldsymbol z_1\right)-G_{\theta_t}(\boldsymbol z)\right\|}{\left\|\boldsymbol z_1-\boldsymbol z\right\|}\\
&+\frac{\left\|\sum\limits_{m=1}^{N}\eta_m \delta(\boldsymbol x)(\nabla_{\boldsymbol x} D_{m}(G_{\theta_t}(\boldsymbol z_1))+ R\right\|+\left\|\sum\limits_{m=1}^{N}\eta_m \delta(\boldsymbol x)(\nabla_{\boldsymbol x} D_{m}(G_{\theta_t}(\boldsymbol z))+  R\right\|}{\left\|\boldsymbol z_1-\boldsymbol z\right\|}\\
&=\frac{2\left\|G_{\theta_t}\left(\boldsymbol z_1\right)-G_{\theta_t}(\boldsymbol z)\right\|}{\left\|\boldsymbol z_1-\boldsymbol z\right\|}\\
&+\frac{\left\|\sum\limits_{m=1}^{N}\eta_m \delta(\boldsymbol x)\left(\nabla_{\boldsymbol x} D_{m}(G_{\theta_t}(\boldsymbol z_1))+ \left(\left\|\nabla_{\boldsymbol x} D_{m}(G_{\theta_t}(\boldsymbol z_1))-\boldsymbol\varepsilon\right\|\right)^2\right)\right\|  }{\left\|\boldsymbol z_1-\boldsymbol z\right\|}\\
&+\frac{\left\|\sum\limits_{m=1}^{N}\eta_m \delta(\boldsymbol x)\left(\nabla_{\boldsymbol x} D_{m}(G_{\theta_t}(\boldsymbol z))+ \left(\left\|\nabla_{\boldsymbol x} D_{m}(G_{\theta_t}(\boldsymbol z))-\boldsymbol\varepsilon\right\|\right)^2\right)\right\|  }{\left\|\boldsymbol z_1-\boldsymbol z\right\|}\\
&\leq\frac{2\left\|G_{\theta_t}\left(\boldsymbol z_1\right)-G_{\theta_t}(\boldsymbol z)\right\|}{\left\|\boldsymbol z_1-\boldsymbol z\right\|}\\
&+\frac{\left\|\sum\limits_{m=1}^{N}\eta_m \delta(\boldsymbol x)\left(\nabla_{\boldsymbol x} D_{m}(G_{\theta_t}(\boldsymbol z_1))+ \left(\left\|\nabla_{\boldsymbol x} D_{m}(G_{\theta_t}(\boldsymbol z_1))\right\|+\|\boldsymbol\varepsilon\|\right)^2\right)\right\|}{\left\|\boldsymbol z_1-\boldsymbol z\right\|}\\
&+\frac{\left\|\sum\limits_{m=1}^{N}\eta_m \delta(\boldsymbol x)\left(\nabla_{\boldsymbol x} D_{m}(G_{\theta_t}(\boldsymbol z))+ \left(\left\|\nabla_{\boldsymbol x} D_{m}(G_{\theta_t}(\boldsymbol z))\right\|+\|\boldsymbol\varepsilon\|\right)^2\right)\right\|}{\left\|\boldsymbol z_1-\boldsymbol z\right\|}.\\
\end{aligned}   
\end{gather}

For the square items in Eq.~(\ref{dpgp1}), Eq.~(\ref{dpgp0}), Eq.~(\ref{dpgp2}), we show that the 
\begin{gather}
\begin{aligned}
&\|\left(\left\|\nabla_{\boldsymbol x} D_{m}(G_{\theta_t}(\boldsymbol z))\right\|+\|\boldsymbol\varepsilon\|\right) \|> \|\left(\left\|\nabla_{\boldsymbol x} D_{m}(G_{\theta_t}(\boldsymbol z))\right\|\right) \|
>
\|\left(\left\|\nabla_{\boldsymbol x} D_{m}(G_{\theta_t}(\boldsymbol z))\right\|-g_0\right)\|.
\end{aligned}   
\end{gather}
Based on this conclusion, we can observe that the relationship among the latent N-size with different gradient penalties is as follows:  $r_{R_1} > r_{R_0} > r_{R_{\boldsymbol\varepsilon}}$.
\end{proof}

\begingroup
\def\thetheorem{\ref{theorem main}}
\begin{theorem}
% \noindent \textbf{Theorem 3.8.}
% \uppercase\expandafter{\romannumeral3.8.}}
Suppose $\boldsymbol z_1$ is attracted to the mode $\mathcal{M}_{i}$ by $\hat{\epsilon}$, then there exists a neighborhood $\mathcal{N}_r\left(\boldsymbol{z}_1\right)$ of $\boldsymbol{z}_1$ such that $\boldsymbol{z}_2$ is distracted to $\mathcal{M}_{i}$ by $(\hat\epsilon/ 2-2\alpha)$, for all $\boldsymbol{z}_2 \in \mathcal{N}_r\left(\boldsymbol{z}_1\right)$. The size of $\mathcal{N}_r\left(\boldsymbol{z}_1\right)$ can be arbitrarily large but is bounded by an open ball of radius $r$ where be controlled by Gradient Penalty terms of the discriminator. The relationship between the radius's size of three Gradient Penalty is $r_{R_1}>r_{R_0}>r_{R_{\boldsymbol\varepsilon}}$. 
\end{theorem}
\addtocounter{theorem}{-1}
\endgroup
\begin{proof}
By combining the Definition~\ref{definition-radius}, Proposition~\ref{relation n d}, Lemma~\ref{lemma compare} and our $\boldsymbol\varepsilon$-centered gradient penalty, we can deviate this conclusion. 
\end{proof}
This theorem encompasses three key implications. Firstly, when the latent vector is attracted to a mode within the image space, the corresponding latent N-size should not be overly large, which can be attributed to two distinct reasons. One, vectors within the neighborhood are more likely to be attracted to the same mode. Two, vectors within this neighborhood face challenges in being attracted to other modes within the image space, with the level of difficulty determined by an upper bound expressed as $(\hat\epsilon/ 2-2\alpha)$. This bound is constructed based on the distances between different modes and within the same mode within the image space.

Secondly, the discriminator gradient penalty in the CFG formulation can regulate the latent N-size. This introduces a trade-off between diversity and training stability. For instance, the $0$-centered gradient penalty ensures stable and convergent training near the Nash equilibrium, but it leads to a minimal discriminator norm, resulting in a larger latent N-size and reduced diversity.

Thirdly, by augmenting the norm of the discriminator, our $\boldsymbol\varepsilon$-centered gradient penalty achieves the smallest latent N-size, consequently leading to the highest level of diversity.

\section{Experiments}

In this section, we describe additional experiments and give a more visual representation.
In addition to the databases mentioned in the paper, we also conducted training on the CeleBA and LSUN Bedroom datasets with resolutions of 128x128 and 256x256. We will provide visual inspections and results for these additional databases.

\subsection{Hyper-parameter \texorpdfstring{$\varepsilon'$}{''}}
In the practice stage, the $\boldsymbol{\varepsilon}$ is a $\varepsilon'$-value constant vector which
has the same dimension of $\nabla_{\boldsymbol x} D(\boldsymbol x)$. The dimension of $\nabla_{\boldsymbol x} D(\boldsymbol x)$ is [C, H, W], the flatten of $\nabla_{\boldsymbol x} D(\boldsymbol x)$ dimension equals $H*W*C$. Our image data has the $N$ pixels height, $N$ pixels weight, and $C$ channels. The  $H*W*C$ will be written as $C N^{2}$. $\|\|$ is the symbol of norm 2.  
So the value of $\|\boldsymbol{\varepsilon}\|$ can be written as
\begin{gather}
\begin{aligned}
\|\boldsymbol{\varepsilon}\| = \sqrt{\sum_{i=1}^{C N^{2}}\varepsilon_{i}^{2}} = \sqrt{C N^{2}\varepsilon^{2}}=\varepsilon',
\end{aligned}
\end{gather}
where $\varepsilon_{i}^{2}=\varepsilon^{2}$, $\varepsilon^{2} =\frac{\varepsilon'^2}{C N^{2}}$. The $\varepsilon'$ is a hyper-parameter that controls the tight bound of $\left\|\nabla_{\boldsymbol x} D(\boldsymbol x)\right\|$. We will set $\varepsilon'= 0.1, 0.3, 1, 5$ in the ablation study to show the effeteness. It is easy to understand the $\varepsilon'= 0.1, 0.3, 1$ because this is a small enhancement of the discriminator norm and thus to a smaller latent N-size. 
If we set the  $\varepsilon'= 5$, the varying range of $\|\nabla_{\boldsymbol x} D(\boldsymbol x)\|$ is too large which leads to a too-small latent N-size, thus leads to a non-convergent result for the neural network. The loss function of the training process will not converge and the synthetic samples will transform to noise. We present the ablation result in Table.~\ref{table:lsun b}
and Table.~\ref{table:mnist}

\begin{table*}[htb!]\small
\caption{\centering Different $\varepsilon'$ settings of our Li-CFG trained in MNIST. We use the FID and IS scores to compare the generated effect. The other two penalties do not have the parameter $\varepsilon'$ so that all the cells fill the same value. Untrained means the loss function does not converge.\label{table:mnist}}
\begin{center}\setlength{\tabcolsep}{0.75mm}{
\begin{tabular}{lcccccccc}
\hline
{\textbf{}} & \textbf{} & \multicolumn{2}{c}{\textbf{FID} } & \multicolumn{1}{l}{\textbf{}}
& \multicolumn{4}{c}{\textbf{IS}}                                             \\
             \textbf{MNIST}               &    \multicolumn{1}{l}{\textbf{$\varepsilon'=0.1$} }       & \multicolumn{1}{l}{\textbf{$\varepsilon'=0.3$}}                      & \multicolumn{1}{l}{\textbf{$\varepsilon'=1$}}                      & \multicolumn{1}{l}{\textbf{$\varepsilon'=5$}} & \multicolumn{1}{l}{\textbf{$\varepsilon'=0.1$}}& \multicolumn{1}{l}{\textbf{$\varepsilon'=0.3$}}                      & \multicolumn{1}{l}{\textbf{$\varepsilon'=1$}}                      & \multicolumn{1}{l}{\textbf{$\varepsilon'=5$}}\\
\hline
ours($\boldsymbol\varepsilon$-centered)            &   \textbf{2.99}          &    \textbf{2.88}      &      \textbf{2.85}                                        &   untrained                               &                 2.28               &   \textbf{2.32}
 &                 2.29                 &     untrained \\
% \textbf{SVHN}              & \textbf{0.913}            & \textbf{0.87}&\textbf{1.39}  &    &   \textbf{1.41} &                                   \\
\textbf{$0$-centered}           & 3.54             & 3.54   &   3.54                                       &    3.54                              & \textbf{2.31}                                    &                2.31
 & \textbf{2.31}                                    &                \textbf{2.31}\\

\textbf{$1$-centered}           & 3.64           &3.64   &   3.64                                       &    3.64                             & 2.3                                   &                2.3 & 2.3                                    &               2.3\\
\hline                          
\end{tabular}}
\end{center}
\end{table*}

\begin{table*}[htb!]\small
\caption{\centering Different $\varepsilon'$ settings of our Li-CFG trained in LSUN bedroom. We use the FID and IS scores to compare the generated effect. The other two penalties do not have the parameter $\varepsilon'$ so that all the cells fill the same value. Untrained means the loss function does not converge.\label{table:lsun b}}
\begin{center}
\setlength{\tabcolsep}{0.75mm}{
\begin{tabular}{lcccccccc}
\hline
{\textbf{}} & \textbf{} & \multicolumn{2}{c}{\textbf{FID} } & \multicolumn{1}{l}{\textbf{}}
& \multicolumn{4}{c}{\textbf{IS}}                                             \\
             \textbf{LSUN Bedroom}               &    \multicolumn{1}{l}{\textbf{$\varepsilon'=0.1$} }       & \multicolumn{1}{l}{\textbf{$\varepsilon'=0.3$}}                      & \multicolumn{1}{l}{\textbf{$\varepsilon'=1$}}                      & \multicolumn{1}{l}{\textbf{$\varepsilon'=5$}} & \multicolumn{1}{l}{\textbf{$\varepsilon'=0.1$}}& \multicolumn{1}{l}{\textbf{$\varepsilon'=0.3$}}                      & \multicolumn{1}{l}{\textbf{$\varepsilon'=1$}}                      & \multicolumn{1}{l}{\textbf{$\varepsilon'=5$}}\\
\hline
ours($\boldsymbol\varepsilon$-centered)           &   \textbf{9.94}          &   \textbf{8.78}     &     \textbf{ 9.73}                                        &   untrained                              &                 2.97                &     2.94
 &                  2.97                  &     untrained \\
% \textbf{SVHN}              & \textbf{0.913}            & \textbf{0.87}&\textbf{1.39}  &    &   \textbf{1.41} &                                   \\
$0$-centered       &10.54            &10.54  &    10.54                                        &    10.54                             & 3.067                                &                3.067
 & 3.067                                  &                3.067\\

$1$-centered         & 11.5           & 11.5   &    11.5                                        &    11.5                              & \textbf{3.154  }                                &              \textbf{3.154  }& \textbf{3.154  }                                 &                \textbf{3.154  }\\
\hline              
\end{tabular}}
\end{center}
\end{table*}

\begin{table*}[htb!]\small
\caption{\centering Different $\gamma$ and the same $\varepsilon'=0.3$ settings of our Li-CFG trained in LSUN T. We use FID and IS score to compare the generated effect. The other two penalties do not have the parameter $\varepsilon$ so that all the cells fill the same value. Untrained means the loss function does not converge.\label{table:lsun t}}
\begin{center}
\setlength{\tabcolsep}{0.75mm}{
\begin{tabular}{lcccccccc}
\hline
{\textbf{}} & \textbf{} & \multicolumn{2}{c}{\textbf{FID} } & \multicolumn{1}{l}{\textbf{}}
& \multicolumn{4}{c}{\textbf{IS}}                                             \\
             \textbf{LSUN T ($\delta(\boldsymbol x)$=1)}               &    \multicolumn{1}{l}{\textbf{$\gamma'=0.1$} }       & \multicolumn{1}{l}{\textbf{$\gamma'=1$}}                      & \multicolumn{1}{l}{\textbf{$\gamma'=10$}}                      & \multicolumn{1}{l}{\textbf{}} & \multicolumn{1}{l}{\textbf{$\gamma'=0.1$}}& \multicolumn{1}{l}{\textbf{$\gamma'=1$}}                      & \multicolumn{1}{l}{\textbf{$\gamma'=10$}}                      & \multicolumn{1}{l}{\textbf{}}\\
\hline
ours($\boldsymbol\varepsilon$-centered)             &  11.41         &  19.47      &     21.59                                        &                                &                  4.6                 &    4.5 
 &                 4.42                &      \\
% \textbf{SVHN}              & \textbf{0.913}            & \textbf{0.87}&\textbf{1.39}  &    &   \textbf{1.41} &                                   \\
$0$-centered          & 12.33             & 19.94    &    21.71                                   &                               &4.57                                   &               4.6
 &4.53                                   &                \\

$1$-centered           & 12.81           & 22.17  &    22.71                                      &                               &4.47                                   &                4.46 & 4.5                                &                \\
CFG method          & 13.54            & 13.54    &    13.54                                         &                               &4.38                                  &                4.38 & 4.38                                  &               \\
\hline 
 \textbf{LSUN T ($\delta(\boldsymbol x)$=5})\\    
ours($\boldsymbol\varepsilon$-centered)             &  19.18        &   21.58      &     21.69                                        &                             &                 4.48                 &    4.49
 &                 4.34                  &    \\
% \textbf{SVHN}              & \textbf{0.913}            & \textbf{0.87}&\textbf{1.39}  &    &   \textbf{1.41} &                                   \\
$0$-centered          & 20.29             &24.77   &   23.02                                        &                              &4.47                                   &               4.43
 & 4.43                                   &                \\

$1$-centered           & 24.09             & 29.39   &    untrained                                       &                                &4.37                                   &               4.41 & untrained                                  &                \\
CFG method           &22.3             & 22.3  &    22.3                                        &                                &4.34                                   &                4.34 &                4.34                    &              \\
\hline 
 \textbf{LSUN T ($\delta(\boldsymbol x)$=10})\\   
ours($\boldsymbol\varepsilon$-centered)             &  21.3      &   23     &     20.7                                      &                                &                4.43                  &    4.42
 &                 4.53                  &   \\
% \textbf{SVHN}              & \textbf{0.913}            & \textbf{0.87}&\textbf{1.39}  &    &   \textbf{1.41} &                                   \\
$0$-centered          & 23.68            & 19.74   &    20.7                                        &                                 &4.39                                    &                4.46
 & 4.38                                    &                \\

$1$-centered           &23.49            &  untrained   &    untrained                                       &                                &4.3                                  &               untrained & untrained                   &                \\
CFG method           & 25.53             & 25.53    &    25.53                                        &                                &4.2                                  &                4.2 & 4.2                                  &               \\
\hline       
\end{tabular}}
\end{center}
\end{table*}

If our experiment setting does not satisfy these conditions, the synthetic samples make noise or collapse. We show the visual inspection of different settings $\varepsilon'$ for MNIST and LSUN Bedroom in Fig.~\ref{Fig.mnist-problem} and Fig.~\ref{Fig.bd-problem}.

\begin{figure}
\centering
\includegraphics[width=0.6\textwidth]{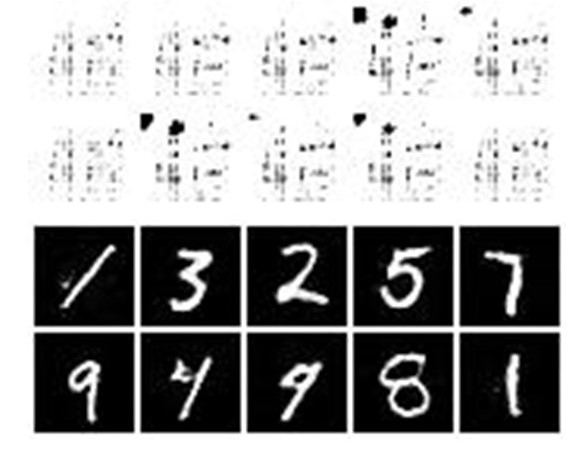} 
\caption{The top two rows are $\varepsilon' = 5$. The bottom two rows are $\varepsilon' = 0.3$. The top two rows collapse, but the bottom two rows are normal. \label{Fig.mnist-problem}}
\vspace{-0.1in}
\end{figure}

\begin{figure}[t]
\centering
\includegraphics[width=0.6\textwidth]{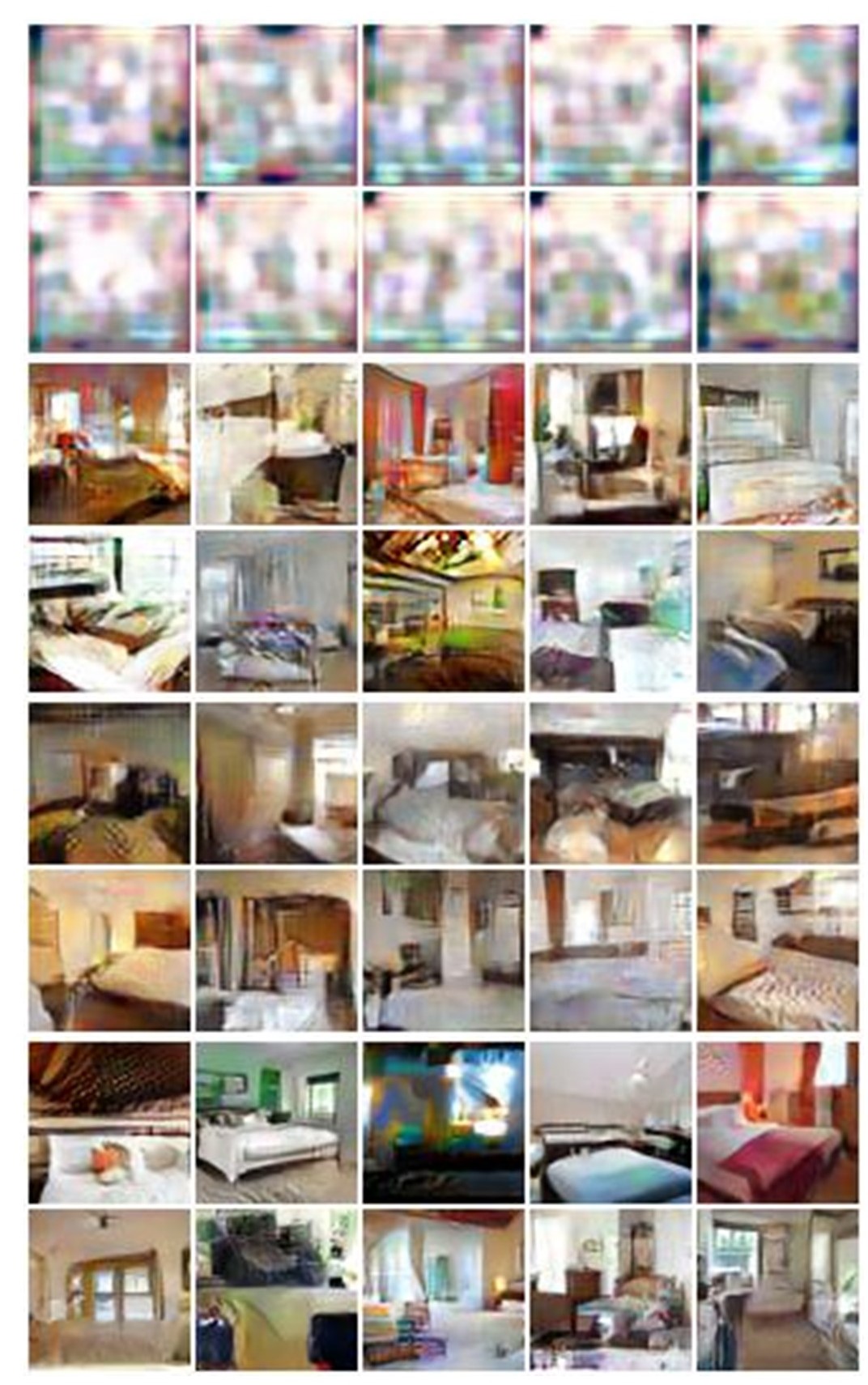} 
\caption{The top two rows are $\varepsilon' = 5$, the second two row are $\varepsilon' = 1$ and the third two rows are $\varepsilon' =0.1$. The bottom two rows are $\varepsilon' = 0.3$. This figure shows that if the hyper-parameter $\varepsilon'$ is too large, it will lead to a 
too-small latent N-size and thus lead to the non-convergent result. The hyper-parameter $\varepsilon'$ should be a reasonable value, neither too large nor too small will lead to the best diversity of synthetic samples. \label{Fig.bd-problem}} 
\vspace{-0.1in}
\end{figure}

\begin{table*}[htb!]\small
\caption{\centering Different $\gamma$ settings of our Li-CFG trained in LSUN T+B. We use FID and IS scores to compare the generated effect. The other two penalties do not have the parameter $\varepsilon$ so that all the cells fill the same value.\label{table:lsun b+t}}
\begin{center}
\begin{tabular}{lcccccccc}
\hline
{\textbf{}} & \textbf{} & \multicolumn{2}{c}{FID } & \multicolumn{1}{l}{\textbf{}}
& \multicolumn{4}{c}{IS}                                              \\
             \textbf{LSUN T+B}               &    \multicolumn{1}{l}{\textbf{$\gamma=0.1$}}                      & \multicolumn{1}{l}{\textbf{$\gamma=1$}}                      & \multicolumn{1}{l}{\textbf{$\gamma=10$}} &
             \multicolumn{1}{l}{\textbf{}}                      &\multicolumn{1}{l}{\textbf{$\gamma=0.1$}}                      & \multicolumn{1}{l}{\textbf{$\gamma=1$}}                      & \multicolumn{1}{l}{\textbf{$\gamma=10$}}
             &\multicolumn{1}{l}{\textbf{}} \\
\hline
\textbf{ours($\boldsymbol\varepsilon$-centered)}                      &    \textbf{15.72}      &      \textbf{16.85}                                         &  \textbf{16.67}                              &   &                             \textbf{5.07}       &     5.06
 &  \textbf{5.08} &                  \textbf{} \\
% \textbf{SVHN}              & \textbf{0.913}            & \textbf{0.87}&\textbf{1.39}  &    &   \textbf{1.41} &                                   \\
\textbf{$0$-centered}          & 16.01   &    17.4                                        &   19.79  &   \textbf{}                           & 5.05                                &                5.08
 &5.05                        &   \textbf{}    \\

\textbf{$1$-centered}           & 16.32    &    18.73                                      &  34.28  &   \textbf{}                            &5.04                                &                 \textbf{ 5.18} & 4.71    &   \textbf{}                \\
\hline      
\end{tabular}
\end{center}
\end{table*}
\subsection{Effect of \texorpdfstring{$\gamma$}{-} with different penalty}
In most cases, a gradient penalty with a center within a small interval around 0 tends to yield better results. However, on some datasets, a gradient penalty centered around 1 might perform better. Our method provides a controllable parameter that allows the center of the gradient penalty to vary within the range specified by the parameter, and we experimentally show that our method yields better results. 
For the parameters $\delta(\boldsymbol x)$ where the CFG method performs well, using a value of 0.1 for the parameter $\gamma$ consistently yields improved results. Conversely, for those parameter settings of $\delta(\boldsymbol x)$ where the CFG method shows poor or inadequate training performance, using a value of 10 for the parameter $\gamma$ often leads to better overall performance. We compare the results of different $\gamma$ values for CFG and Li-CFG in Table.~\ref{table:lsun t} and Table.~\ref{table:lsun b+t}.

\subsection{Effect of Gradient Penalty  with different \texorpdfstring{$\delta(\boldsymbol x)$}{-}}
The results of the CFG method can exhibit significant variations in response to minor changes in parameter values. This phenomenon is illustrated in Fig.~\ref{Fig.T_aba_1}, where we observe that different choices of the $\delta(\boldsymbol x)$ parameter lead to diverse outcomes. However, by introducing the gradient penalty, the CFG method consistently demonstrates increased stability across various $\delta(\boldsymbol x)$ values. Notably, in cases where the CFG method struggles to converge, the addition of the gradient penalty leads to improved results.

\begin{figure*}
\centering
\includegraphics[width=1\textwidth]{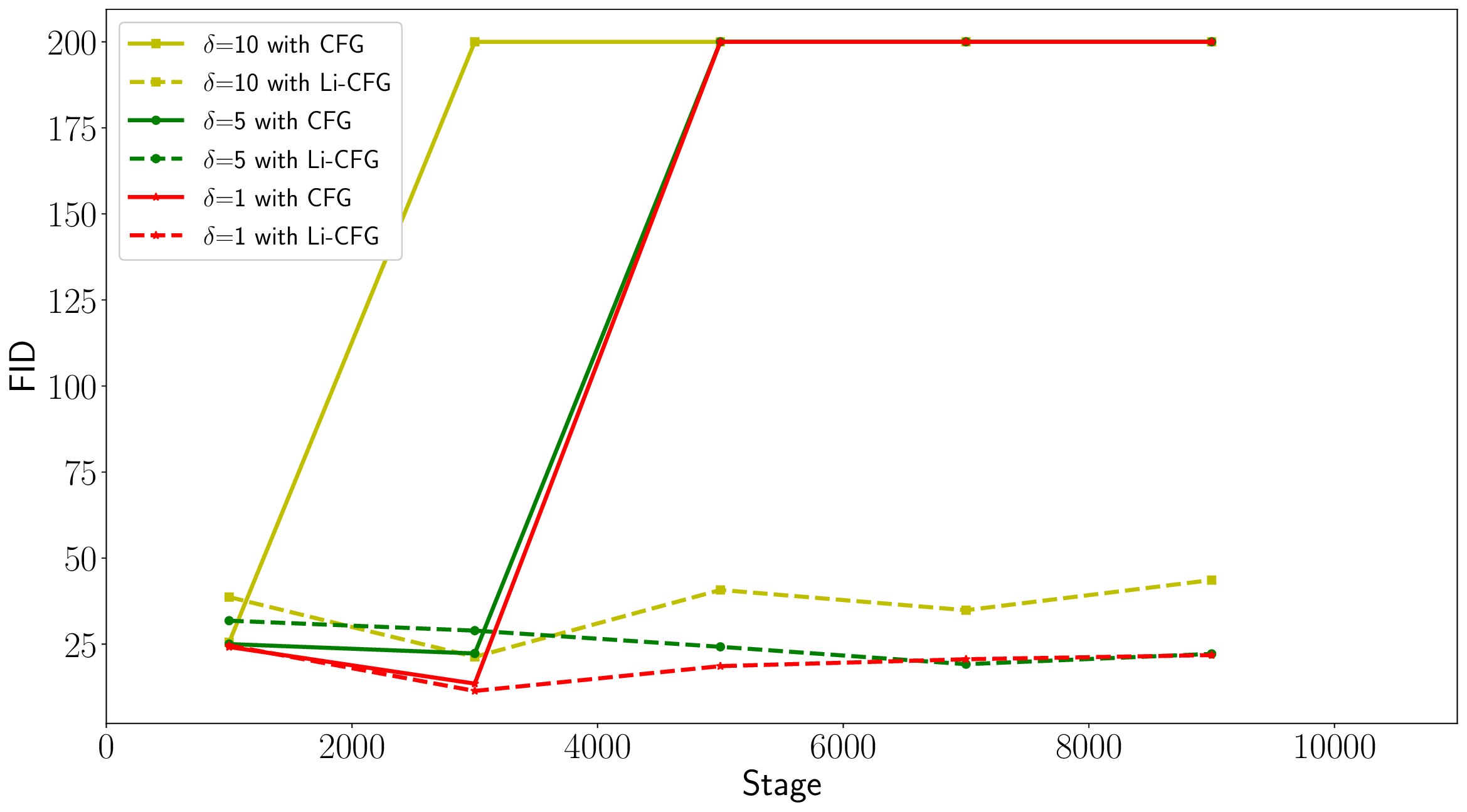} 
\caption{This figure contains FID values with three different $\delta(\boldsymbol x)$ settings in LSUN tower datasets. Even though the FID scores of the CFG method get 200 which means the training of the CFG method does not converge, our Li-CFG always achieves desirable and competitive FID scores.  \label{Fig.T_aba_1}} 
\end{figure*}

\subsection{Visual inspection of synthesized images}\label{appedix.figures}
In this section, we introduce the visual inspection figures of the CFG method, Li-CFG, and other GAN models. The Fig.~\ref{Fig.main0},~\ref{Fig.main1},~\ref{Fig.main2},~\ref{Fig.main3},~\ref{Fig.main5},~\ref{Fig.bed256},~\ref{Fig.main4} are separated into two columns of real images and generated images respectively. The real images are represented in the left column while the generated images are displayed in the right column which contains different GP regularization with Li-CFG. The settings of generated images in the CFG method and Li-CFG are almost the same. We also display the generated images with CeleBA, LSUN Bedroom and ImageNet with resolution of 128 and 256. {We present the results of synthetic datasets in the Fig.~\ref{Fig.gan-ring},~\ref{Fig.wgan-ring},~\ref{Fig.hingegan-ring},~\ref{Fig.lsgan-ring},~\ref{Fig.gan-square},~\ref{Fig.wgan-square},~\ref{Fig.lsgan-square},~\ref{Fig.hingegan-square}. Moreover, we showcase the results of BigGAN and DDGAN from real-world datasets in the Fig.~\ref{Fig.biggan-i64},~\ref{Fig.ddgan-lsun},~\ref{Fig.ddgan-lsun1}.}

\begin{figure*}
\centering 
\includegraphics[width=1\textwidth]{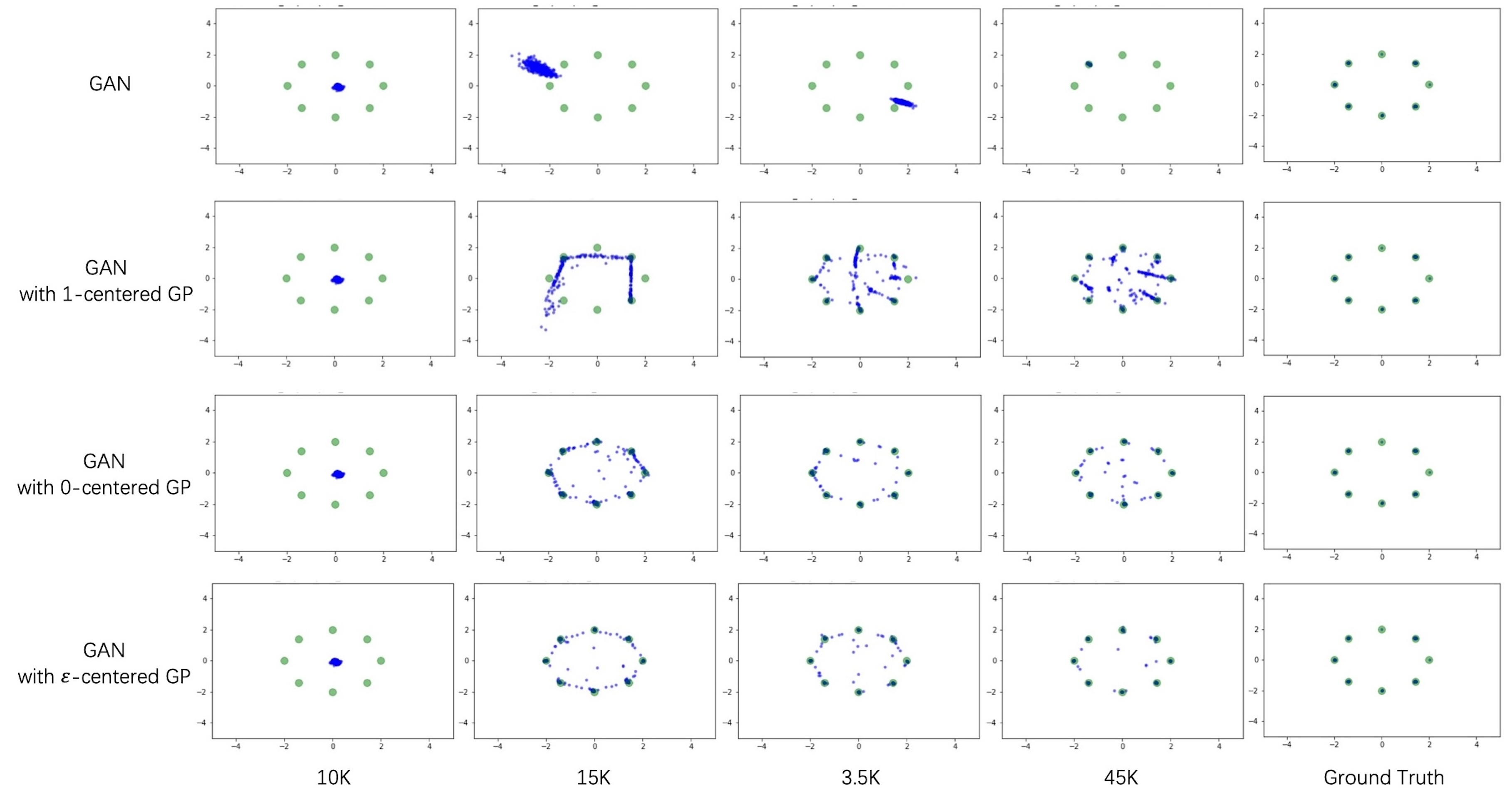} 
\caption{{Results for the original GAN with varying gradient penalties on the Ring dataset are displayed as follows: From top to bottom, the sequence includes the original GAN, the original GAN with $1$-centered gradient penalty, the original GAN with $0$-centered gradient penalty, and the original GAN with $\boldsymbol\varepsilon$-centered gradient penalty. Progressing from left to right, each column represents outcomes from different stages of training. The far-right column displays the ground truth data for comparison.}\label{Fig.gan-ring}}
\vspace{-0.1in}
\end{figure*}

\begin{figure*}
\centering 
\includegraphics[width=1\textwidth]{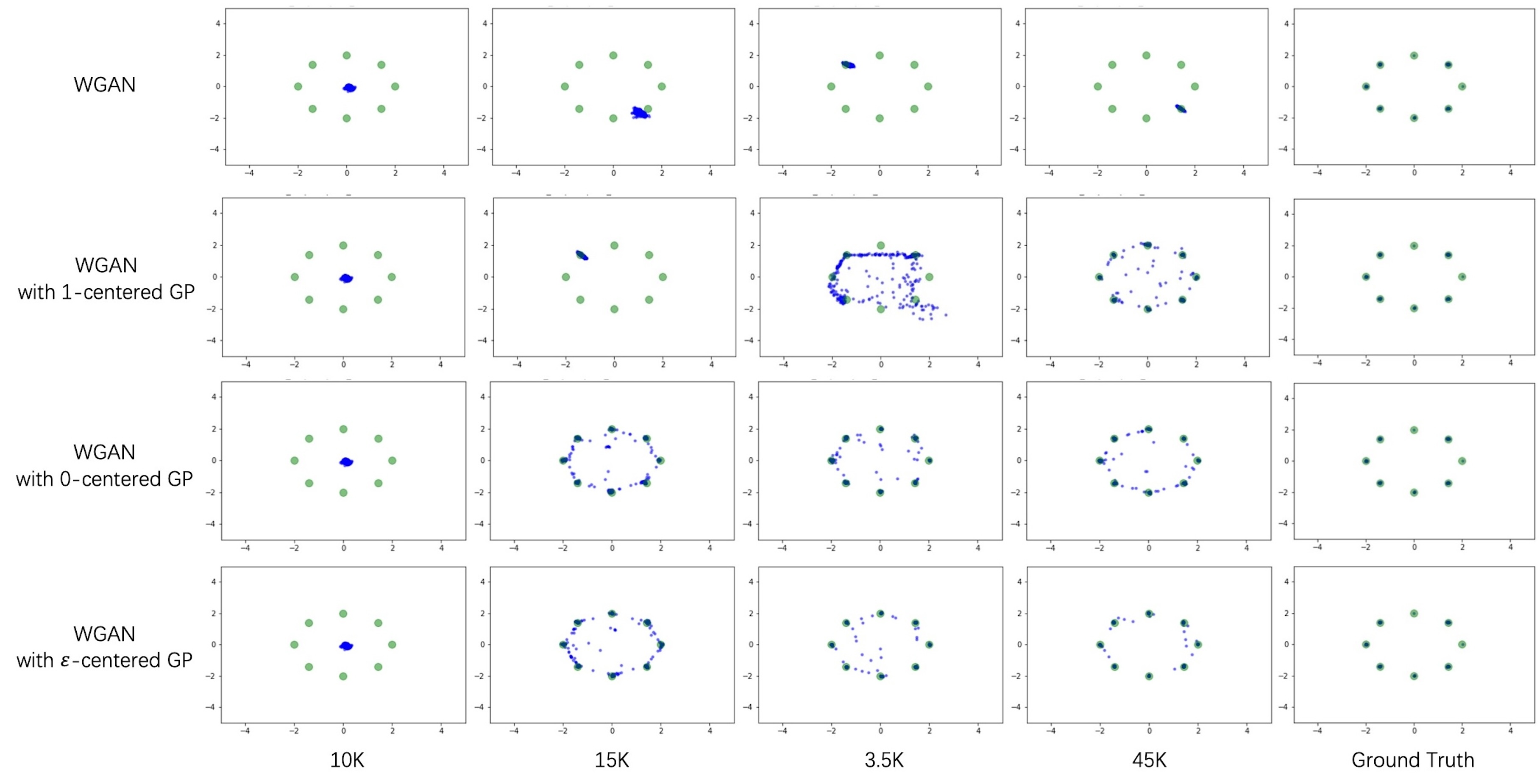} 
\caption{{Results for the WGAN with varying gradient penalties on the Ring dataset are displayed as follows: From top to bottom, the sequence includes the WGAN, the WGAN with $1$-centered gradient penalty, the WGAN with $0$-centered gradient penalty, and the WGAN with $\boldsymbol\varepsilon$-centered gradient penalty. Progressing from left to right, each column represents outcomes from different stages of training. The far-right column displays the ground truth data for comparison.}\label{Fig.wgan-ring}}
\vspace{-0.1in}
\end{figure*}

\begin{figure*}
\centering 
\includegraphics[width=1\textwidth]{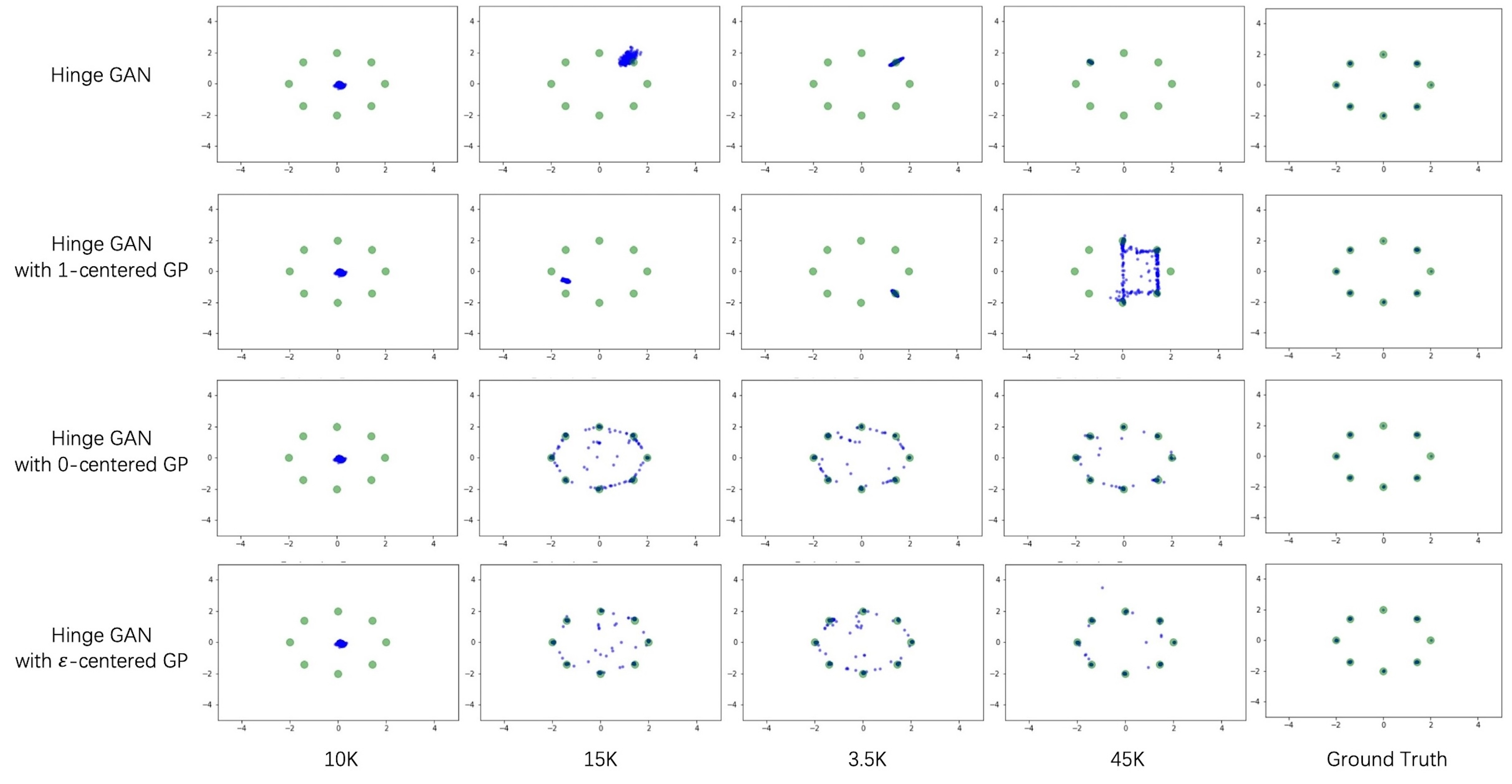} 
\caption{{Results for the Hinge GAN with varying gradient penalties on the Ring dataset are displayed as follows: From top to bottom, the sequence includes the Hinge GAN, the Hinge GAN with $1$-centered gradient penalty, the Hinge GAN with $0$-centered gradient penalty, and the Hinge GAN with $\boldsymbol\varepsilon$-centered gradient penalty. Progressing from left to right, each column represents outcomes from different stages of training. The far-right column displays the ground truth data for comparison.} \label{Fig.hingegan-ring}}
\vspace{-0.1in}
\end{figure*}

\begin{figure*}
\centering 
\includegraphics[width=1\textwidth]{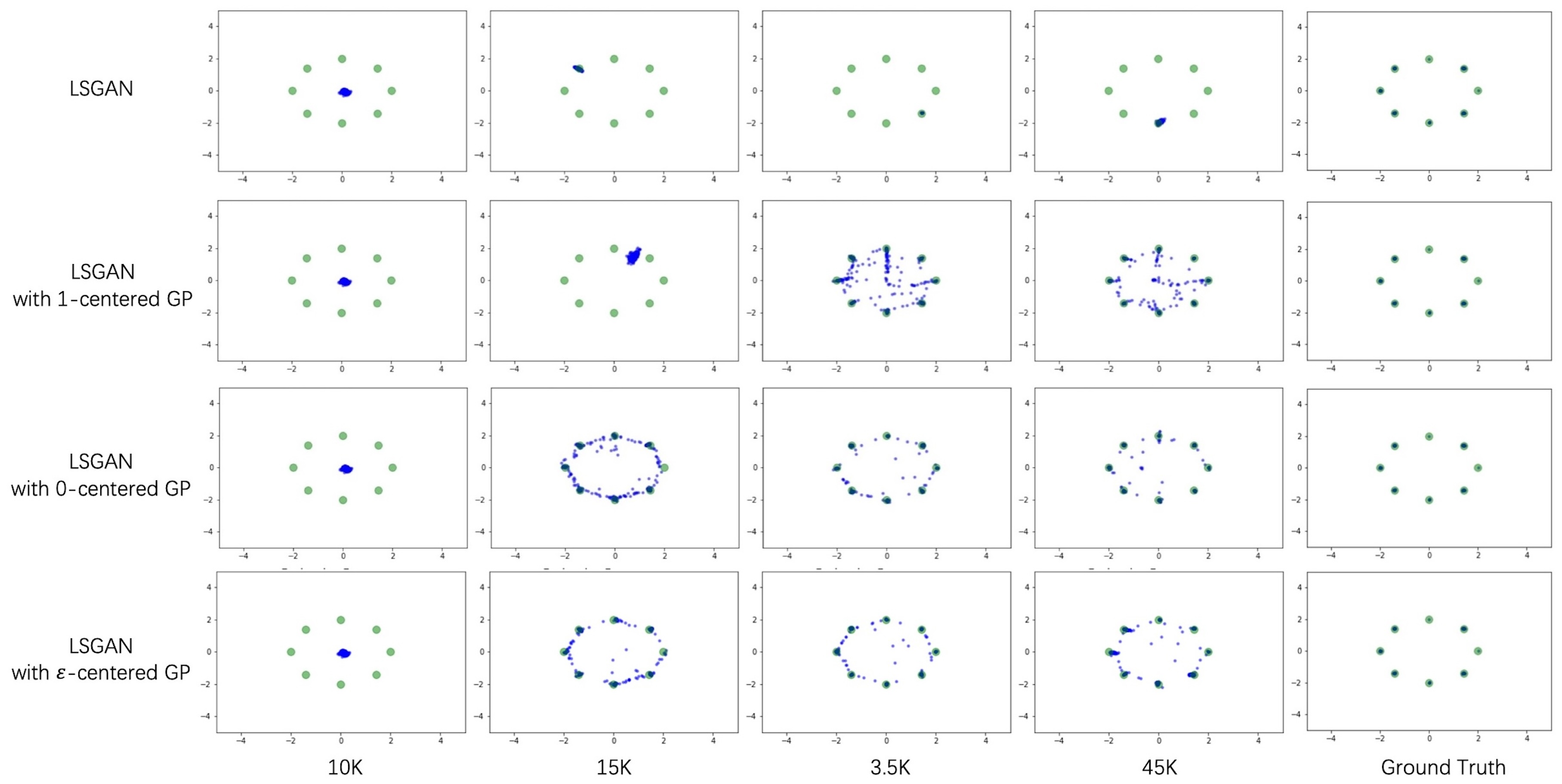} 
\caption{{Results for the LSGAN with varying gradient penalties on the Ring dataset are displayed as follows: From top to bottom, the sequence includes the LSGAN, the LSGAN with $1$-centered gradient penalty, the LSGAN with $0$-centered gradient penalty, and the LSGAN with $\boldsymbol\varepsilon$-centered gradient penalty. Progressing from left to right, each column represents outcomes from different stages of training. The far-right column displays the ground truth data for comparison.} \label{Fig.lsgan-ring}}
\vspace{-0.1in}
\end{figure*}

\begin{figure*}
\centering 
\includegraphics[width=1\textwidth]{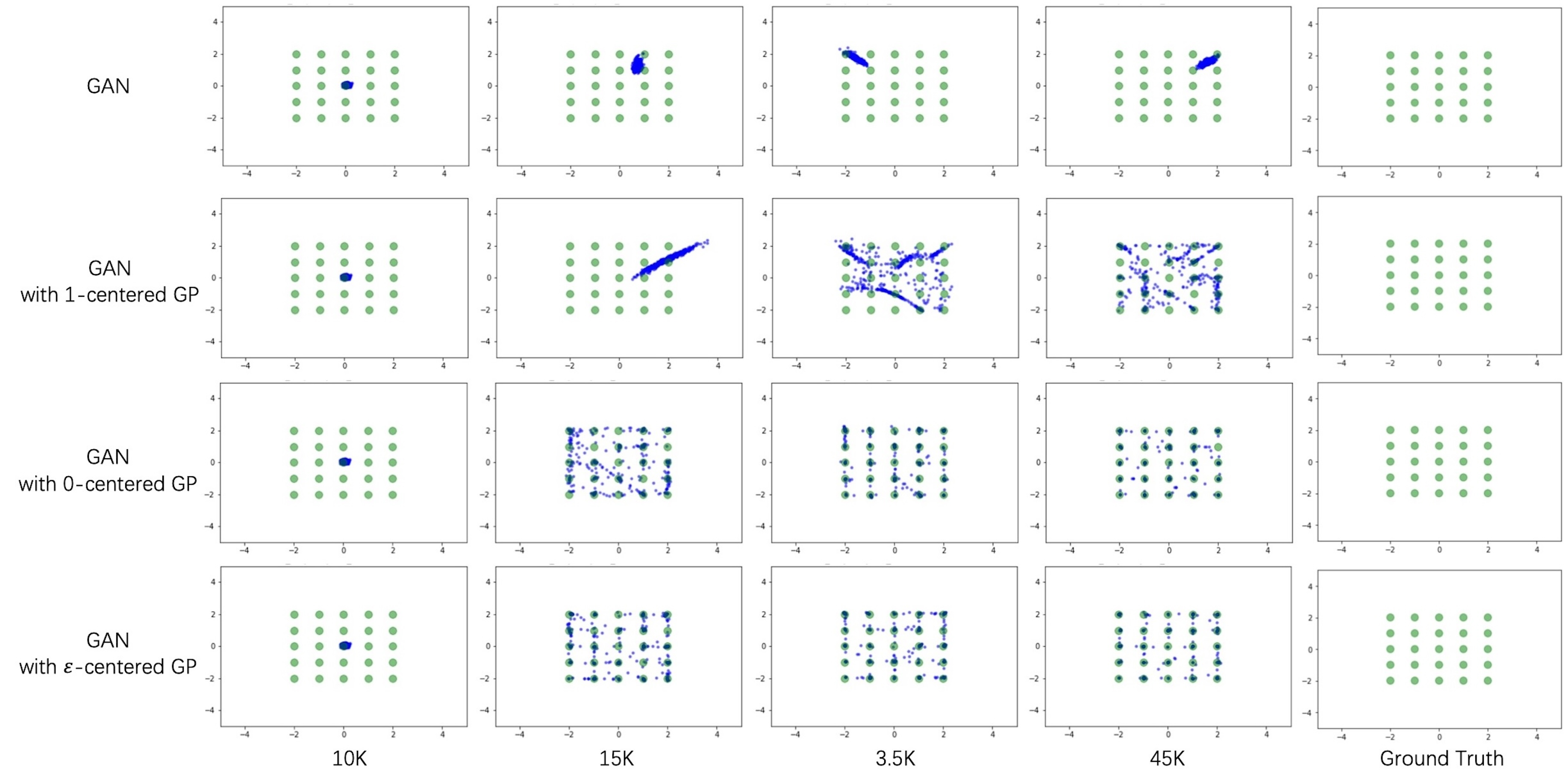} 
\caption{{Results for the origin GAN with varying gradient penalties on the Grid dataset are displayed as follows: From top to bottom, the sequence includes the origin GAN, the origin GAN with $1$-centered gradient penalty, the origin GAN with $0$-centered gradient penalty, and the origin GAN with $\boldsymbol\varepsilon$-centered gradient penalty. Progressing from left to right, each column represents outcomes from different stages of training. The far-right column displays the ground truth data for comparison.} \label{Fig.gan-square}}
\vspace{-0.1in}
\end{figure*}

\begin{figure*}
\centering 
\includegraphics[width=1\textwidth]{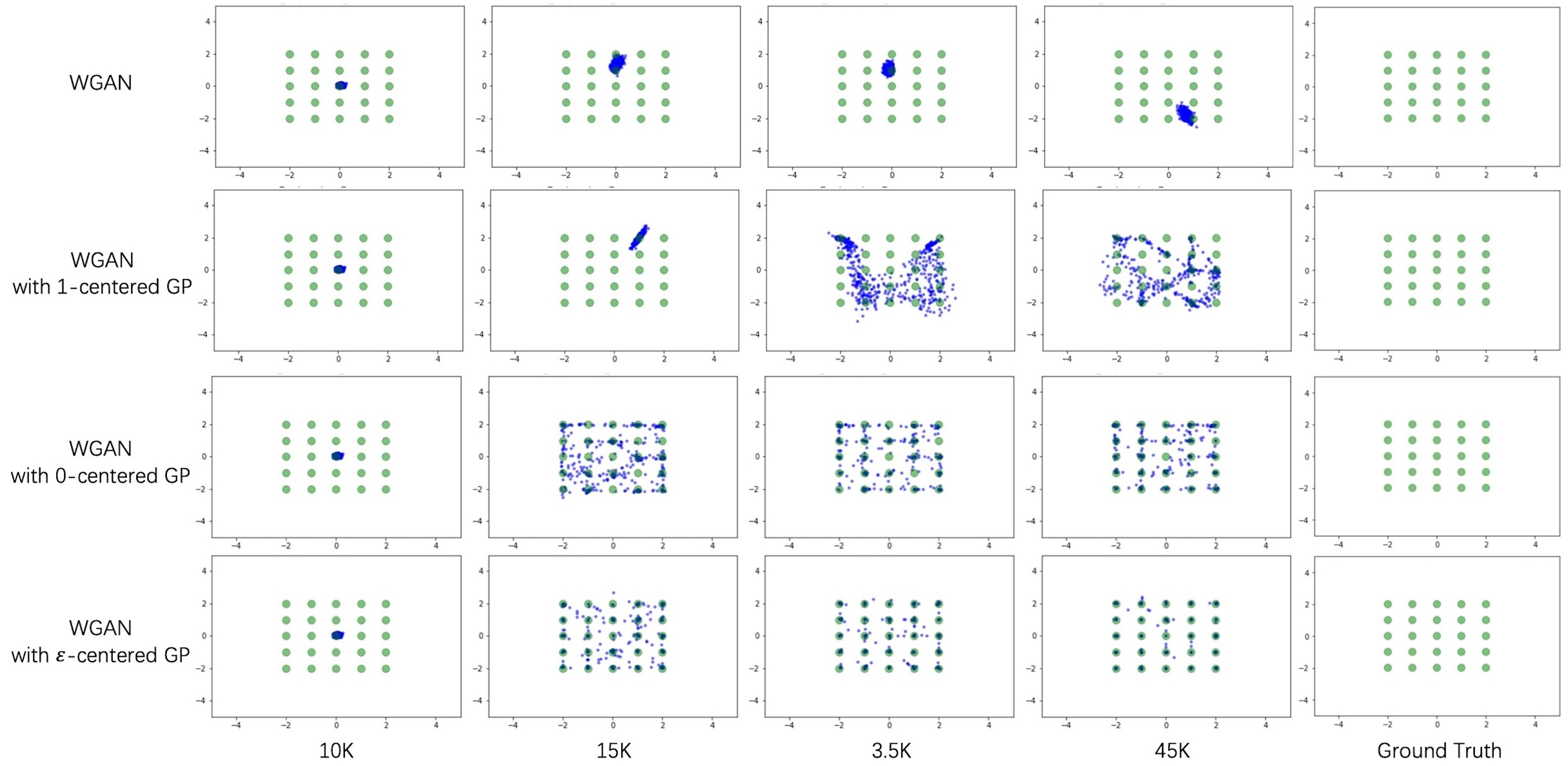} 
\caption{{Results for the WGAN with varying gradient penalties on the Grid dataset are displayed as follows: From top to bottom, the sequence includes the  WGAN, the WGAN with $1$-centered gradient penalty, the WGAN with $0$-centered gradient penalty, and the WGAN with $\boldsymbol\varepsilon$-centered gradient penalty. Progressing from left to right, each column represents outcomes from different stages of training. The far-right column displays the ground truth data for comparison.} \label{Fig.wgan-square}}
\vspace{-0.1in}
\end{figure*}

\begin{figure*}
\centering 
\includegraphics[width=1\textwidth]{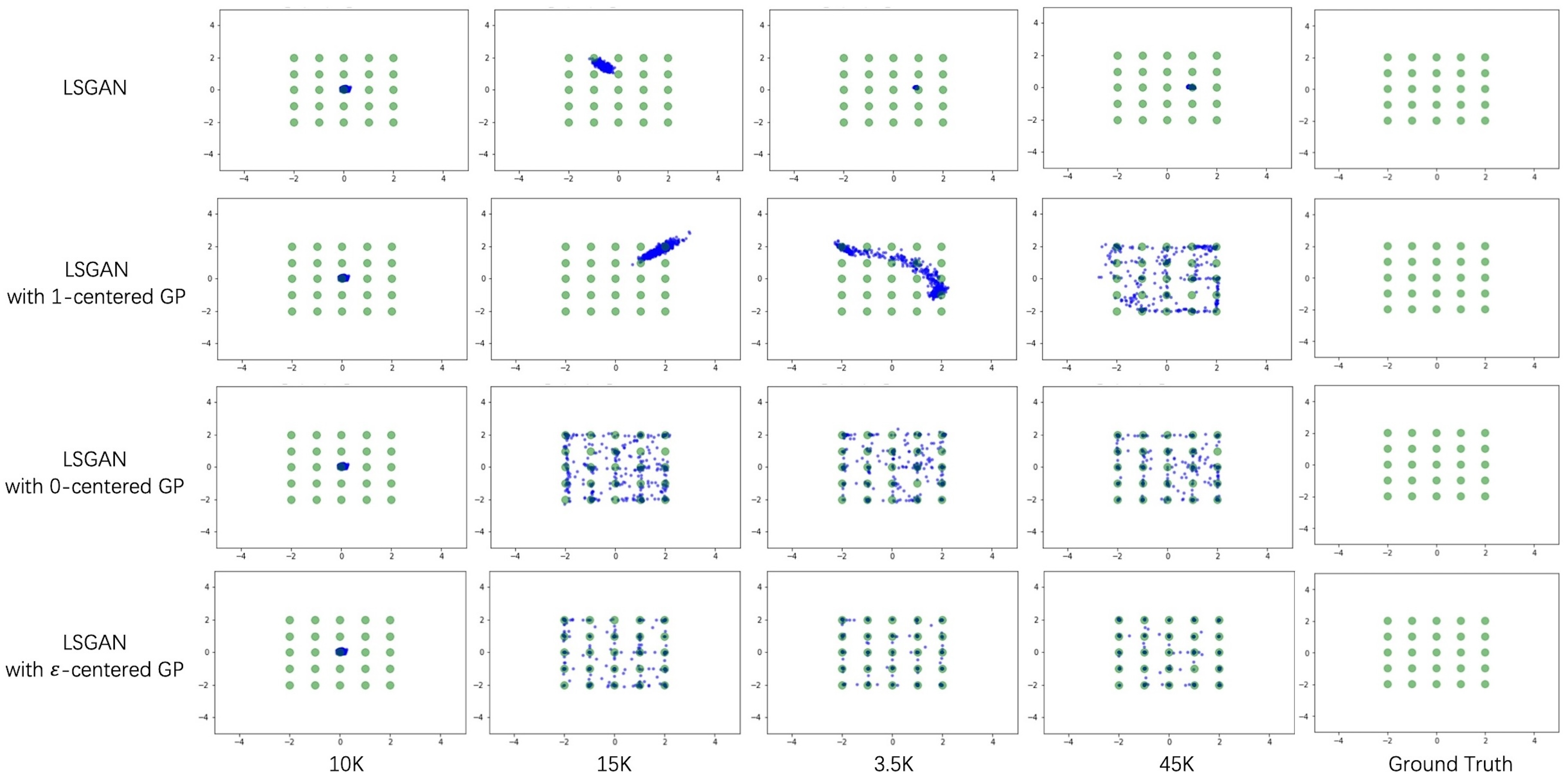} 
\caption{{Results for the LSGAN with varying gradient penalties on the Grid dataset are displayed as follows: From top to bottom, the sequence includes the  LSGAN, the LSGAN with $1$-centered gradient penalty, the LSGAN with $0$-centered gradient penalty, and the LSGAN with $\boldsymbol\varepsilon$-centered gradient penalty. Progressing from left to right, each column represents outcomes from different stages of training. The far-right column displays the ground truth data for comparison.}\label{Fig.lsgan-square}}
\vspace{-0.1in}
\end{figure*}

\begin{figure*}
\centering 
\includegraphics[width=1\textwidth]{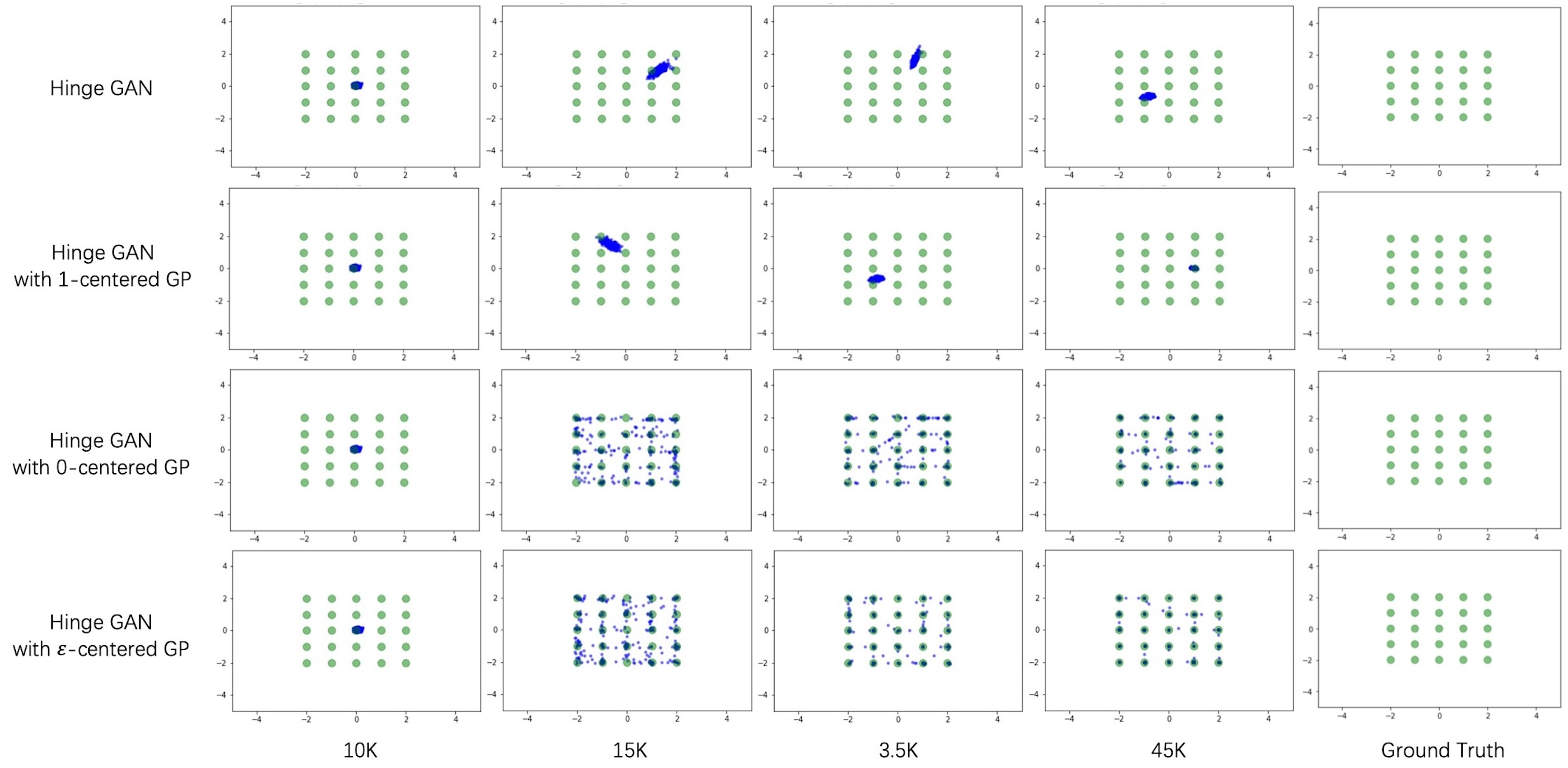} 
\caption{{Results for the Hinge GAN with varying gradient penalties on the Grid dataset are displayed as follows: From top to bottom, the sequence includes the  Hinge GAN, the Hinge GAN with $1$-centered gradient penalty, the Hinge GAN with $0$-centered gradient penalty, and the Hinge GAN with $\boldsymbol\varepsilon$-centered gradient penalty. Progressing from left to right, each column represents outcomes from different stages of training. The far-right column displays the ground truth data for comparison.}\label{Fig.hingegan-square}}
\vspace{-0.1in}
\end{figure*}

\begin{figure*}
\centering 
\includegraphics[width=1\textwidth]{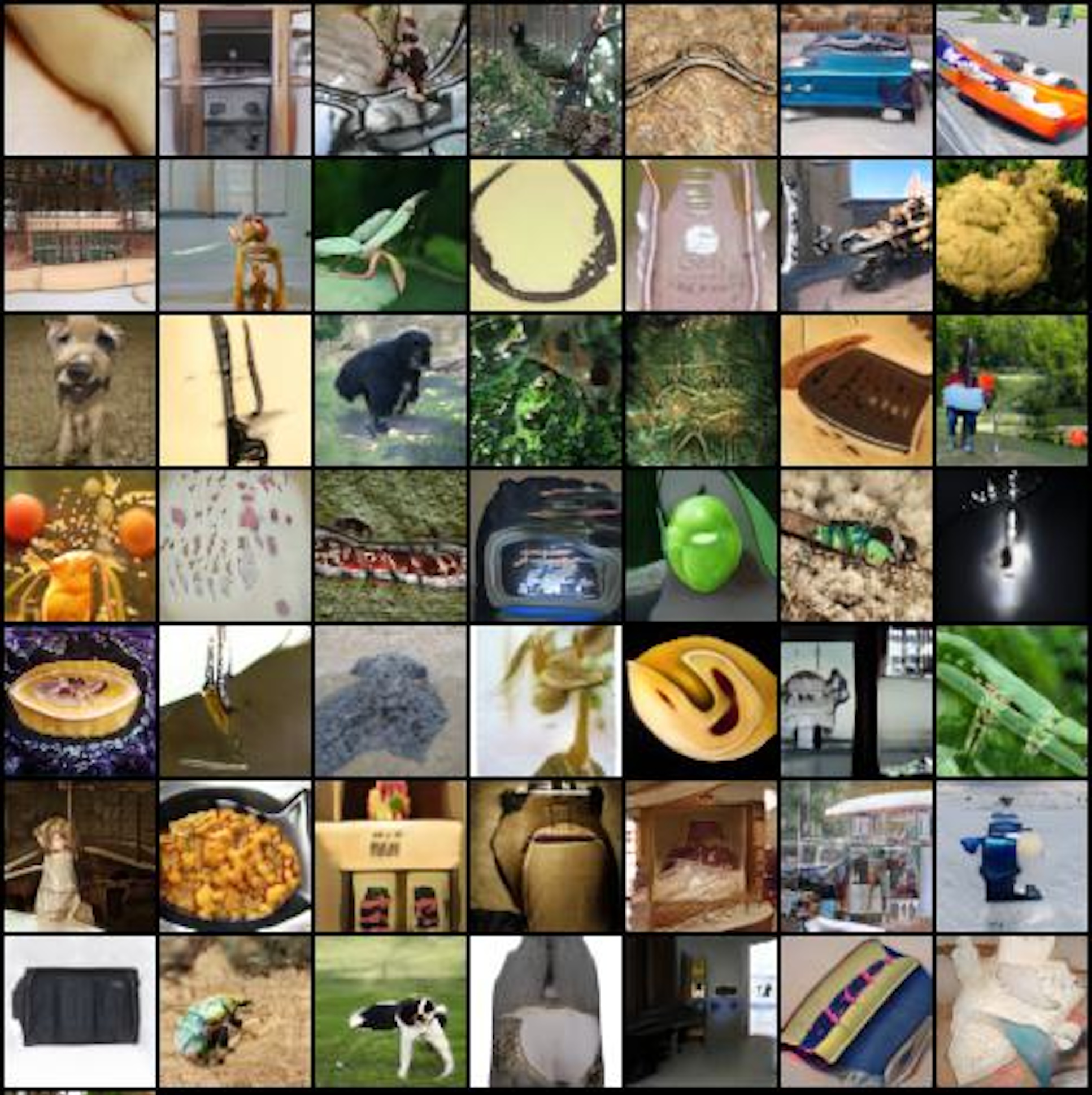} 
\caption{{Result for ImageNet64 from BigGAN with our $\boldsymbol\varepsilon$-centered gradient penalty.}\label{Fig.biggan-i64}}
\vspace{-0.1in}
\end{figure*}

\begin{figure*}
\centering 
\includegraphics[width=1\textwidth]{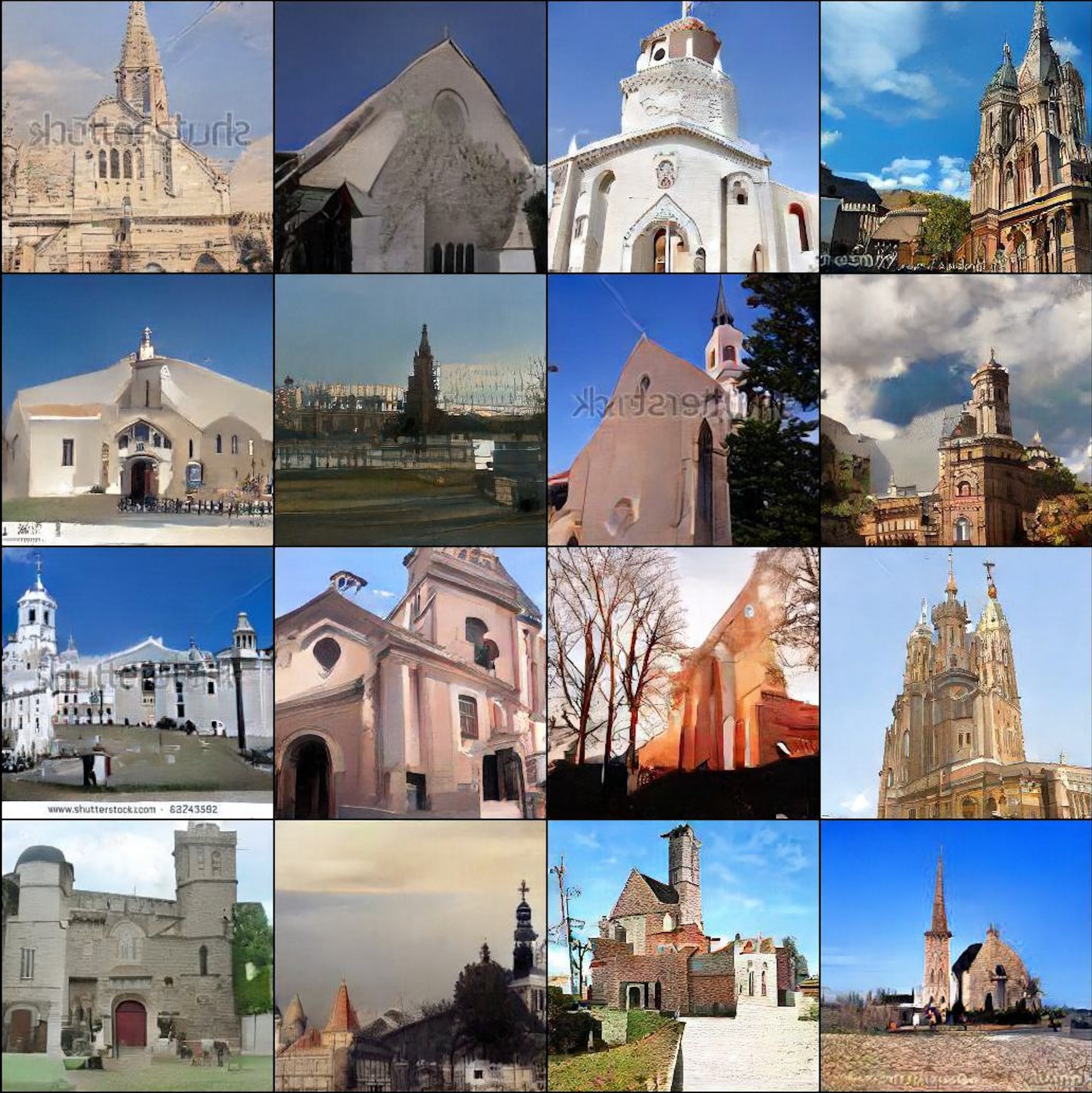} 
\caption{{Result for LSUN Church 256 from DDGAN with our $\boldsymbol\varepsilon$-centered gradient penalty.}\label{Fig.ddgan-lsun}}
\vspace{-0.1in}
\end{figure*}

\begin{figure*}
\centering 
\includegraphics[width=1\textwidth]{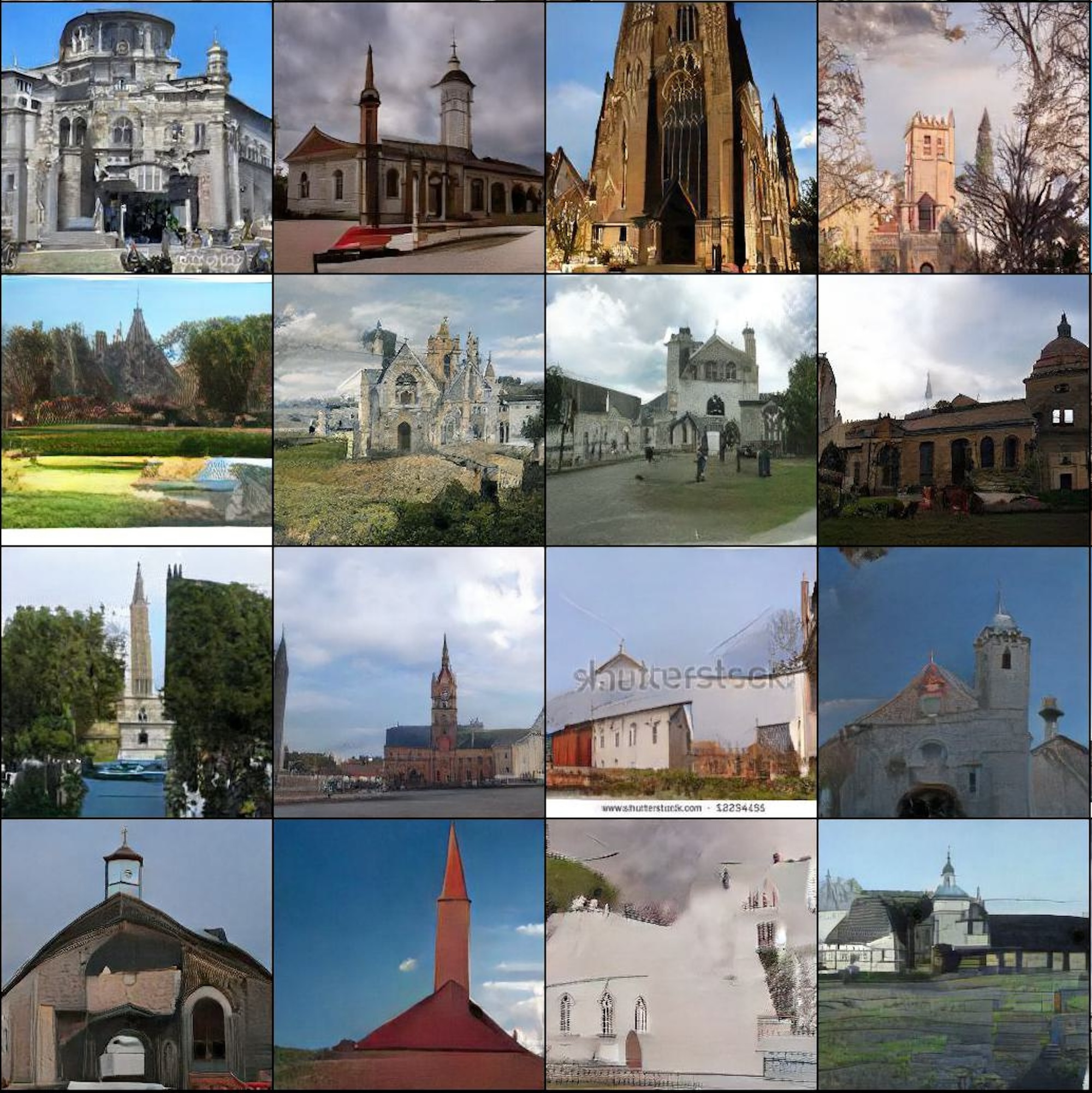} 
\caption{{Result for LSUN Church 256 from DDGAN with our $\boldsymbol\varepsilon$-centered gradient penalty.} \label{Fig.ddgan-lsun1}}
\vspace{-0.1in}
\end{figure*}

\noindent\textbf{Synthetic Datasets.}
{In the present results from synthetic datasets, we observed that unconstrained methods like the original GAN, LSGAN, WGAN and HingeGAN struggle to converge to all modes of the ring or grid datasets. However, these methods, when supplemented with gradient penalty, show an enhanced ability to converge to a mixture of Gaussians. Among the three types of gradient penalties tested, the $1$-centered gradient penalty exhibited inferior convergence compared to the $0$-centered gradient penalty and our proposed $\boldsymbol\varepsilon$-centered gradient penalty. Notably, our $\boldsymbol\varepsilon$-centered gradient penalty demonstrated a higher efficacy in driving more sample points to converge to the Gaussian points compared to the $0$-centered gradient penalty.}

\noindent \textbf{ImageNet.}
We present the experiment results of Li-CFG on the ImageNet datasets here. We compare the FID result of 
Li-CFG and CFG methods which are the same Neural architecture with parameters of different magnitudes. The Neural architecture simply uses the DCGAN and the res-block. Fig.~\ref{Fig.image0} shows the FID scores of the CFG method, the CFG method with two times the number of parameters, and Li-CFG with twice the number of parameters. Fig.~\ref{Fig.image1} presents the visualization effect of the above three models. Fig.~\ref{Fig.image2} manifest our Li-CFG results on the ImageNet datasets.

\begin{figure*}[t]
\vspace{-0.1in}
\centering
\includegraphics[width=1\textwidth]{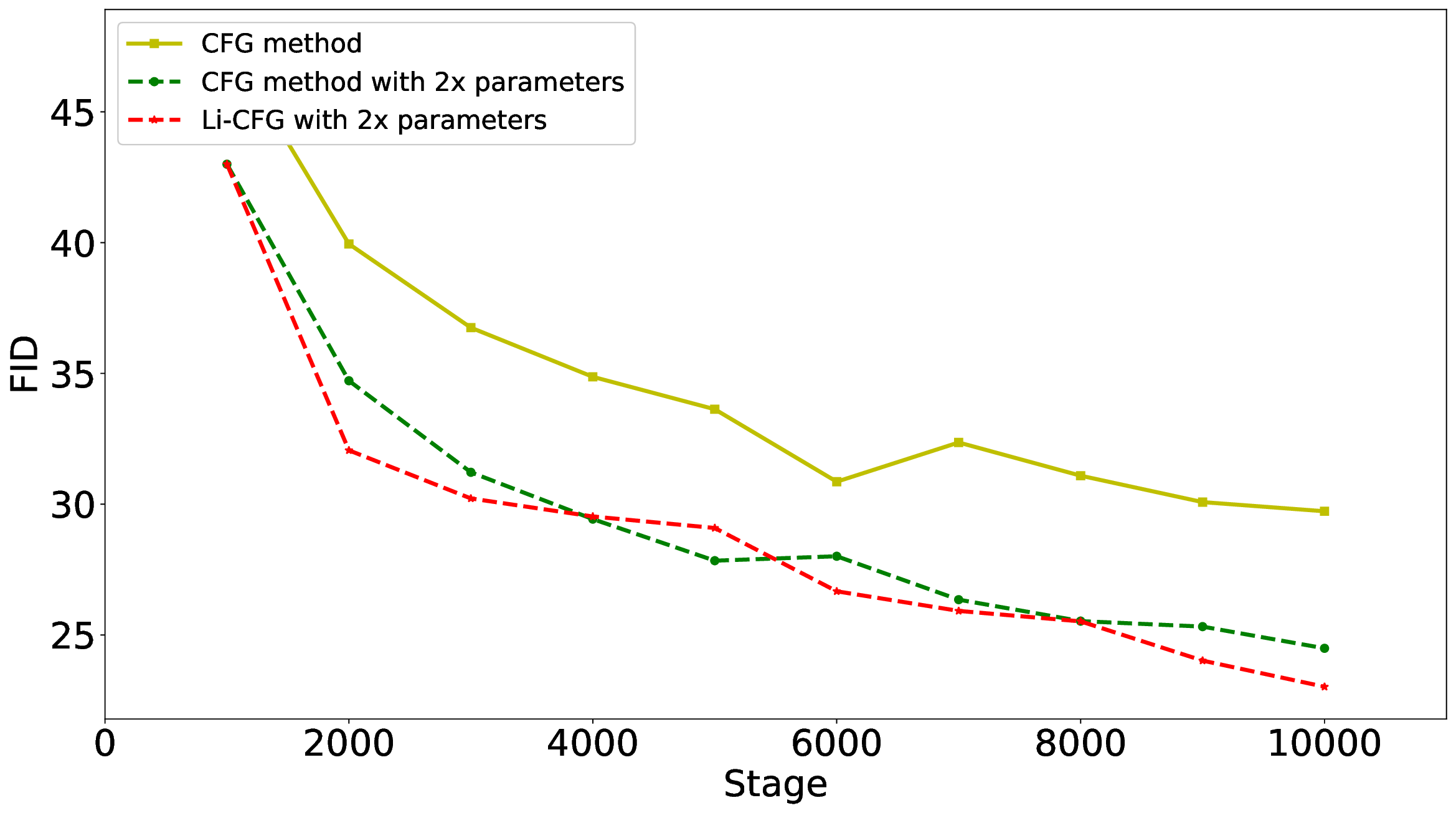}
\caption{Results for ImageNet: The yellow line is the FID results of the CFG method at different stages on the ImageNet datasets. The green and red lines represent the FID results of the CFG method and Li-CFG, respectively, with 2x parameters at different stages on the ImageNet datasets. The figure expresses such a conclusion that larger magnitude parameters of the CFG method behavior are much better than smaller magnitude parameters of the CFG method on the ImageNet datasets.
.\label{Fig.image0}} 
\vspace{-0.1in}
\end{figure*}

\begin{figure*}[t]
\centering
\includegraphics[width=0.8\textwidth]{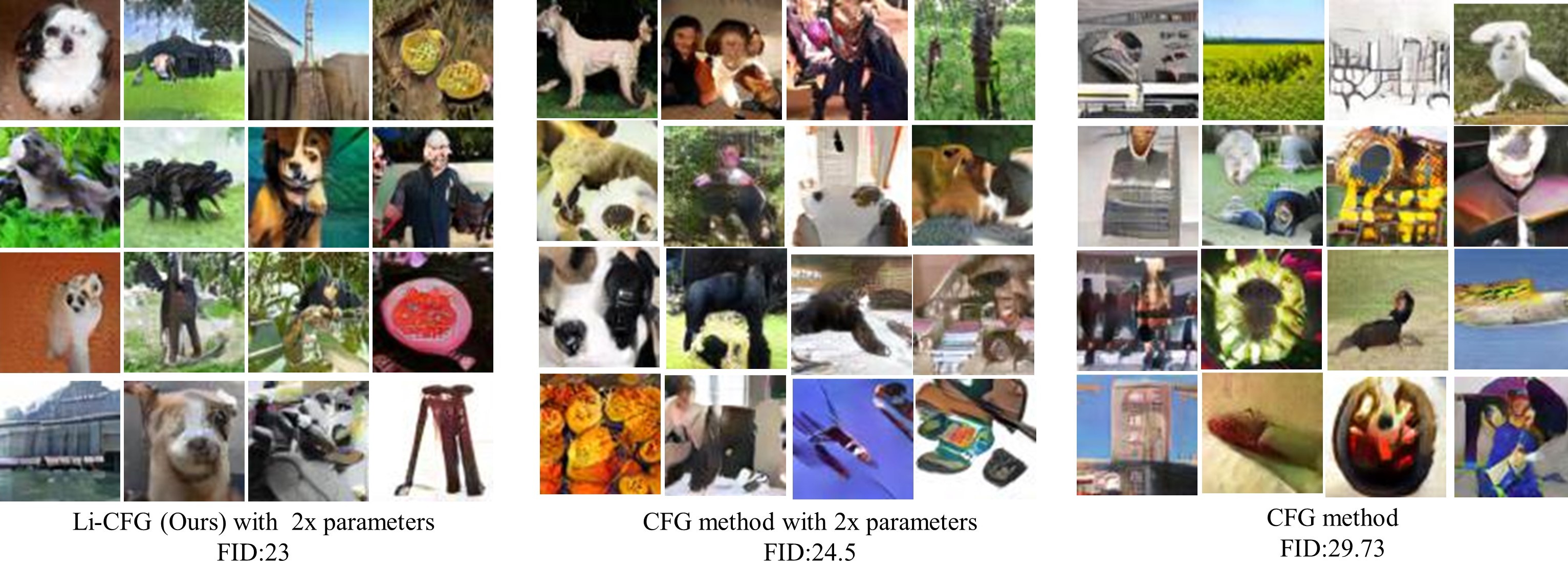}
\caption{Visual results for ImageNet: The left and middle columns display the visual results of the CFG method and Li-CFG with 2x parameters on the ImageNet datasets, respectively. The right column showcases the visual results of the CFG method on the ImageNet datasets.
.\label{Fig.image1}} 
\vspace{-0.1in}
\end{figure*}

\begin{figure*}[t]
\centering
\includegraphics[width=1\textwidth]{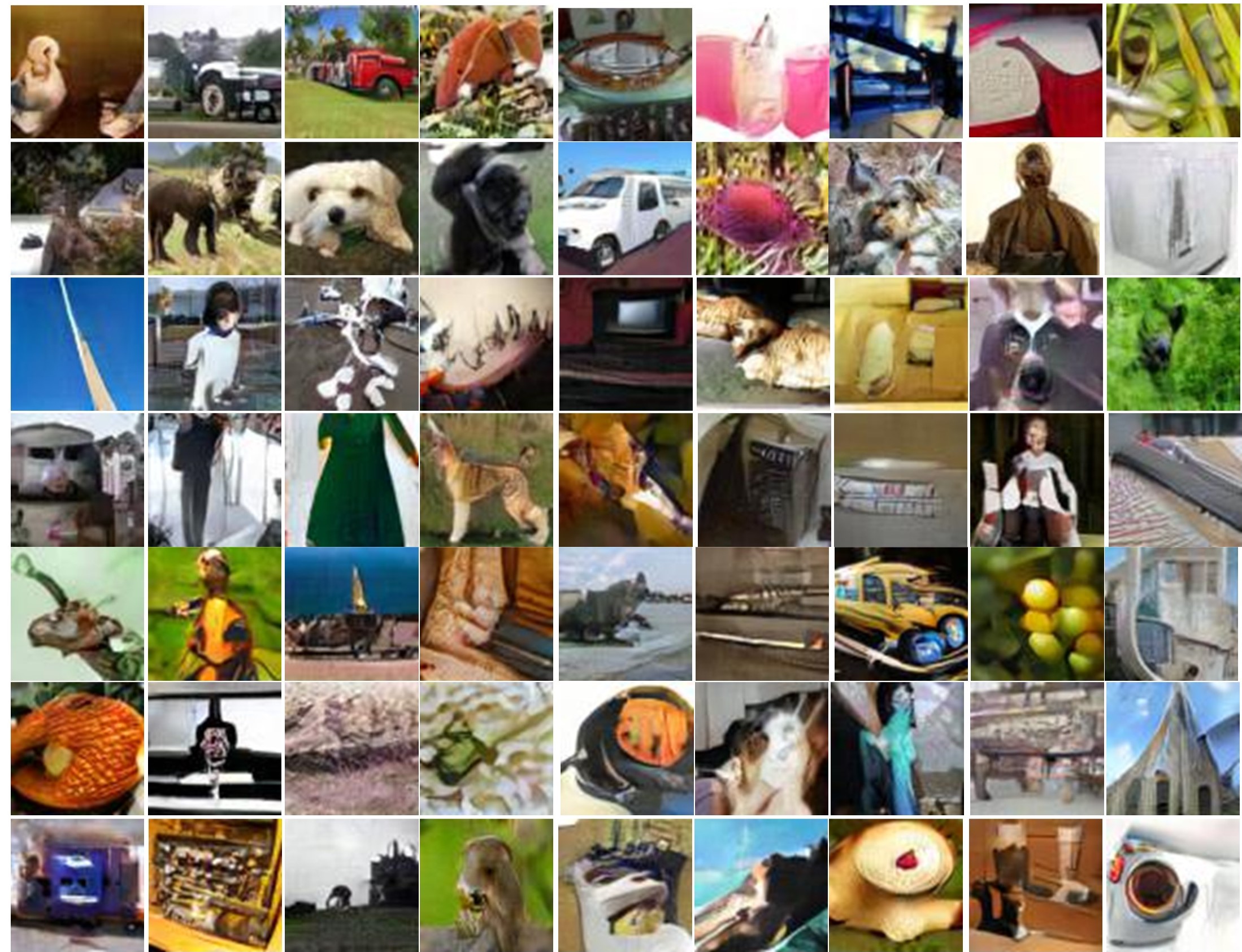}
\vspace{-0.1in}
\caption{Visual result of Li-CFG on ImageNet datasets.
.\label{Fig.image2}} 
\vspace{-0.1in}
\end{figure*}

\begin{figure*}[t]
\centering
\includegraphics[width=0.95\textwidth]{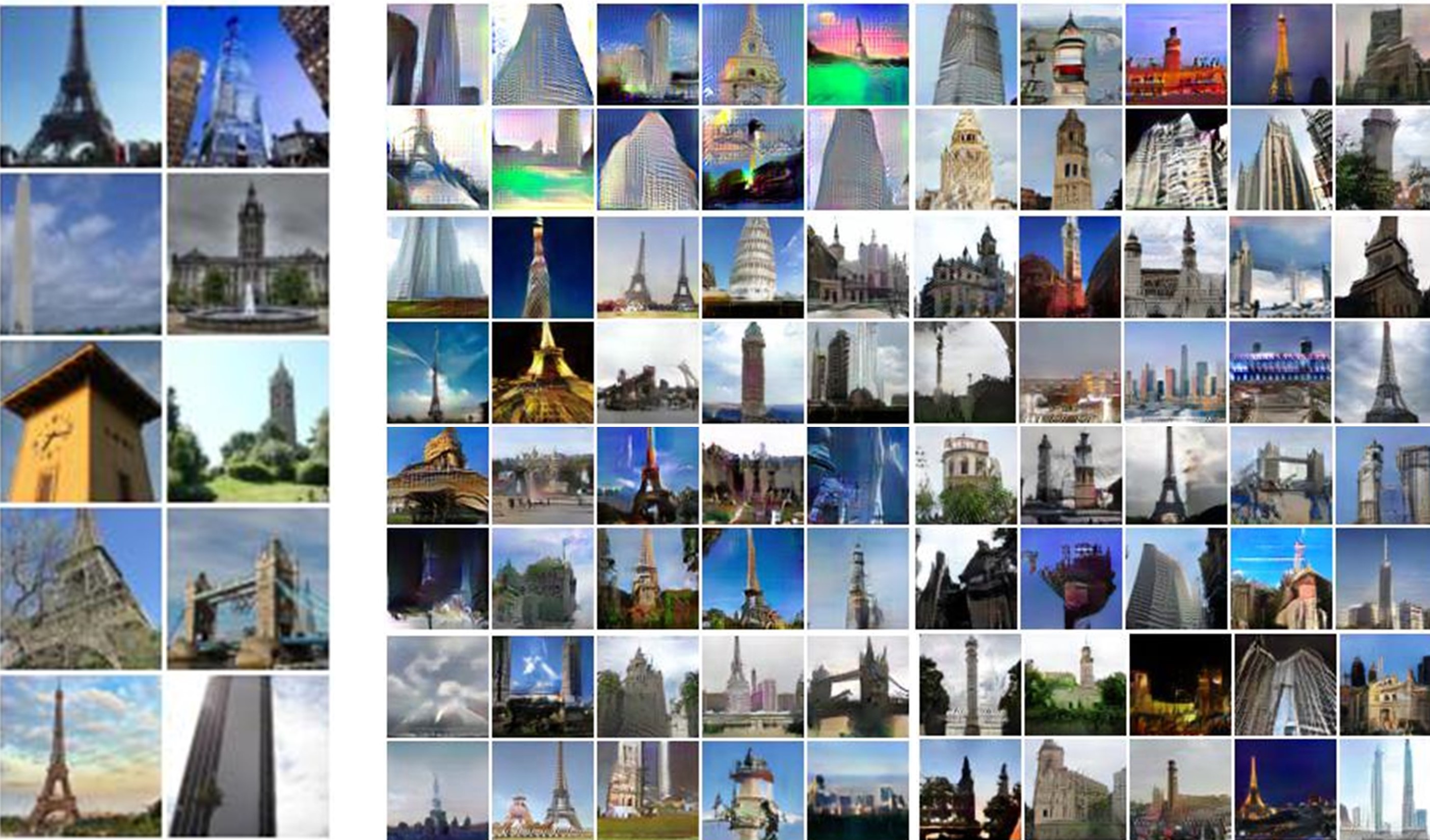} 
\caption{Results for LSUN T: The left column consists of real images. The first two rows in the right column depict the visual results of the CFG method. The second two rows in the right column depict the visual results of the Li-CFG with $\boldsymbol\varepsilon$-centered gradient penalty (ours). The third two rows in the right column depict the visual results of the Li-CFG with the $1$-centered gradient penalty. The last two rows in the right column depict the visual results of the Li-CFG with $0$-centered gradient penalty. The parameter settings for Li-CFG are as follows: $\eta =$2.5e-4, $\gamma = 0.1$, $\delta(\boldsymbol x) = 1$, and $\varepsilon' = 0.3$.\label{Fig.main0}} 
\vspace{-0.1in}
\end{figure*}

\begin{figure*}[t]
\centering
\includegraphics[width=0.95\textwidth]{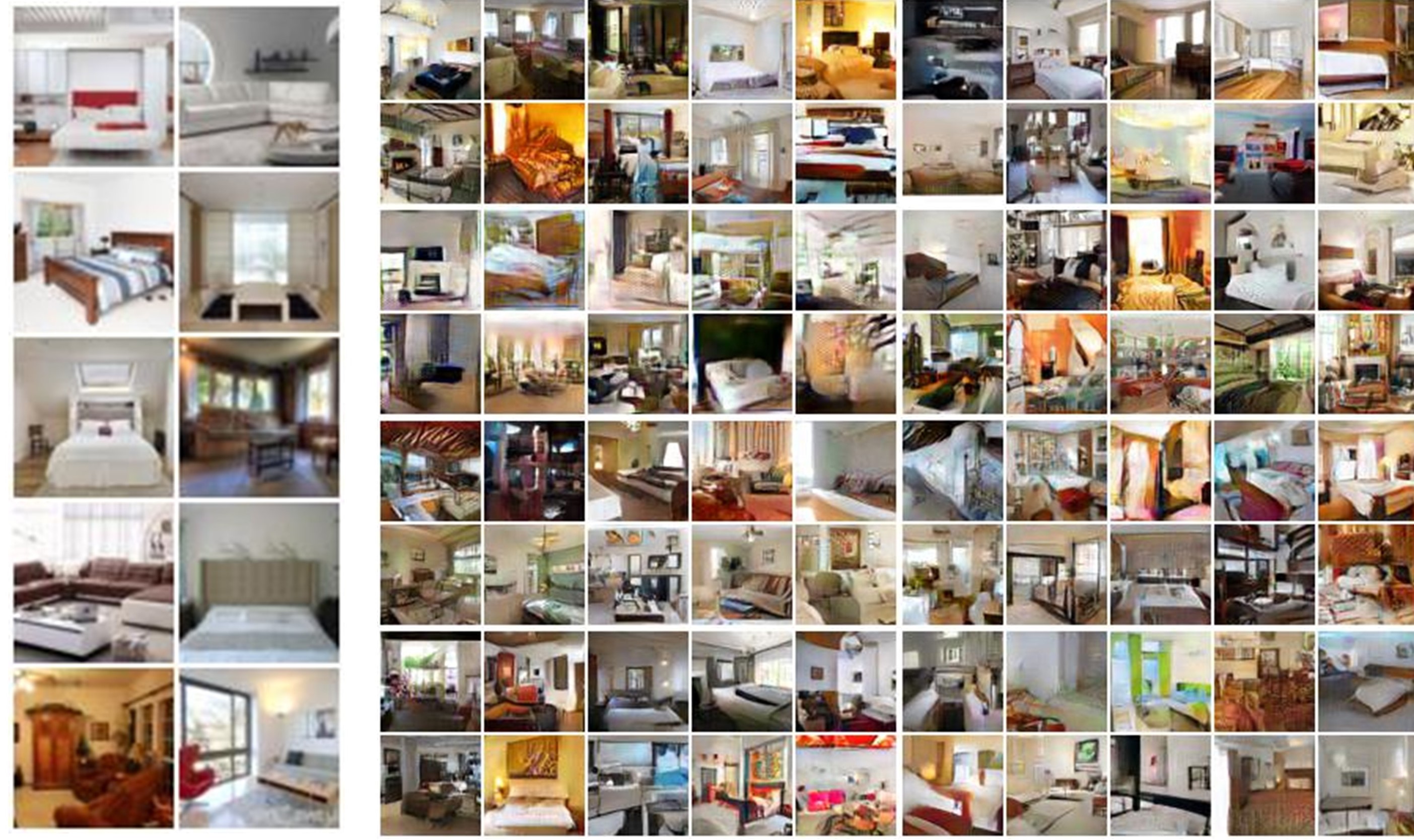} 
\caption{Results for BR+LR: The left column consists of real images. The first two rows in the right column depict the visual results of the CFG method. The second two rows in the right column depict the visual results of the Li-CFG with $\boldsymbol\varepsilon$-centered gradient penalty (ours). The third two rows in the right column depict the visual results of the Li-CFG with the $1$-centered gradient penalty. The last two rows in the right column depict the visual results of the Li-CFG with $0$-centered gradient penalty. The parameter settings for Li-CFG are as follows: $\eta =$2.5e-4, $\gamma = 0.1$, $\delta(\boldsymbol x) = 1$, and $\varepsilon' = 0.3$.\label{Fig.main1}} 
\vspace{-0.1in}
\end{figure*}

\begin{figure*}[t]
\centering
\includegraphics[width=1\textwidth]{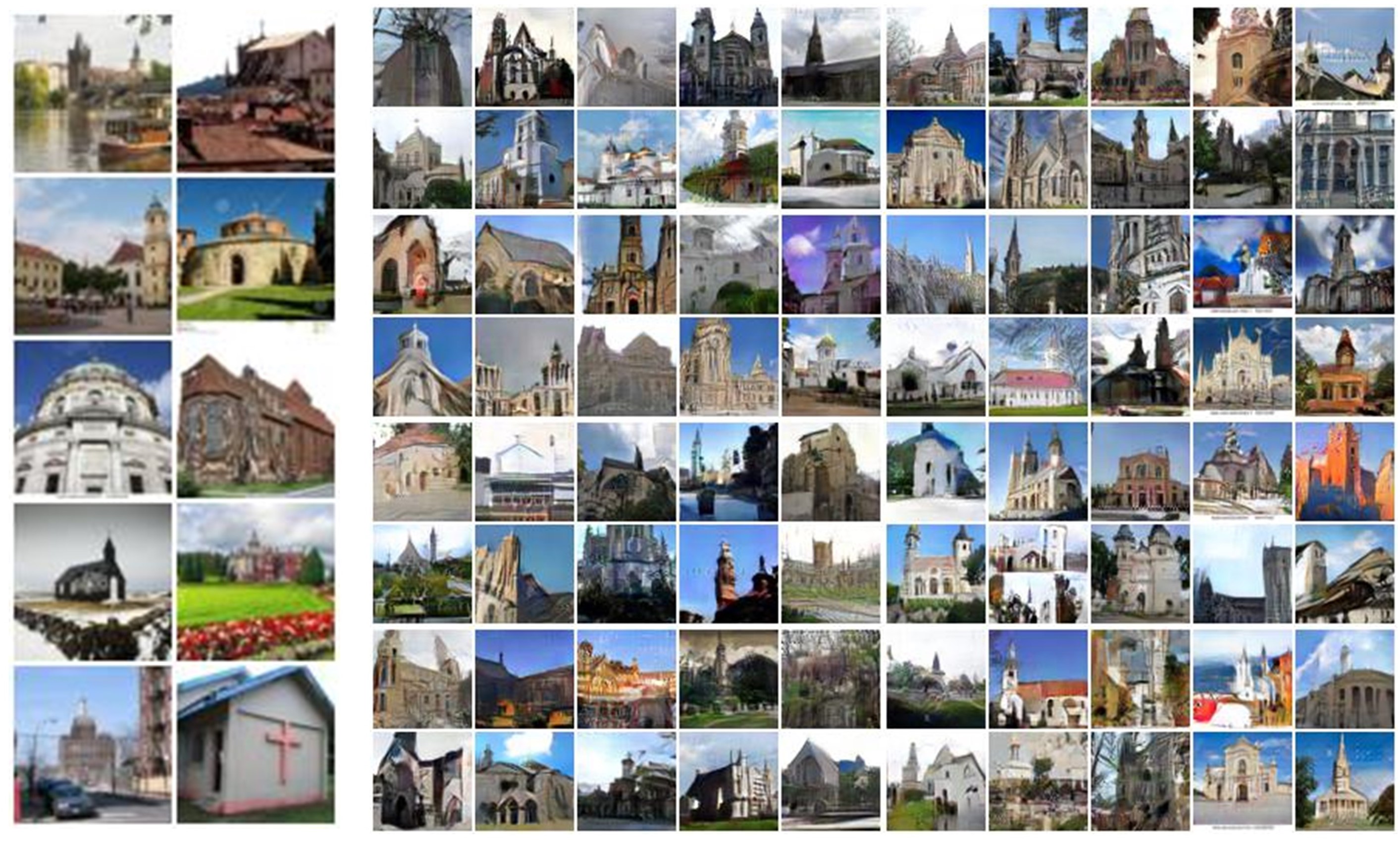} 
\caption{Results for LSUN Church: The left column consists of real images. The first two rows in the right column depict the visual results of the CFG method. The second two rows in the right column depict the visual results of the Li-CFG with $\boldsymbol\varepsilon$-centered gradient penalty (ours). The third two rows in the right column depict the visual results of the Li-CFG with the $1$-centered gradient penalty. The last two rows in the right column depict the visual results of the Li-CFG with $0$-centered gradient penalty. The parameter settings for Li-CFG are as follows: $\eta =$2.5e-4, $\gamma = 0.1$, $\delta(\boldsymbol x) = 1$, and $\varepsilon' = 0.3$. \label{Fig.main2}} 
\vspace{-0.1in}
\end{figure*}
\begin{figure*}[t]
\centering
\includegraphics[width=1\textwidth]{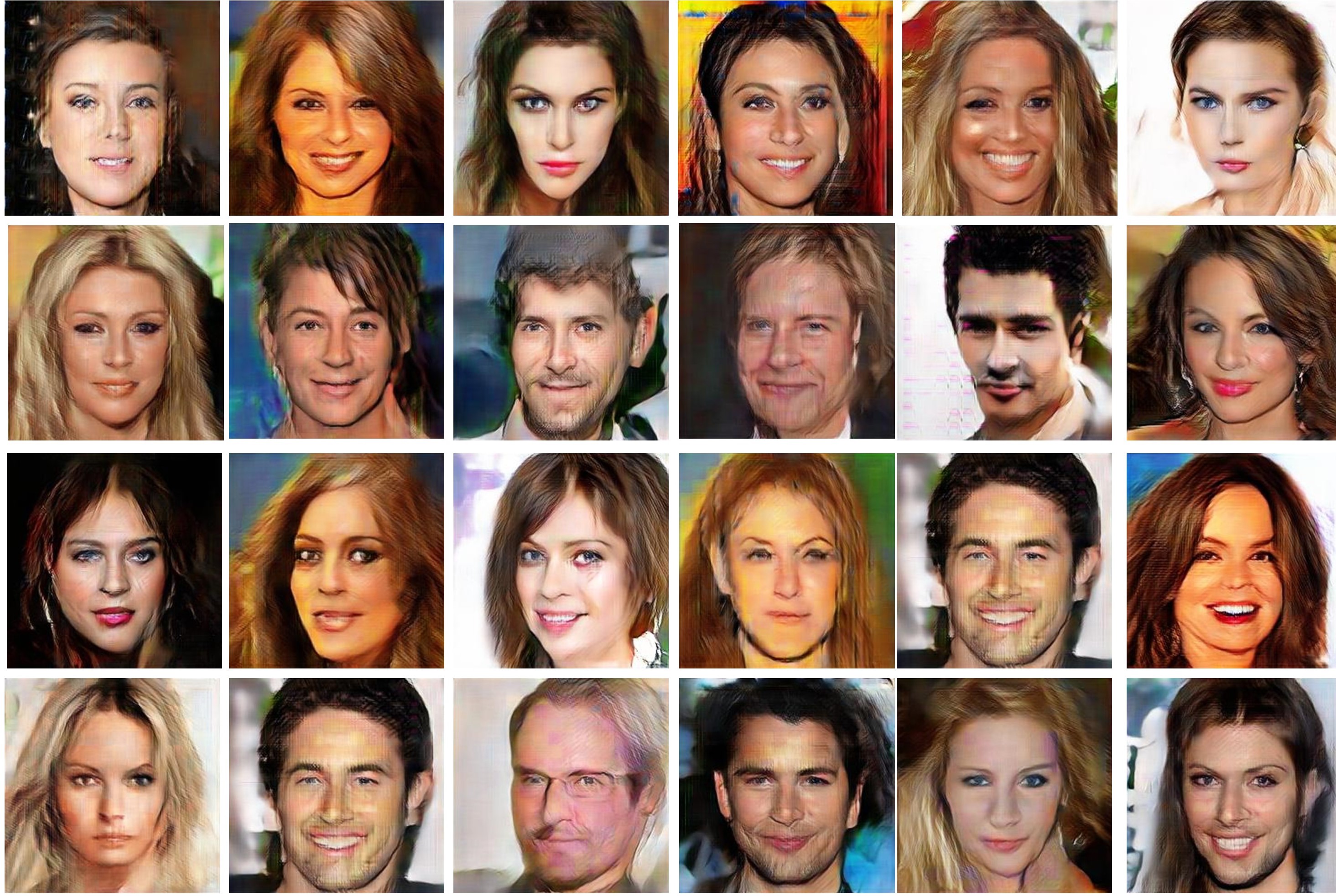} 
\caption{Results for celeBA with resolution of 256: The outcomes are obtained using our $\boldsymbol\varepsilon$-centered Li-CFG method.  The parameter settings for Li-CFG are as follows: $\eta=$2.5e-4, $\gamma=$10, $\delta(\boldsymbol x)=$10 and $\varepsilon'=$0.3. \label{Fig.main5}} 
\vspace{-0.1in}
\end{figure*}

\begin{figure*}[t]
\centering
\includegraphics[width=0.9\textwidth]{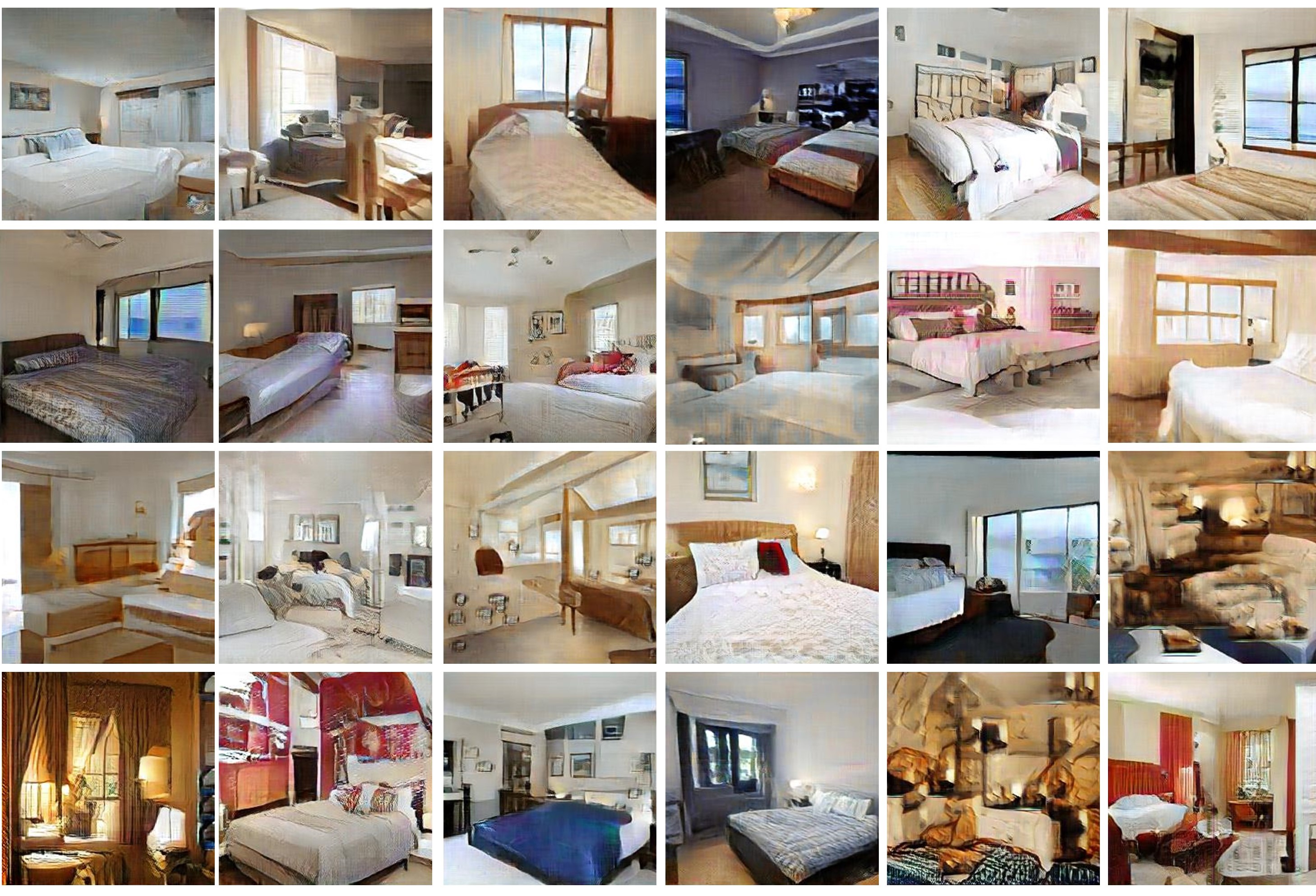} 
\caption{Results for LSUN Bedroom with the resolution of 256: The outcomes are obtained using our $\boldsymbol\varepsilon$-centered Li-CFG method.
The parameter settings for Li-CFG are as follows: 
$\eta=$2.5e-4, $\gamma=$10, $\delta(\boldsymbol x)=$10 and $\varepsilon'=$0.3. \label{Fig.bed256}} 
\vspace{-0.1in}
\end{figure*}

\begin{figure*}[t]
\centering
\includegraphics[width=0.9\textwidth]{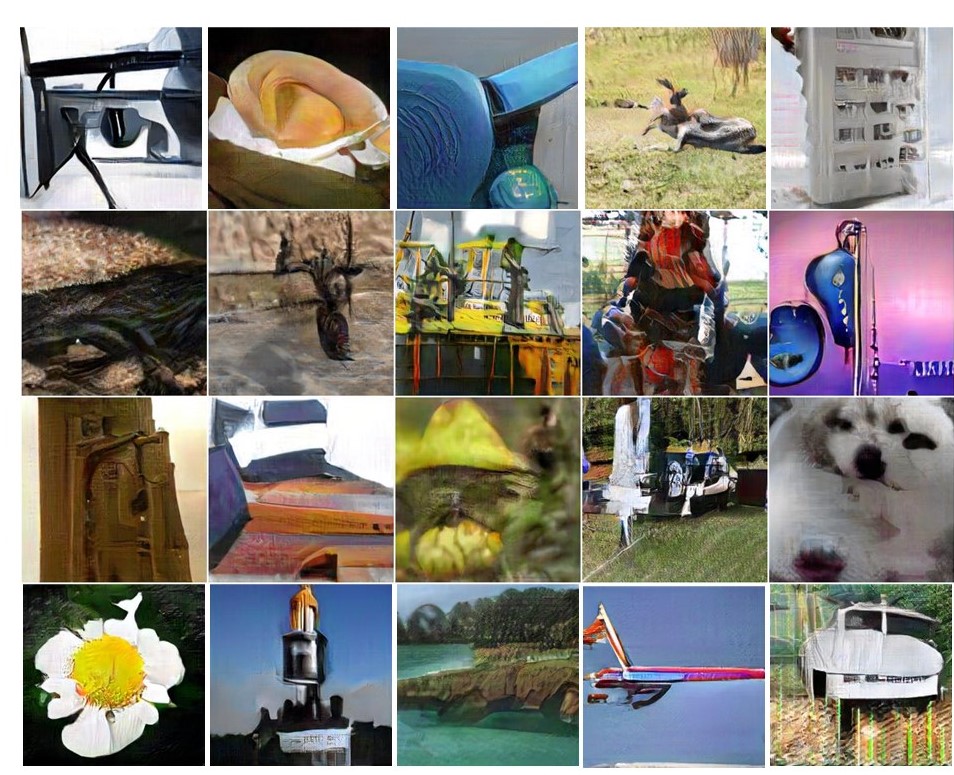}
\caption{Results for ImageNet with the resolution of 256: The outcomes are obtained using our $\boldsymbol\varepsilon$-centered Li-CFG method.
The parameter settings for Li-CFG are as follows: $\eta=$2.5e-4, $\gamma=$0.1; $\delta(\boldsymbol x)=$20 and $\varepsilon'=$1. \label{Fig.main3}} 
\vspace{-0.1in}
\end{figure*}

\begin{figure*}[t]
\centering
\includegraphics[width=1\textwidth]{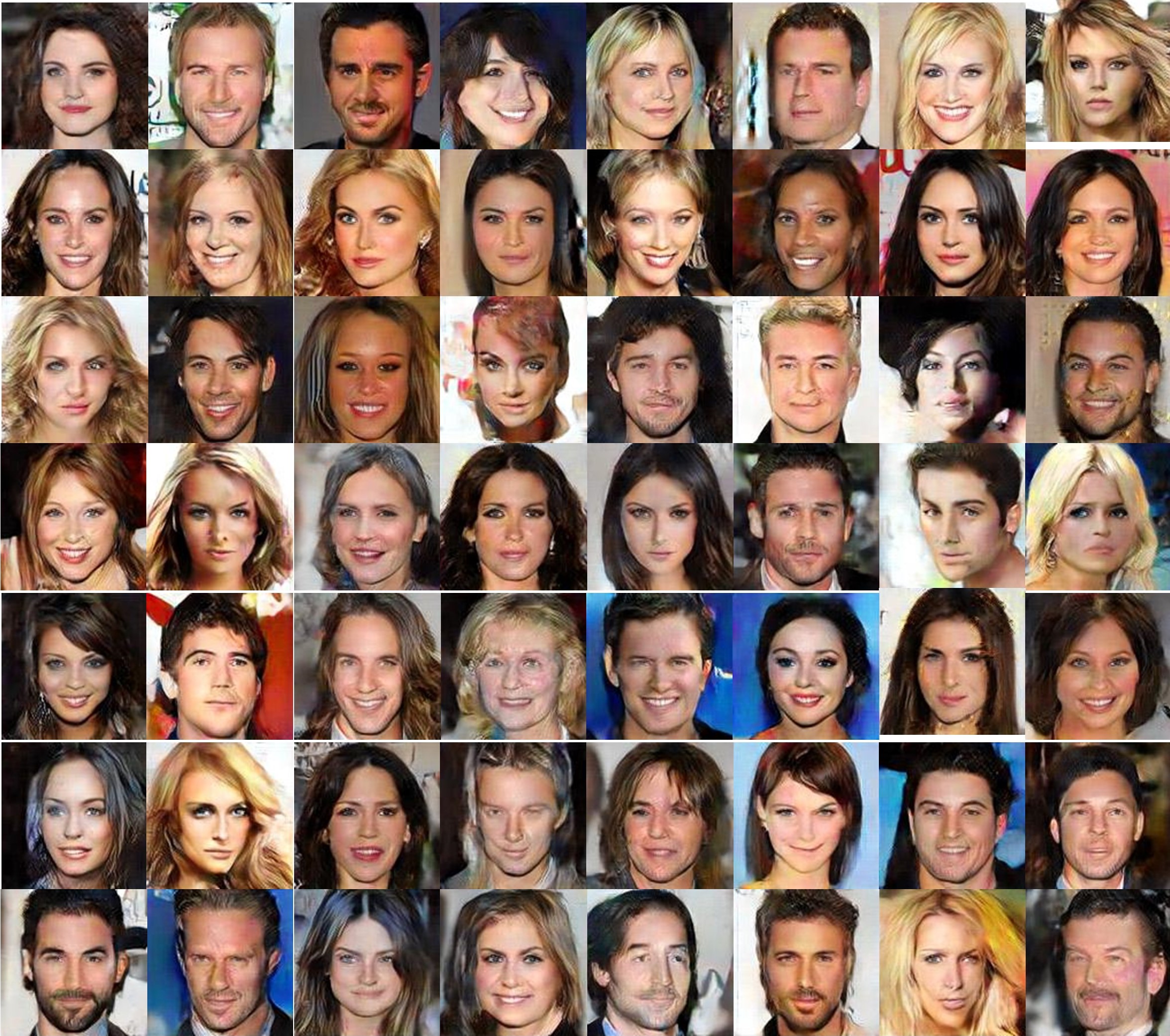} 
\caption{Results for CeleBA with the resolution of 128: The outcomes are obtained using our $\boldsymbol\varepsilon$-centered Li-CFG method. 
The parameter settings for Li-CFG are as follows: $\eta=$2.5e-4, $\gamma=$10; $\delta(\boldsymbol x)=$5 and $\varepsilon'=$0.3.\label{Fig.main4}} 
\vspace{-0.1in}
\end{figure*}




\end{appendices}

% \bibliographystyle{unsrtnat}
% \bibliographystyle{natbib}

\end{document}